%% file: arxiv_submission.tex
\documentclass[dvipsnames]{article}

\PassOptionsToPackage{numbers, compress}{natbib}

\usepackage{times}
\usepackage{xcolor}
\usepackage{natbib}
\setcitestyle{authoryear,round,citesep={;},aysep={,},yysep={;}}

\newif\ificlrfinal
\iclrfinalfalse

\def\addcontentsline#1#2#3{}

\renewenvironment{abstract}{\vskip.075in\centerline{\large
Abstract}\vspace{0.5ex}\begin{quote}}{\par\end{quote}\vskip 1ex}

\makeatletter

\def\section{\@startsection {section}{1}{\z@}{-2.0ex plus
    -0.5ex minus -.2ex}{1.5ex plus 0.3ex
minus0.2ex}{\large\raggedright}}

\def\subsection{\@startsection{subsection}{2}{\z@}{-1.8ex plus
-0.5ex minus -.2ex}{0.8ex plus .2ex}{\normalsize\raggedright}}
\def\subsubsection{\@startsection{subsubsection}{3}{\z@}{-1.5ex
plus      -0.5ex minus -.2ex}{0.5ex plus
.2ex}{\normalsize\raggedright}}
\def\paragraph{\@startsection{paragraph}{4}{\z@}{1.5ex plus
0.5ex minus .2ex}{-1em}{\normalsize\bf}}
\def\subparagraph{\@startsection{subparagraph}{5}{\z@}{1.5ex plus
  0.5ex minus .2ex}{-1em}{\normalsize}}

\makeatother

\skip\footins 9pt plus 4pt minus 2pt
\def\footnoterule{\kern-3pt \hrule width 12pc \kern 2.6pt }
\date{\vspace{-1em}}

\title{Unpacking Information Bottlenecks: Surrogate Objectives for Deep Learning}

\author{%
  Andreas Kirsch \hspace{5mm}
  Clare Lyle \hspace{5mm}
  Yarin Gal \\
  OATML \\
  Department of Computer Science \\
  University of Oxford \\
  \texttt{\{andreas.kirsch, clare.lyle, yarin\}@cs.ox.ac.uk}
}

\usepackage[utf8]{inputenc} %
\usepackage[T1]{fontenc}    %
\usepackage{hyperref}       %
\usepackage{url}            %
\usepackage{booktabs}       %
\usepackage{amsfonts}       %
\usepackage{nicefrac}       %
\usepackage{microtype}      %

\usepackage{rotating}

\usepackage{graphicx}
\usepackage{fancyhdr}

\usepackage{hyperref}
\definecolor{mydarkblue}{rgb}{0,0.08,0.45}
\hypersetup{ %
pdftitle={},
pdfauthor={},
pdfsubject={},
pdfkeywords={},
pdfborder=0 0 0,
pdfpagemode=UseNone,
colorlinks=true,
linkcolor=mydarkblue,
citecolor=mydarkblue,
filecolor=mydarkblue,
urlcolor=mydarkblue,
pdfview=FitH}

\input{math_commands.tex}

\usepackage{graphicx}
\usepackage{wrapfig}
\usepackage{subcaption}

\usepackage{tablefootnote}
\usepackage{amsmath}
\usepackage{bbold} %
\usepackage{bbm}
\usepackage{amsthm}
\usepackage{txfonts}
\usepackage{mathtools}
\usepackage{graphbox}
\usepackage{xspace}
\usepackage{cleveref}

\usepackage{chngcntr} %

\usepackage{tikz}

\allowdisplaybreaks

\newtheorem{theorem}{Theorem}
\newtheorem{proposition}{Proposition}
\newtheorem{observation}{Proposition}
\newtheorem{corollary}{Corollary}

\newtheorem*{exobservation}{Proposition}

\DeclareMathOperator{\opE}{\mathbbm{E}}
\DeclareMathOperator{\opH}{H}
\DeclareMathOperator{\opI}{I}
\DeclareMathOperator{\opp}{p}
\DeclareMathOperator{\oppm}{p_\theta}
\DeclareMathOperator{\opph}{\hat{p}}
\DeclareMathOperator{\opq}{q}
\DeclareMathOperator{\opr}{r}
\DeclareMathOperator{\opmus}{\mu^*}
\DeclareMathOperator{\opVar}{\mathrm{Var}}
\DeclareMathOperator{\opCov}{\mathrm{Cov}}
\DeclareMathOperator{\opDiag}{\mathrm{diag}}

\renewcommand{\Var}[1]{\opVar[ #1 ]}
\renewcommand{\Cov}[1]{\opCov[ #1 ]}

\newcommand{\hVar}[1]{\widehat{\opVar}[#1]}
\newcommand{\cVar}[2]{\opVar[#1 \mathbin{\vert} #2]}
\newcommand{\hcVar}[2]{\widehat{\opVar}[#1 \mathbin{\vert} #2]}
\newcommand{\cCov}[2]{\opCov[#1 \mathbin{\vert} #2]}
\DeclareMathOperator{\tr}{tr}

\newcommand{\MuS}[1]{\opmus \mathopen{} ( #1  )\mathclose{}}
\newcommand{\Entropy}[1]{\opH [ #1 ]}
\newcommand{\CrossEntropy}[2]{\opH ( #1 \mathbin{\vert \vert} #2 )}
\newcommand{\ic}[1]{h \left ( #1 \right )}
\newcommand{\Hc}[2]{\opH [ #1 \mathbin{\vert} #2 ]}
\newcommand{\MI}[2]{\opI [ #1 ; #2 ]}
\newcommand{\MIt}[3]{\opI [ #1 ; #2 ; #3 ]}
\newcommand{\MIc}[3]{\opI [ #1 ; #2 \mathbin{\vert} #3 ]}
\newcommand{\MIct}[4]{\opI [ #1 ; #2 ; #3 \mathbin{\vert} #4 ]}
\newcommand{\prob}[1]{\opp ( #1 )}
\newcommand{\probc}[2]{\opp ( #1 \mathbin{\vert} #2 )}
\newcommand{\probh}[1]{\opph ( #1 )}
\newcommand{\probhc}[2]{\opph ( #1 \mathbin{\vert} #2 )}
\newcommand{\probm}[1]{\oppm ( #1 )}
\newcommand{\probmc}[2]{\oppm ( #1 \mathbin{\vert} #2 )}
\newcommand{\qprob}[1]{\opq ( #1 )}
\newcommand{\qprobc}[2]{\opq ( #1 \mathbin{\vert} #2 )}
\newcommand{\rprob}[1]{\opr ( #1 )}

\newcommand{\kale}[2]{\KL ( #1 \mathbin{\vert \vert} #2 )}
\newcommand{\N}{\mathcal{N}}
\newcommand{\normaldist}[2]{\N(#1,\,#2)}

\newcommand{\reals}{\mathbb{R}}

\newcommand{\x}{x}
\newcommand{\X}{\textnormal{X}}
\newcommand{\y}{y}
\newcommand{\Y}{\textnormal{Y}}
\newcommand{\yhat}{\hat{y}}
\newcommand{\Yhat}{\hat{\textnormal{Y}}}
\newcommand{\z}{z}
\newcommand{\Z}{\textnormal{Z}}

\newcommand{\pmdecoder}{\probmc{\yhat}{\z}}
\newcommand{\pmdecoderce}{\probmc{\Yhat=\y}{\z}}
\newcommand{\pdecoder}{\probc{\y}{\z}}
\newcommand{\pdecoderhat}{\probc{\Y=\yhat}{\z}}

\newcommand{\pmpredict}{\probmc{\yhat}{\x}}
\newcommand{\pmpredictce}{\probmc{\Yhat=\y}{\x}}

\newcommand{\pmgenerative}{\probm{\x}\pmpredict}
\newcommand{\pmgenerativece}{\probm{\x}\pmpredictce}

\newcommand{\CEpmdecoder}{\CrossEntropy{\pdecoder}{\pmdecoderce}}
\newcommand{\CEpmpredict}{\CrossEntropy{\pycx}{\pmpredictce}}

\newcommand{\thetaEntropy}[2]{\opH_\theta [ #1 \mathbin{\vert} #2 ]}

\definecolor{ik-darkblue}{HTML}{005aff}
\definecolor{ik-brightblue}{HTML}{00a2ff}

\definecolor{ik-darkorange}{HTML}{ff6c00}
\definecolor{ik-brightorange}{HTML}{ffb000}

\definecolor{ik-pink}{HTML}{e55ca5}
\definecolor{ik-green}{HTML}{009926}
\definecolor{ik-purple}{HTML}{66001d}
\definecolor{ik-gray}{HTML}{bebebe}

\colorlet{labeluncertaintycolor}{ik-pink}
\newcommand{\labeluncertainty}{\textcolor{labeluncertaintycolor}{\Hc\Y\X}}
\newcommand{\LabelUncertainty}{\textcolor{labeluncertaintycolor}{Label Uncertainty}\xspace}

\colorlet{encodinguncertaintycolor}{ik-purple}
\newcommand{\encodinguncertainty}{\textcolor{encodinguncertaintycolor}{\Hc\Z\X}}
\newcommand{\EncodingUncertainty}{\textcolor{encodinguncertaintycolor}{Encoding Uncertainty}\xspace}

\colorlet{residualinfocolor}{ik-darkblue}
\newcommand{\residualinfo}{\textcolor{residualinfocolor}{\MIc\X\Y\Z}}
\newcommand{\ResidualInfo}{\textcolor{residualinfocolor}{Residual Information}\xspace}

\colorlet{redundantinfocolor}{ik-darkorange}
\newcommand{\redundantinfo}{\textcolor{redundantinfocolor}{\MIc\X\Z\Y}}
\newcommand{\RedundantInfo}{\textcolor{redundantinfocolor}{Redundant Information}\xspace}

\colorlet{datanoisecolor}{ik-gray}
\newcommand{\datanoise}{\textcolor{datanoisecolor}{\Hc\X{\Y,\Z}}}

\colorlet{capturedinfocolor}{ik-green}
\newcommand{\capturedinfo}{\textcolor{capturedinfocolor}{\MI\Y\Z}}
\newcommand{\CapturedInfo}{\textcolor{capturedinfocolor}{Preserved Relevant Information}\xspace}

\colorlet{relevantinfocolor}{ik-brightblue}
\newcommand{\relevantinfo}{\textcolor{relevantinfocolor}{\MI\X\Y}}
\newcommand{\RelevantInfo}{\textcolor{relevantinfocolor}{Relevant Information}\xspace}

\colorlet{preservedinfocolor}{ik-brightorange}
\newcommand{\preservedinfo}{\textcolor{preservedinfocolor}{\MI\X\Z}}
\newcommand{\PreservedInfo}{\textcolor{preservedinfocolor}{Preserved Information}\xspace}

\colorlet{rdecoderuncertaintycolor}{ik-darkorange}
\newcommand{\rdecoderuncertainty}{\textcolor{rdecoderuncertaintycolor}{\Hc\Z\Y}}
\newcommand{\RDecoderUncertainty}{\textcolor{rdecoderuncertaintycolor}{Reverse Decoder Uncertainty}\xspace}

\colorlet{decoderuncertaintycolor}{ik-darkblue}
\newcommand{\decoderuncertainty}{\textcolor{decoderuncertaintycolor}{\Hc\Y\Z}}
\newcommand{\DecoderUncertainty}{\textcolor{decoderuncertaintycolor}{Decoder Uncertainty}\xspace}

\colorlet{labelentropycolor}{ik-darkblue}
\newcommand{\labelentropy}{\textcolor{labelentropycolor}{\Entropy{\Y}}}
\newcommand{\LabelEntropy}{\textcolor{labelentropycolor}{Label Entropy}\xspace}

\colorlet{encodingentropycolor}{ik-darkorange}
\newcommand{\encodingentropy}{\textcolor{encodingentropycolor}{\Entropy{\Z}}}
\newcommand{\EncodingEntropy}{\textcolor{encodingentropycolor}{Encoding Entropy}\xspace}

\colorlet{decoderxecolor}{decoderuncertaintycolor}
\newcommand{\decoderXE}{\textcolor{decoderxecolor}{\opH_\theta [ \Y \mathbin{\vert} \Z ]}}
\newcommand{\DecoderCE}{\textcolor{decoderxecolor}{Decoder Cross-Entropy}\xspace}

\colorlet{predictxecolor}{labeluncertaintycolor}
\newcommand{\predictXE}{\textcolor{predictxecolor}{\opH_\theta [ \Y \mathbin{\vert} \X ]}}
\newcommand{\PredictionCE}{\textcolor{predictxecolor}{Prediction Cross-Entropy}\xspace}

\newcommand{\pmencoder}{\probmc{\z}{\x}}

\newcommand{\px}{\probh{\x}}
\newcommand{\pxy}{\probh{\x,\y}}
\newcommand{\pycx}{\probhc{\y}{\x}}

\newcommand{\pxyz}{\prob{\x,\y,\z}}

\newcommand{\pyz}{\prob{\y,\z}}

\newcommand{\pz}{\prob{\z}}
\newcommand{\py}{\probh{\y}}
\newcommand{\pzcx}{\probc{\z}{\x}}

\newcommand{\ddtheta}{\frac{d}{d\theta}}

\newcommand{\DVIB}{DVIB\xspace}
\newcommand{\VCEB}{VCEB\xspace}
\newcommand{\CEB}{CEB\xspace}
\newcommand{\DIB}{DIB\xspace}
\newcommand{\IB}{IB\xspace}

\newcommand{\ZeroEntropyNoise}{zero-entropy noise\xspace}

\newcommand{\EDM}{Entropy Distance Metric\xspace}
\newcommand{\edmp}[2]{EDM \left( #1, #2 \right )}
\newcommand{\edmyz}{\edmp{\Y}{\Z}}

\newcommand{\IWSGD}{Importance Weighted Stochastic Gradient Descent\xspace}
\newcommand{\MSD}{Multi-Sample Dropout\xspace}

\newcommand{\MSA}{\textcolor{rdecoderuncertaintycolor}{\chainedE{\Vert Z \Vert^2}{}}}
\newcommand{\logVarZY}{\textcolor{rdecoderuncertaintycolor}{\log \cVar{\Z}{\Y}}}
\newcommand{\logVarZ}{\textcolor{rdecoderuncertaintycolor}{\log \Var{\Z}}}

\newcommand{\E}[2]{\opE_{#2} \left [ #1 \right ]}
\newcommand{\chainedE}[2]{\opE_{#2} {#1}}

\newcommand{\TK}[1]{{\leavevmode\color{red}TK[[#1]]TK}}
\renewcommand{\TK}[1]{}

\newcommand{\CUT}[1]{}

\AtBeginDocument{\setlength\abovecaptionskip{2mm}}
\AtBeginDocument{\setlength\belowcaptionskip{0pt}}

\AtBeginDocument{\setlength\abovedisplayskip{1mm}}
\AtBeginDocument{\setlength\belowdisplayskip{1mm}}

\AtBeginDocument{\setlength\abovedisplayshortskip{1mm}}
\AtBeginDocument{\setlength\abovedisplayshortskip{1mm}}
\begin{document}

\maketitle

\begin{abstract}
    The Information Bottleneck principle offers both a mechanism to explain how deep neural networks train and generalize, as well as a regularized objective with which to train models. However, multiple competing objectives are proposed in the literature, and the information-theoretic quantities used in these objectives are difficult to compute for large deep neural networks, which in turn limits their use as a training objective. In this work, we review these quantities and compare and unify previously proposed objectives, which allows us to develop surrogate objectives more friendly to optimization without relying on cumbersome tools such as density estimation. We find that these surrogate objectives allow us to apply the information bottleneck to modern neural network architectures. We demonstrate our insights on MNIST, CIFAR-10 and Imagenette with modern DNN architectures (ResNets).
\end{abstract}

\section{Introduction}

The Information Bottleneck (\IB) principle, introduced by \citet{Tishby2000}, proposes that training and generalization in deep neural networks (DNNs) can be explained by information-theoretic principles \citep{Tishby2015,Shwartz-Ziv2017,Achille2017}. 
This is attractive as the success of DNNs remains largely unexplained by tools from computational learning theory \citep{Zhang2019,Bengio2009}.
The \IB principle suggests that learning consists of two competing objectives: maximizing the mutual information between the latent representation and the label to promote accuracy, while at the same time minimizing the mutual information between the latent representation and the input to promote generalization. 
Following this principle, many variations of \IB objectives have been proposed \citep{Alemi2019,Strouse2016,fischer2020modelrobustness,fischer2020conditional,Fisher2019,Gondek,Achille2017}, which, in supervised learning, have been demonstrated to benefit robustness to adversarial attacks \citep{Alemi2019,Fisher2019} and generalization and regularization against overfitting to random labels \citep{Fisher2019}. 

Whether the benefits of training with IB \textit{objectives} are due to the IB \textit{principle}, or some other unrelated mechanism, remains unclear \citep{Saxe2019,AliAmjad2018,Tschannen2019}, suggesting that although recent work has also tied the principle to successful results in both unsupervised and self-supervised learning \citep[among others]{Oord2018,Belghazi2018,Zhang,burgess2018understanding}, our understanding of how \IB objectives affect representation learning remains unclear.

Critical to studying this question is the computation of the information-theoretic quantities\footnote{We shorten these to \emph{information quantities} from now on.} used.
While progress has been made in developing mutual information estimators for DNNs \citep{Poole2019,Belghazi2018,Noshad2019,McAllester2018,Kraskov2003}, current methods still face many limitations when concerned with high-dimensional random variables \citep{McAllester2018} and rely on complex estimators or generative models. This presents a challenge to training with IB objectives.

\begin{figure}[!tbp]
    \begin{minipage}[t]{0.49\textwidth}
        \centering
        \includegraphics[width=0.95\linewidth, clip, trim=0 0 0 5]{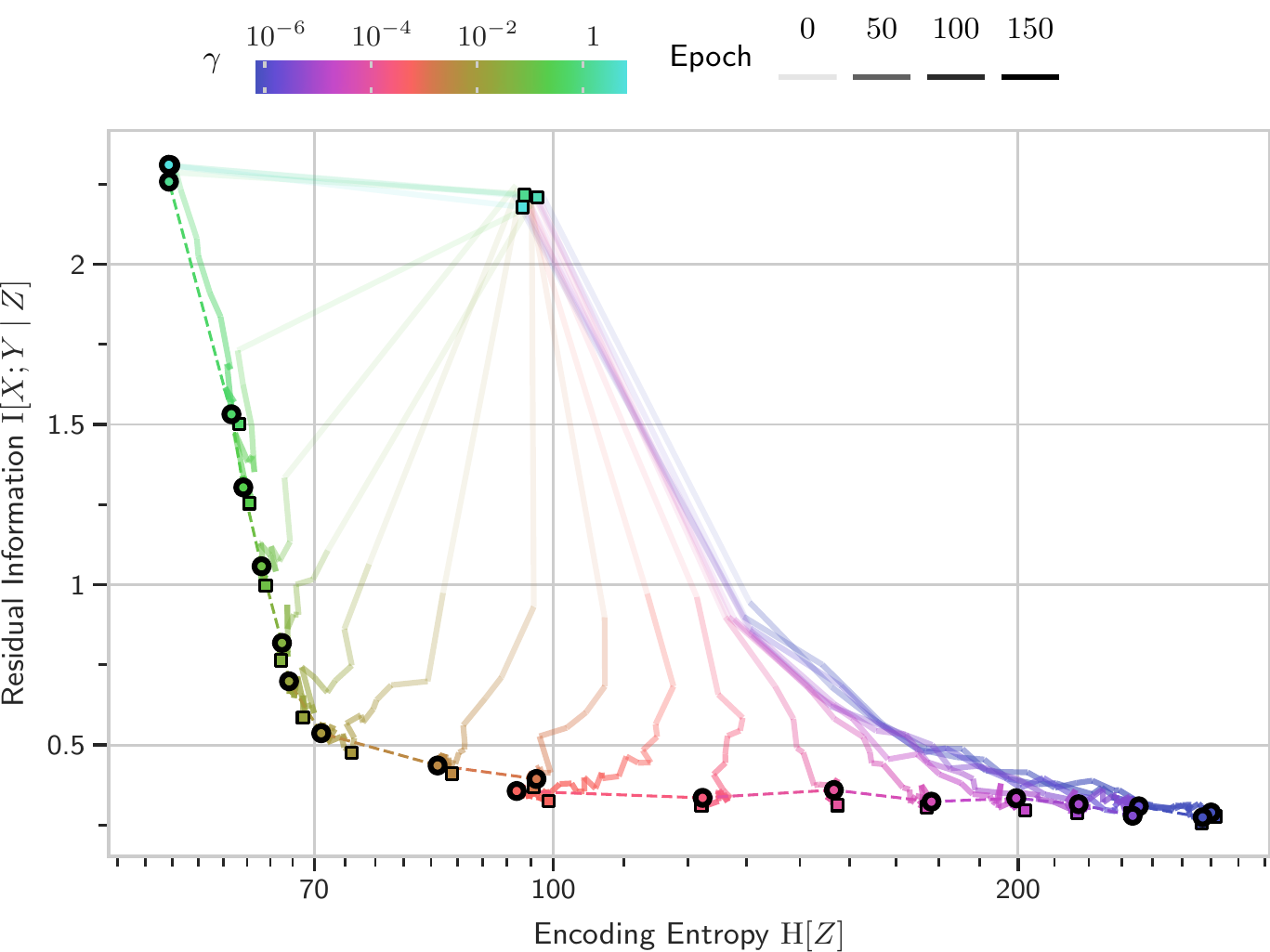}%
        \caption{\emph{Information plane plot of the training trajectories of ResNet18 models with our surrogate objective $\min_\theta \decoderXE + \gamma \MSA$ on \textbf{Imagenette}.} Color shows $\gamma$; transparency the training epoch. %
        Compression (\EncodingEntropy $\downarrow)$ trades-off with test performance (\ResidualInfo $\downarrow$). See \secref{sec:exp surrogate regularizers}.
        }%
        \label{fig:imagenette_beta_trajectories_test}%
    \end{minipage}
	\hfill
	\begin{minipage}[t]{0.49\textwidth}
        \centering %
        \includegraphics[width=0.95\linewidth]{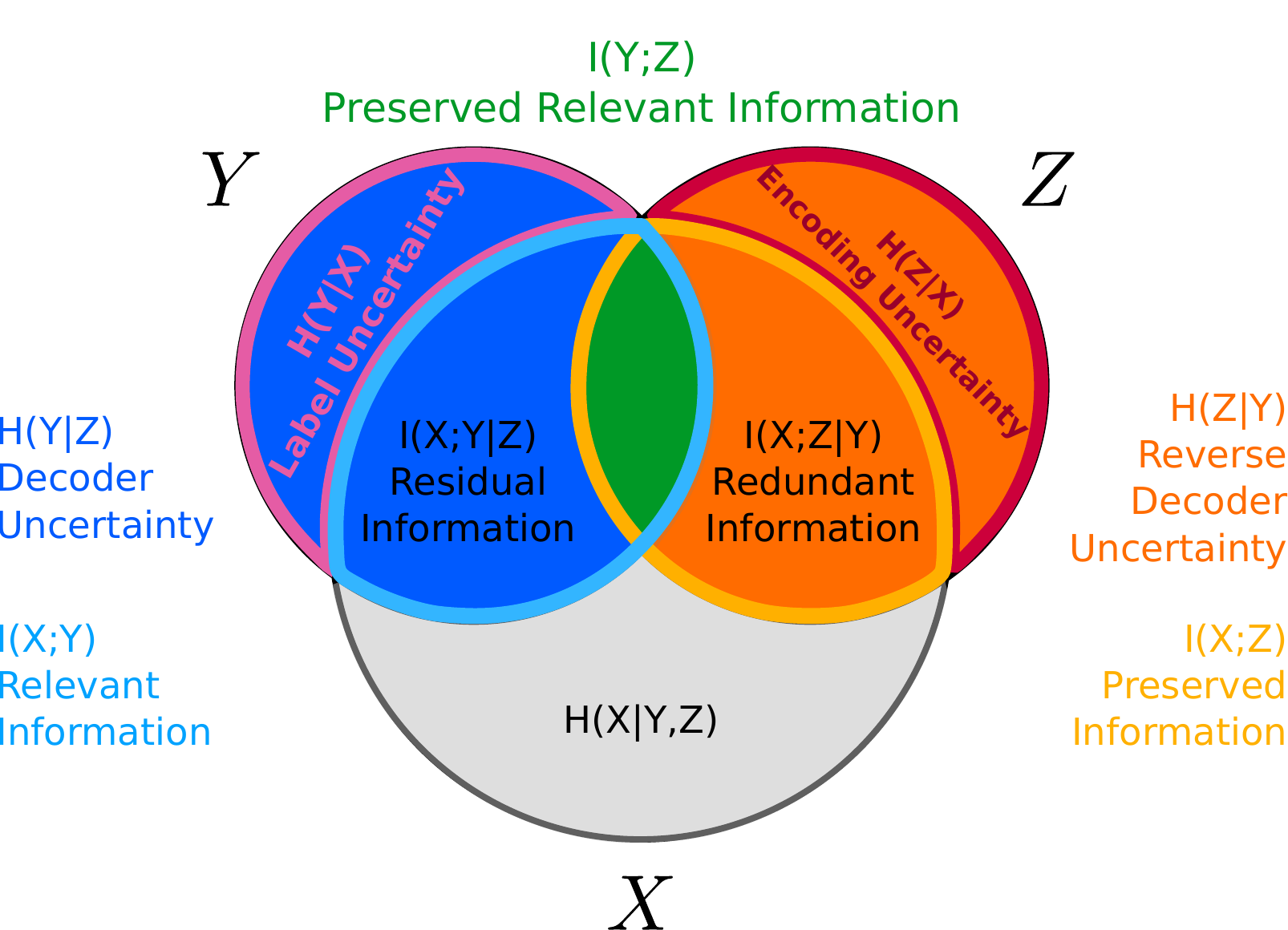} %
        \caption{\emph{Mickey Mouse I-diagram for the information quantities in the model $\prob{\x,\y,\z} = \pxy \pmencoder$.} $\X$ is for the data, $\Y$ for the labels, $\Z$ for the latent encodings. See \secref{sec:information quantities and information diagrams} for more details and \secref{app:mickey mouse intuitions} for a description of all the quantities. $\capturedinfo=\textcolor{capturedinfocolor}{\MIt{\X}{\Y}{\Z}}$ because $\MIc{\Y}{\Z}{\X}=0$. Best viewed in color.}%
        \label{fig:mickeymouse}%
    \end{minipage}
    \vspace{-1.5em}
\end{figure}

In this paper, we analyze information quantities and relate them to surrogate objectives for the \IB principle which are more friendly to optimization, showing that complex or intractable \IB objectives can be replaced with simple, easy-to-compute surrogates that produce similar performance and similar behaviour of information quantities over training.
\textbf{Sections \ref{sec:information quantities and information diagrams} \& \ref{sec:information bottlenecks}} review commonly-used information quantities for which we provide mathematically grounded intuition via information diagrams and unify different \IB objectives by identifying two key information quantities, \DecoderUncertainty $\decoderuncertainty$ and \RDecoderUncertainty $\rdecoderuncertainty$ which act as the main loss and regularization terms in our unified \IB objective.
In particular, \textbf{\Secref{sec:decoder uncertainty}} demonstrates that using the \DecoderUncertainty as a training objective can minimize the training error, and shows how to estimate an upper bound on it efficiently for well-known DNN architectures.
We expand on the findings of \citet{Alemi2019} in their variational \IB approximation and demonstrate that this upper bound is equal to the commonly-used cross-entropy loss\footnote{This connection was assumed without proof by \citet{Achille2017,Achille2018}.} under dropout regularization.
\textbf{\Secref{sec:regularizer estimators}} examines pathologies of differential entropies that hinder optimization and 
proposes adding Gaussian noise to force differential entropies to become non-negative, which leads to new surrogate terms to optimize the \RDecoderUncertainty.
Altogether this leads to simple and tractable surrogate \IB objectives such as the following, which uses dropout, adds Gaussian noise over the feature vectors $f(x; \eta)$, and uses an L2 penalty over the noisy feature vectors:
\begin{equation}
    \min_\theta \E{- \log \probc{\Yhat = y}{z=f_\theta(x; \eta) + \epsilon} 
    + \gamma \left \Vert f_\theta(x; \eta) 
    + \epsilon \right \Vert_2^2}{\substack{x,y \sim \pxy, \epsilon \sim \N\\ \eta \sim \text{dropout mask}}}. \label{eq:explicit_loss}
\end{equation}
\textbf{\Secref{sec:experiments}} describes experiments that validate our insights qualitatively and quantitatively on MNIST, CIFAR-10 and Imagenette, and shows that with objectives like the one in equation \eqref{eq:explicit_loss} we obtain information plane plots (as in \figref{fig:imagenette_beta_trajectories_test}) similar to those predicted by \citet{Tishby2015}. Our simple surrogate objectives thus induce the desired behavior of \IB objectives while scaling to large, high-dimensional datasets. We present \textbf{evaluations on CIFAR-10 and Imagenette images}\footnote{Recently, \citet{fischer2020modelrobustness} report results on CIFAR-10 and ImageNet, see \secref{app:detailed_prior_art_comparison}.}.

Compared to existing work, we show that we can optimize \IB objectives for well-known DNN architectures using standard optimizers, losses and simple regularizers, without needing complex estimators, generative models, or variational approximations. This will allow future research to make better use of IB objectives and study the IB principle more thoroughly.

\section{Background}
\textbf{Information quantities \& information diagrams.}
\label{sec:information quantities and information diagrams}
We denote entropy $\Entropy{\cdot}$, joint entropy $\Entropy{\cdot,\cdot}$, conditional entropy $\Hc{\cdot}{\cdot}$, mutual information $\MI{\cdot}{\cdot}$ and Shannon's information content $\ic{\cdot}$ \citep{Cover1991,Mackay, Shannon}. We will further require the Kullback-Leibler divergence $\kale{\cdot}{\cdot}$ and cross-entropy $\CrossEntropy{\cdot}{\cdot}$. 
The definitions can be found in \secref{app:information quantities}.
We will use differential entropies interchangeably with entropies: equalities between them are preserved in the differential setting, and inequalities will be covered in \secref{sec:continuous entropies}.

Information diagrams (I-diagrams), like the one depicted in \figref{fig:mickeymouse}, clarify the relationship between information quantities: \emph{similar to Venn diagrams, a quantity equals the sum of its parts in the diagram}. 
Importantly, they offer a grounded intuition as \citet{Yeung1991} show that we can define a signed measure $\opmus$ such that information quantities map to abstract sets and are consistent with set operations. We provide details on how to use I-diagrams and what to watch out for in \secref{app:information diagrams}.

\textbf{Probabilistic model.}
\label{sec:probabilistic model}
We will focus on a supervised classification task that makes prediction $\Yhat$ given data $\X$ using a latent encoding $\Z$, while the provided target is $\Y$. We assume categorical $\Y$ and $\Yhat$, and continuous $\X$.
Our probabilistic model based on these assumptions is as follows:
\begin{gather}
    \prob{\x,\y,\z,\yhat} = \pxy \pmencoder \pmdecoder. \label{eq:probmodel}
\end{gather}
Thus, $\Z$ and $\Y$ are independent given $\X$, and $\Yhat$ is independent of $\X$ and $\Y$ given $\Z$.
The data distribution $\pxy$ is only available to us as an empirical sample distribution. $\theta$ are the parameters we would like to learn. 
$\pmencoder$ is the encoder from data $\X$ to latent $\Z$, and
$\pmdecoder$ the decoder from latent $\Z$ to prediction $\Yhat$. Together, $\pmencoder$ and $\pmdecoder$ form the discriminative model $\pmpredict$:
\begin{equation}
    \pmpredict=\chainedE{\pmdecoder}{\pmencoder} \label{eq:discriminative_model}.
\end{equation}
We can derive the cross-entropy loss $\CEpmpredict$ \citep{Solla1988,Hinton1990} %
by minimizing the Kullback-Leibler divergence between the empirical sample distribution $\pxy$ and the parameterized distribution $\pmgenerative$, where we set $\probm{x}=\px$. See \secref{sec:xel}.

\textbf{Mickey Mouse I-diagram.}
\label{sec:mickey mouse}
The corresponding I-diagram for $\X$, $\Y$, and $\Z$ is depicted in \figref{fig:mickeymouse}.
As some of the quantities have been labelled before, we try to follow conventions and come up with consistent names otherwise.
\emph{\Secref{app:mickey mouse intuitions} provides intuitions for these quantities, and \secref{app:all definitions} lists all definitions and equivalences explicitly}.
For categorical $\Z$,
all the quantities in the diagram are positive, which allows us to read off inequalities from the diagram:
only $\textcolor{capturedinfocolor}{\MIt\X\Y\Z}$ could be negative, but as $\Y$ and $\Z$ are independent given $\X$, we have $\MIc{\Y}{\Z}{\X} = 0$, and
\begin{math}
    \textcolor{capturedinfocolor}{\MIt{\X}{\Y}{\Z}}=\capturedinfo - \MIc{\Y}{\Z}{\X} = \capturedinfo \ge 0.
\end{math}
\Secref{sec:continuous entropies} investigates how to preserve inequalities for continuous $\Z$.
\section{Surrogate \IB{} \& \DIB{} objectives}
\label{sec:information bottlenecks}
\label{sec:different perspective on IB and DIB}

\subsection{IB Objectives} 
\label{sec:ib background}
\citet{Tishby2000} introduce the \IB objective as a relaxation of a constrained optimization problem: minimize the mutual information between the input $X$ and its latent representation $Z$ while still accurately predicting $Y$ from $Z$. An analogous objective which yields deterministic $Z$, the Deterministic Information Bottleneck (\DIB) was proposed by \citet{Strouse2016}. Letting $\beta$ be a Lagrange multiplier, we arrive at the \IB and \DIB objectives:
\begin{equation}
    \min \preservedinfo - \beta \capturedinfo \; \quad \text{ for IB, and } \quad \min \encodingentropy 
    - \beta \capturedinfo  \text{ for DIB.}
\end{equation} 
This principle can be recast as a generalization of finding minimal sufficient statistics for the labels given the data \citep{Shamir2010,Tishby2015,Fisher2019}:
it strives for minimality and sufficiency of the latent $\Z$. Minimality is achieved by minimizing the \PreservedInfo $\preservedinfo$; while sufficiency is achieved by maximizing the \CapturedInfo $\capturedinfo$.
We defer an in-depth discussion of the \IB principle to the appendix \Secref{app:optimization goals}. We discuss the several variants of \IB objectives, and justify our focus on \IB and \DIB, in \Secref{app:ib objectives}.

The information quantities that appear in the \IB objective are not tractable to compute for the representations learned by many function classes of interest, including neural networks; for example, \citet{Strouse2016} only obtain an analytical solution to their Deterministic Information Bottleneck (DIB) method for the tabular setting. \citet{Alemi2019} address this challenge by constructing a variational approximation of the \IB objective, but their approach has not been applied to more complex datasets than MNIST variants. \citet{Belghazi2018} use a separate statistics network to approximate the mutual information, a computationally expensive strategy that does not easily lend itself to optimization.

In this section, we introduce and justify tractable surrogate losses that are easier to apply in common deep learning pipelines, and which can be scaled to large and high-dimensional datasets.
We begin by proposing the following reformulation of \IB and \DIB objectives.
\begin{observation}
\label{obs:rewritten_objectives}    
For \IB, we obtain 
\begin{align}
    \argmin \preservedinfo - \beta \capturedinfo &= \argmin \decoderuncertainty + \beta' \underbrace{\redundantinfo}_{=\rdecoderuncertainty - \encodinguncertainty}, \\
\intertext{
and, for \DIB,
}
    \argmin \encodingentropy - \beta \capturedinfo &= \argmin \decoderuncertainty + \beta' \rdecoderuncertainty
    = \argmin \decoderuncertainty + \beta'' \encodingentropy
\end{align}
with $\beta' := \frac{1}{\beta - 1} \in [0,\infty)$ and $\beta'' := \frac{1}{\beta}\in [0,1)$.
The derivation can be found in \secref{app:different perspective on IB and DIB}.
\end{observation}
In the next sections, we show that \DecoderUncertainty $\decoderuncertainty$ provides a loss term, which minimizes the training error, and \DIB's \RDecoderUncertainty $\rdecoderuncertainty$ and \IB's \RedundantInfo $\redundantinfo$, respectively, provide a regularization term, which helps generalization. 
\CUT{
We implicitly limit ourselves to $\beta \ge 1$ (and allow for $\beta \to \{1, \infty\}$).
For $\beta < 1$, we would be maximizing the \DecoderUncertainty{}, which does not make sense: the trivial solution to this is one where $\Z$ contains no information on $\Y$\footnote{In the case of \DIB{}, the trivial solution is to map every input deterministically to a single value;
whereas for \IB{}, we only minimize the \RedundantInfo{}, and an optimal solution has no information on $\Y$ while being free to contain noise.}.
}
Another perspective can be found by relating the objectives to the Entropy Distance Metric introduced by \citet{Mackay}, which we detail in \secref{app:EDM}.

\subsection{\DecoderUncertainty{} \texorpdfstring{$\decoderuncertainty$}{}}
\label{sec:decoder uncertainty}
The \DecoderUncertainty{} $\decoderuncertainty$ is the first term in our reformulated IB and DIB objectives, and captures the data fit component of the \IB principle. This quantity is not easy to compute directly for arbitrary representations $Z$, so we turn our attention to two related entities instead, where we use $\theta$ as subscript to mark dependence on the model:
the \PredictionCE, denoted $\predictXE$ (more commonly known as the model's cross-entropy loss; see \secref{sec:xel}), and the \DecoderCE, denoted $\decoderXE$.
Noting that $\ic{x} = -\ln x$, we define these terms as follows:
\begin{align}
    & \predictXE := 
    \CEpmpredict 
    = \chainedE{\ic{\chainedE{\pmdecoderce}{\pmencoder}}}{\pxy}
    \\
    & \decoderXE := 
    \CEpmdecoder
    = \chainedE{\chainedE{\ic{\pmdecoderce}}{\pmencoder}}{\pxy}.
\end{align}
Jensen's inequality yields $\predictXE \le \decoderXE$, with equality iff $Z$ is a deterministic function of $X$. The notational similarity\footnote{This notation is compatible with $\mathcal{V}$-Entropy introduced by \citet{Xu2020}.} between $\decoderXE$ and $\decoderuncertainty$ is deliberately suggestive: this cross-entropy bounds the conditional entropy $\decoderuncertainty$, as characterized in the following proposition. 
\begin{observation}
    The \DecoderCE provides an upper bound on the \DecoderUncertainty:
    \begin{align}
        \decoderuncertainty \le \decoderuncertainty + \kale{\pdecoder}{\pmdecoder} = \decoderXE,  \label{eq:decoder bound with KALE} 
    \end{align}
    and further bounds the training error:
    \begin{align}
        \prob{\text{``$\Yhat$ is wrong''}} & \le 1 - e^{-\decoderXE}
        = 1 - e^{-\left (\decoderuncertainty + \kale{\pdecoder}{\pmdecoder} \right )}.
    \end{align}
    Likewise, for $\predictXE$ and $\labeluncertainty$. See \secref{app:minimizing the training error} for a derivation.
\end{observation} 
Hence, by bounding $\kale{\pdecoder}{\pmdecoder}$, we can obtain a bound for the training error in terms of $\decoderuncertainty$. We examine one way of doing so by using optimal decoders $\pmdecoder := \pdecoderhat$ for the case of categorical $\Z$ in \secref{app:optimal decoders}.

\citet{Alemi2019} use the \DecoderCE bound in equation \eqref{eq:decoder bound with KALE} to variationally approximate $\pdecoder$. We make this explicit by applying the reparameterization trick to rewrite the latent $z$ as a parametric function of its input $x$ and some independent auxiliary random variable $\eta$, i.e. $f_\theta(x, \eta) \overset{D}{=} z \sim \pmencoder$, yielding 
\begin{align}
    \decoderuncertainty \le \decoderXE 
    = \chainedE{\chainedE{\ic{\probmc{\Yhat=y}{z=f_\theta(x; \eta)}}}{\prob{\eta}}}{\pxy}.
    \label{eq:reparam_decoder_xe}
\end{align}
Equation \eqref{eq:reparam_decoder_xe} can be applied to many forms of stochastic regularization that turn deterministic models into stochastic ones, in particular dropout. This allows us to use modern DNN architectures as stochastic encoders.

\textbf{Dropout regularization}
\label{sec:decoderxe and dropout}
When we interpret $\eta$ as a sampled dropout mask for a DNN, DNNs that use dropout regularization \citep{Srivastava2014}, or variants like DropConnect \citep{Wan2013}, fit the equation above as stochastic encoders. Monte-Carlo dropout \citep{Gal2015}, for example, even specifically estimates the predictive mean $\pmpredict$ from equation \eqref{eq:discriminative_model}. The following result extends the observation by \citet{Burda} that sampling yields an unbiased estimator for the \DecoderCE $\decoderXE$, while it only yields a biased estimator for the \PredictionCE $\predictXE$ (which it upper-bounds). 

\begin{corollary}
Let $x, y, z$ and $f_\theta$ be defined as previously, with $\eta$ a sampled stochastic dropout mask. Then $\ic{\probmc{\Yhat=y}{z=f_\theta(x; \eta)}}$ evaluated for a single sample $\eta$ is an unbiased estimator of the \DecoderCE $\decoderXE$, and an estimator of an upper bound on the \PredictionCE $\predictXE$.
\end{corollary}
This distinction between the \DecoderCE and the \PredictionCE has been observed in passing in the literature, but not made explicit.
Multi-sample approaches like \MSD \citep{Inoue2019}, for example, optimize $\decoderXE$, while \IWSGD \citep{Noh2017} optimizes $\predictXE$. 
\citet{Dusenberry2020} observe empirically in the different context of rank-1 Bayesian Neural Networks that optimizing $\decoderXE$ instead of $\predictXE$ is both easier and also yields better generalization performance (NLL, accuracy, and ECE), while they also put forward an argument for why the stochastic gradients for $\decoderXE$ might benefit from lower variance. 
We empirically compare training with either cross-entropy in \secref{sec:exp for cross-entropies} and show results in \figref{fig:continuous_bounds_all} in the appendix.
We conclude this section by highlighting that $\decoderuncertainty$ is therefore already minimized in modern DNN architectures that use dropout together with a cross-entropy loss. This means that, at least for one half of our reformulation of the \IB objective, we can apply off-the-shelf, scalable objectives and optimizers for its minimization.

\subsection{Surrogates for the regularization terms}
\label{sec:regularizer estimators}

In the previous section, we have examined how to tractably estimate the error minimization term $\decoderuncertainty$.
In this section, we will examine tractable optimization of the regularization terms $\rdecoderuncertainty$ and $\redundantinfo$, respectively. We discuss how to minimize entropies meaningfully and show how this unifies \DIB and \IB via the inequality $\redundantinfo \le \rdecoderuncertainty \le \encodingentropy$ before providing tractable upper-bounds for $\rdecoderuncertainty$ and $\encodingentropy$. 

\begin{figure}[!tbp]
	\begin{minipage}[t]{0.49\textwidth}
        \centering
        \includegraphics[width=0.95\linewidth, clip, trim=0 0 0 0]{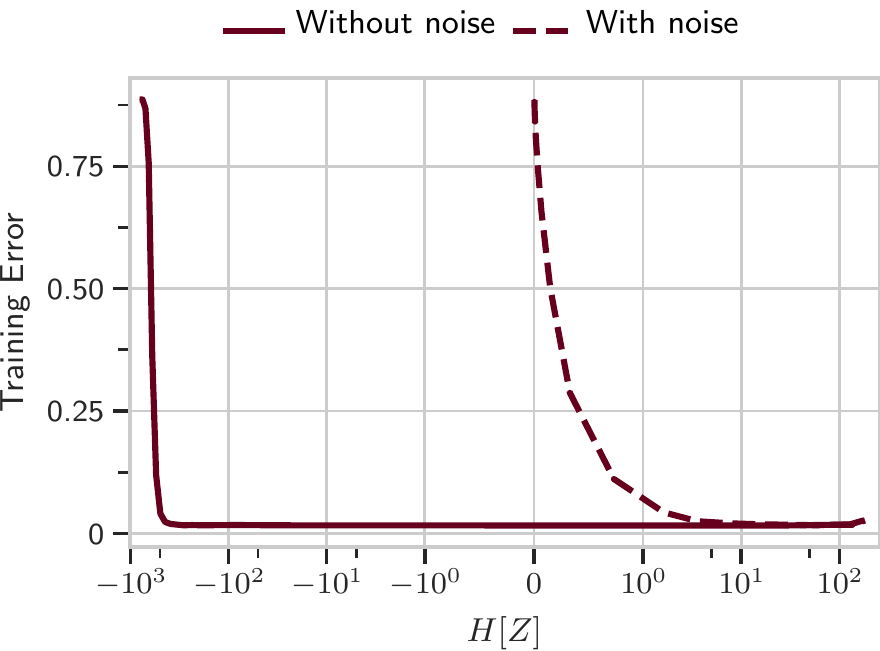}
        \caption{\emph{Decreasing the entropy of a noise-free latent does not affect the training error.} (Though floating-point issues start affecting it negatively eventually.) When adding \ZeroEntropyNoise, the error rate increases as the entropy approaches zero. See \secref{sec:continuous entropies} and \ref{sec:exp minimizing entropy} for more details.
        }
        \label{fig:entropy_minimization_and_noise}
	\end{minipage}
	\hfill
	\begin{minipage}[t]{0.49\textwidth}
        \includegraphics[width=0.95\linewidth, clip, trim=0 0 0 0]{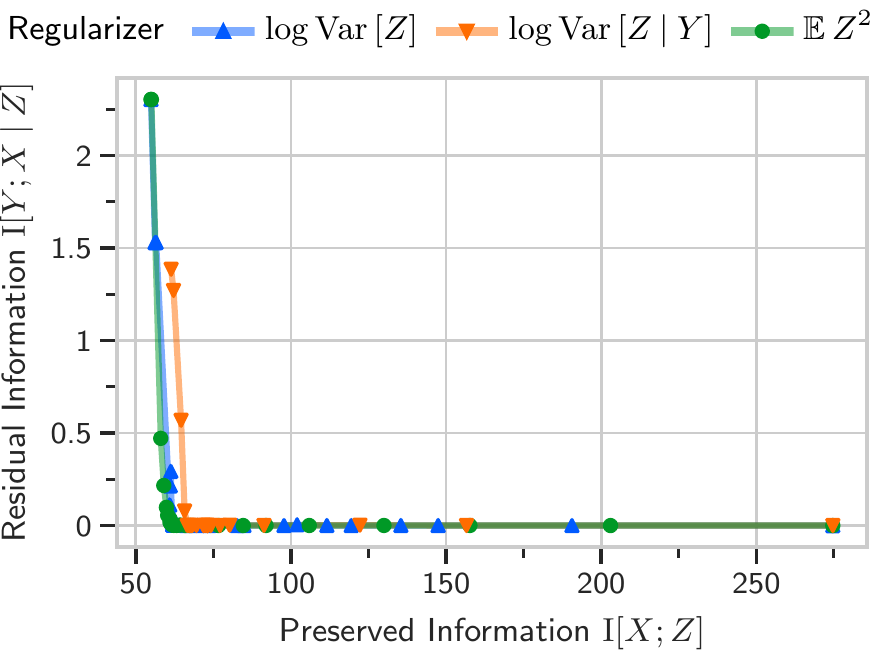}
        \caption{\emph{Information plane plot of the latent $\Z$ similar to \citet{Tishby2015} but using a \emph{ResNet18} model on \emph{CIFAR-10} using the different regularizes from \secref{sec:summary surrogate objectives}.} A larger version can be found in \figref{fig:cifar10_kraskov_IBP_big}.
        See \secref{sec:exp info plane plots} for more details. Best viewed in color.}
        \label{fig:cifar10_beta_ib}
    \end{minipage}
    \vspace{-1.5em}
\end{figure}

\textbf{Differential entropies}
\label{sec:continuous entropies}
In most cases, the latent $\Z$ is a continuous random variable in many dimensions.
Unlike entropies on discrete probability spaces, differential entropies defined on continuous spaces are not bounded from below.
This means that the \DIB objective is not guaranteed to have an optimal solution and allows for pathological optimization trajectories in which the variance of the latent $Z$ can be scaled to be arbitrarily small, achieving arbitrarily high-magnitude negative entropy.
We provide a toy experiment demonstrating this in \secref{sec:exp minimizing entropy}.

Intuitively, one can interpret this issue as being allowed to encode information in an arbitrarily-small real number using infinite precision, similar to arithmetic coding \citep{Mackay, Shwartz-Ziv2017}\footnote{Conversely, \citet{Mackay} notes that without upper-bounding the ``power" $\chainedE{\Z^2}{\pz}$, all information could be encoded in a single very large integer.}. 
In practice, due to floating point constraints, optimizing DIB naively will invariably end in garbage predictions and underflow as activations approach zero. It is therefore not desirable for training. This is why \citet{Strouse2016} only consider analytical solutions to \DIB by evaluating a limit for the tabular case. \citet{Mackay} proposes the introduction of noise to solve this issue in the application of continuous communication channels. 

However, here we propose adding specific noise to the latent representation to lower-bound the conditional entropy of $Z$,
which allows us to enforce non-negativity across all \IB information quantities as in the discrete case and transport inequalities to the continuous case:
for a continuous $\hat{\Z} \in \reals^k$ and independent noise $\eps$, we set $\Z := \hat{\Z}+\epsilon$; the differential entropy then satisfies $\encodingentropy = \Entropy{\hat{\Z}+\epsilon} \geq \Entropy{\epsilon}$; and by using \emph{\ZeroEntropyNoise} $\epsilon \sim \normaldist{0}{\frac{1}{2\pi e} I_k}$ specifically, we obtain
\begin{math}
    \encodingentropy \geq \Entropy{\epsilon} = 0.
\end{math}
\begin{observation}
After adding \ZeroEntropyNoise, the inequality $\redundantinfo \le \rdecoderuncertainty \le \encodingentropy$ also holds for continuous $\Z$, and we can minimize $\redundantinfo$ in the \IB objective by minimizing $\rdecoderuncertainty$ or $\encodingentropy$, similarly to the \DIB objective. 
\end{observation}
Strictly speaking, \ZeroEntropyNoise is not necessary for optimizing the bounds: any Gaussian noise is sufficient, but \ZeroEntropyNoise is aesthetically appealing as it preserves inequalities from the discrete setting. In a sense, this propostion bounds the \IB objective by the \DIB objective. 
However, adding noise changes the optimal solutions: whereas \DIB in \citet{Strouse2016} leads to hard clustering in the limit, adding noise leads to soft clustering when optimizing the \DIB objective, as is the case with the \IB objective. We show 
in \secref{sec:entropy minimization with gaussian noise} that minimizing the \DIB objective with noise leads to soft clustering (for the case of an otherwise deterministic encoder).
Altogether, in addition to \citet{Shwartz-Ziv2017}, we argue that noise is essential to obtain meaningful differential entropies and to avoid other pathological cases as described further in \secref{sec:pathologies}.
\label{sec:upper-bounding rdecoderuncertainty}%
It is not generally possible to compute $\rdecoderuncertainty$ exactly for continuous latent representations $Z$, but we can derive an upper bound.
The maximum-entropy distribution for a given covariance matrix $\Sigma$ is a Gaussian with the same covariance. 
\begin{observation}
\label{obs:simple_approximators}
    The \RDecoderUncertainty{} can be approximately bounded using the empirical variance $\hcVar{\Z_i}{\y}$:
    \begin{align}
        \rdecoderuncertainty &\le \chainedE{ \sum_i \tfrac{1}{2} \ln (2 \pi e \; {\cVar{\Z_i}{\y})}}{\py} \approx \chainedE{ \sum_i \tfrac{1}{2} \ln (2 \pi e \; {\hcVar{\Z_i}{\y})}}{\py},
    \end{align}
    where $\Z_i$ are the individual components of $\Z$. $\encodingentropy$ can be bounded similarly. More generally, we can create an even looser upper bound by bounding the mean squared norm of the latent:
    \begin{align}
        \MSA \le C' \Rightarrow \rdecoderuncertainty \le \encodingentropy \le C,
    \end{align}
    with $C':=\frac{k e^{2C/k}}{2 \pi e}$ for $\Z \in \reals^k$.
    See \secref{sec:math for upperbound} for proof.
\end{observation}

\textbf{Surrogate objectives} 
\label{sec:summary surrogate objectives}
These surrogate terms provide us with three different upper-bounds that we can use as surrogate regularizers. We refer to them as: conditional log-variance regularizer ($\logVarZY$), log-variance regularizer ($\logVarZ$) and activation $L_2$ regularizer ($\MSA$). 
We can now propose the main results of this paper: \IB surrogate objectives that reduce to an almost trivial implementation using the cross-entropy loss and one of the regularizers above while adding \ZeroEntropyNoise to the latent $\Z$.

\begin{theorem}
Let $\Z$ be obtained by adding a single sample of \ZeroEntropyNoise to a single sample of the output $z$ of the stochastic encoder. Then each of the following objectives is an estimator of an upper bound on the IB objective. In particular, for the surrogate objective $\MSA$, we obtain:
\begin{align}
    \min\CEpmdecoder + \gamma \|z\|^2; \label{eq:objective}
\end{align}
for $\logVarZY$:
\begin{align}
\label{eq:varobjective}
    \min \CEpmdecoder + \gamma \chainedE{ \sum_i \frac{1}{2} \ln (2 \pi e \; {\hcVar{\Z_i}{\y})}}{\py}; %
\end{align}
and for $\logVarZ$:
\begin{align}
    \min \CEpmdecoder + \gamma  \sum_i \frac{1}{2} \ln (2 \pi e \; {\hVar{\Z_i}}). \label{eq:logvarobjective}
\end{align} 
\end{theorem}
For the latter two surrogate regularizers, we can relate their coefficient $\gamma$ to $\beta'$, $\beta''$ and $\beta$ from \secref{sec:different perspective on IB and DIB}. However, as regularizing $\MSA$ does not approximate an entropy directly, its coefficient does not relate to the Lagrange multiplier of any fixed \IB objective. 
We compare the performance of these objectives in \secref{sec:exp surrogate regularizers}.

\section{Experiments}
\label{sec:experiments}

\begin{figure}[!tbp]
    \begin{minipage}[t]{\textwidth}
        \centering
        \includegraphics[width=0.6\textwidth, clip, trim=0 160 0 0]{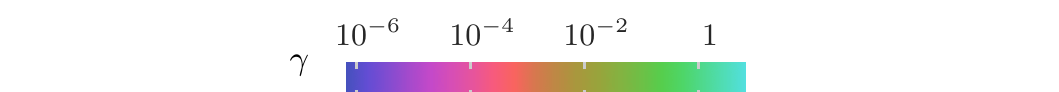}
    \end{minipage}
	\begin{minipage}[t]{0.49\textwidth}
        \centering
        \includegraphics[width=\linewidth, clip, trim=0 0 0 35]{iclr_images/main_robustness.pdf}
        \subcaption{\emph{Robustness for different attack strengths $\eps$.} The dashed black line represents a model trained only with cross-entropy and no noise injection. We see that models trained with the surrogate IB objective (colored by $\gamma$) see improved robustness over a model trained only to minimize the cross-entropy training objective (shown in black) while the models regularized with weight-decay actually perform worse.}
	\end{minipage}
	\hfill
    \begin{minipage}[t]{0.49\textwidth}
        \includegraphics[width=0.95\linewidth, clip, trim=0 0 0 30]{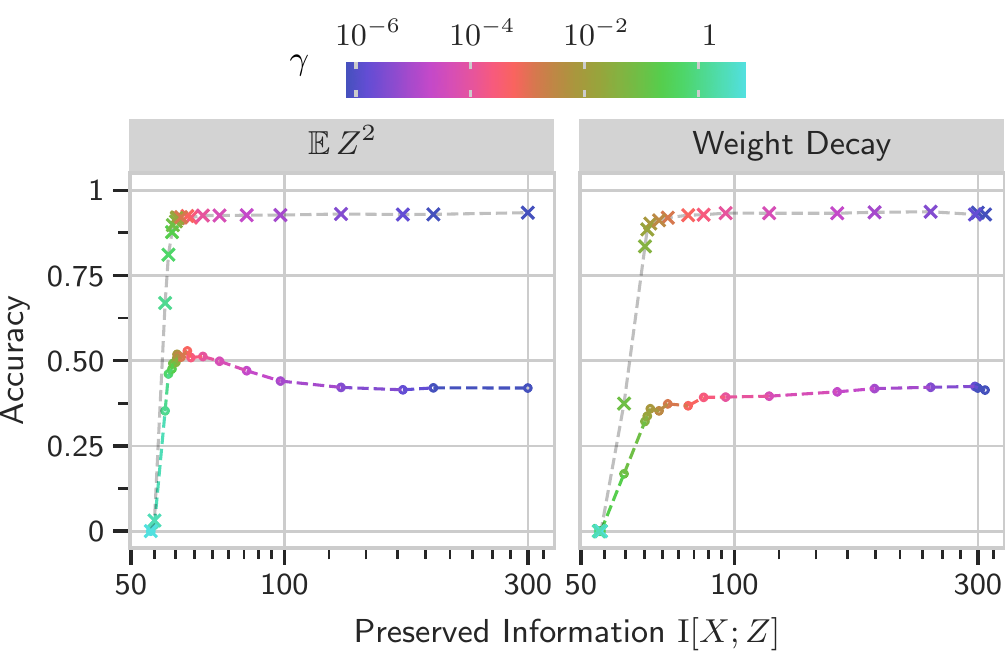}
        \subcaption{\emph{Average robustness over $\eps \in [0, 0.1]$ compared to normal accuracy for different amounts of \PreservedInfo.} $\circ$ markers show robustness. $\times$ markers show the normal accuracy. We see that robustness depends on the \PreservedInfo. If the latent is compressesed too much, robustness (and accuracy) are low. If the latent is not compressed enough, robustness and thus generalization suffer.}
    \end{minipage}
    \caption{\emph{Adversarial robustness of ResNet18 models trained on CIFAR-10 with surrogate objectives in comparison to regularization with L2 weight-decay as non-IB method.} The robustness is evaluated using FGSM, PGD, DeepFool and BasicIterative attacks of varying $\epsilon$ values.}
    \label{fig:adversarial_robustness}
    \vspace{-1.5em}
\end{figure}

We now provide empirical verification of the claims made in the previous sections. Our goal in this section is to highlight two main findings: first, that our surrogate objectives obtain similar behavior to what we expect of exact IB objectives with respect to their effect on robustness to adversarial examples. In particular, we show that our surrogate \IB objectives improve adversarial robustness compared to models trained only on the cross-entropy loss, consistent with the findings of \citet{Alemi2019}.
Second, we show the effect of our surrogate objectives on information quantities during training by plotting information plane diagrams, demonstrating that models trained with our objectives trade off between $I[X;Z]$ and $I[Y;Z]$ as expected. We show this by recovering information plane plots similar to the ones in \citet{Tishby2015} and qualitatively examine the optimization behavior of the networks through their training trajectories.
We demonstrate the scalability of our surrogate objectives by applying our surrogate IB objectives to the CIFAR-10 and Imagenette datasets, high-dimensional image datasets.

For details about our experiment setup, DNN architectures, hyperparameters and additional insights, see \secref{app:experiment_details}. In particular, empirical quantification of our observations on the relationship between the \DecoderCE loss and the \PredictionCE are deferred to the appendix due to space limitations as well as the description of the toy experiment that shows that minimizing $\rdecoderuncertainty$ for continuous latent $\Z$ without adding noise does not constrain information meaningfully and that adding noise solves the issue as detailed in \secref{sec:continuous entropies}. 

\textbf{Robustness to adversarial attacks}
\label{sec:robustness}
\citet{Alemi2019} and \citet{fischer2020modelrobustness} observe that their \IB objectives lead to improved adversarial robustness over standard training objectives. We perform a similar evaluation to see whether our surrogate objectives also see improved robustness. We train a fully-connected residual network on CIFAR-10 for a range of regularization coefficients $\gamma$ using our $\MSA$ surrogate objective; we then compare against a similar regularization method that does not have an information-theoretic interpretation: L2 weight-decay. We inject \ZeroEntropyNoise in both cases. 
After training, we evaluate the models on adversarially perturbed images using the FGSM \citep{szegedy2013intriguing}, PGD \citep{madry2018towards}, BasicIterative \citep{kurakin2016adversarial} and DeepFool \citep{moosavi2016deepfool} attacks for varying levels of the perturbation magnitude parameter $\epsilon$. We also compare to a simple unregularized cross-entropy baseline (black dashed line). To compute overall robustness, we use each attack in turn and only count a sample as robust if it defeats them all.
As depicted in \figref{fig:adversarial_robustness}, we find that our surrogate objectives yield significantly more robust models while obtaining similar test accuracy on the unperturbed data whereas weight-decay regularization reduces robustness against adversarial attacks.
Plots for the other two regularizers can be found in the appendix in \figref{fig:appendix_robustness} and \figref{fig:appendix_integrated_robustness}.

\textbf{Information plane plots for CIFAR-10}
\label{sec:exp surrogate regularizers}
To compare the different surrogate regularizers, we again use a ResNet18 model on CIFAR-10 with \ZeroEntropyNoise added to the final layer activations $\Z$, with $K=256$ dimensions, as an encoder and add a single $K \times 10$ linear unit as a decoder.
We train with the surrogate objectives from \secref{sec:regularizer estimators} for various $\gamma$, chosen in logspace from different ranges to compensate for their relationship to $\beta$ as noted in \secref{sec:summary surrogate objectives}:
for $\logVarZ$, $\gamma \in [10^{-5}, 1]$;
for $\logVarZY$, $\gamma \in [10^{-5}, 10]$;
and for $\MSA$, by trial and error, $\gamma \in [10^{-6},10]$.
We estimate information quantities using the method of \citet{Kraskov2003}.

\begin{figure}[!tbp]
    \centering
    \includegraphics[width=0.95\linewidth, clip, trim=0 0 0 5]{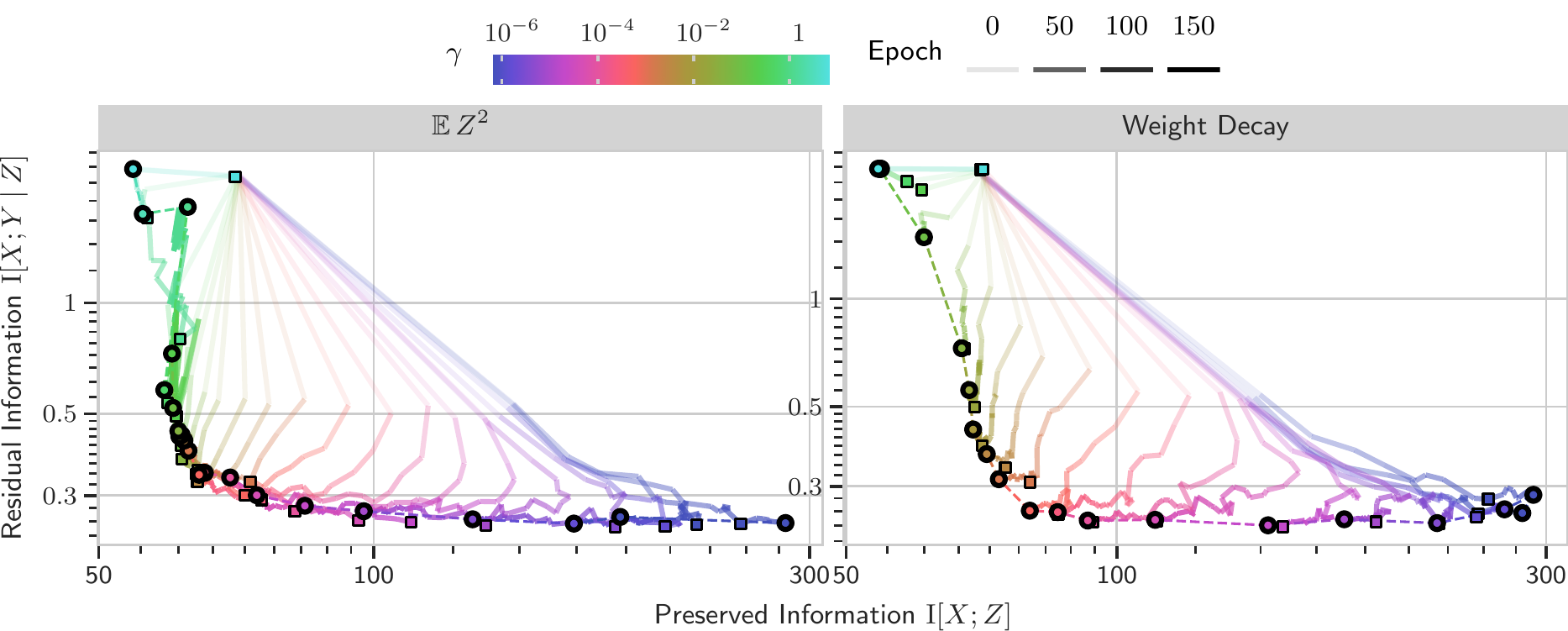}%
    \caption{\emph{Information plane plot of the training trajectories of ResNet18 models with the $\MSA$ surrogate objective or L2 weight-decay on CIFAR-10.} The color shows $\gamma$; the transparency the training epoch. %
    Compression (\PreservedInfo $\downarrow)$ trades-off with performance (\ResidualInfo $\downarrow$). See \secref{sec:exp surrogate regularizers}. While the trajectories are similar, robustness is very different, see \figref{fig:adversarial_robustness}.
    }%
    \label{fig:cifar10_beta_trajectories}%
    \vspace{-1.5em}
\end{figure}

\label{sec:exp info plane plots}
\Figref{fig:cifar10_beta_trajectories} shows an information plane plot for regularizing with $\MSA$ for different $\gamma$ over different epochs for the training set. 
Similar to \citet{Shwartz-Ziv2017}, we observe that there is an initial expansion phase followed by compression.
The jumps in performance (reduction of the \ResidualInfo) are due to drops in the learning rate. 
In \figref{fig:cifar10_beta_ib}, we can see that the saturation curves for all 3 surrogate objectives qualitatively match the predicted curve from \citet{Tishby2015}. 
\Figref{fig:cifar10_kraskov_IBP_big} shows the difference between the regularizers more clearly, and
\figref{appfig:cifar10_beta_trajectories_no_dropout} shows the training trajectories for all three regularizers. More details in \secref{app:measuring information quantities}.

\textbf{Information plane plots for Imagenette} %
To show that our surrogate objectives also scale up to larger datasets, we run a similar experiment on Imagenette \citep{imagewang}, which is a subset of ImageNet with 10 classes with $224 \times 224 \times 3=1.5\times10^5$ input dimensions, and on which we obtain 90\% test accuracy. See the \figref{fig:imagenette_beta_trajectories_test}, which shows the trajectories on the test set. We obtain similar plots to the ones obtained for CIFAR-10, \emph{showing that our surrogate objectives scale well to higher-dimensional datasets despite their simplicity.}

\section{Conclusion}

The contributions of this paper have been threefold:
First, we have proposed simple, tractable training objectives which capture many of the desirable properties of \IB methods while also scaling to problems of interest in deep learning. 
For this we have introduced implicit stochastic encoders, e.g. using dropout, and compared multi-sample dropout approaches to identify the one that approximates the \DecoderUncertainty $\decoderXE$, relating them to the cross-entropy loss that is commonly used for classification problems. This widens the range of DNN architectures that can be used with \IB objectives considerably.
We have demonstrated that our objectives perform well for practical DNNs without cumbersome density models.
Second, we have motivated our objectives by providing insight into limitations of \IB training, demonstrating how to avoid pathological behavior in \IB objectives, and  
by endeavouring to provide a unifying view on \IB approaches.
Third, we have provided mathematically grounded intuition by using I-diagrams for the information quantities involved in \IB, shown common pitfalls when using information quantities and how to avoid them, and examined how the quantities relate to each other.
Future work investigating the practical constraints on the expressivity of a given neural network may provide further insight into how to measure compression in neural networks. %
Moreover, the connection to Bayesian Neural Networks remains to be explored.

\section*{Acknowledgements}

The authors want to thank Sam Ballard for revamping the color scheme, and Lewis Smith and Tim Rudner for fruitful discussions. We would also like to thank OATML in general for their feedback at several stages of the project.
AK is supported by the UK EPSRC CDT in Autonomous Intelligent Machines
and Systems (grant reference EP/L015897/1). 
CL is supported by an Open Philanthropy Fund AI Fellowship.

\clearpage

\bibliographystyle{plainnat}
\bibliography{submission}
\clearpage
\appendix
\counterwithin{figure}{section}

\section{Information quantities \& information diagrams}
\label{sec:information quantities and information diagrams long}

Here we introduce notation and terminology in greater detail than in the main paper. We review well-known information quantities and provide more details on using information diagrams \citep{Yeung1991}.

\subsection{Information quantities}
\label{app:information quantities}
We denote entropy $\Entropy{\cdot}$, joint entropy $\Entropy{\cdot,\cdot}$, conditional entropy $\Hc{\cdot}{\cdot}$, mutual information $\MI{\cdot}{\cdot}$ and Shannon's information content $\ic{\cdot}$ following \citet{Cover1991,Mackay, Shannon} :
\begin{align*}
    \ic{x} &= -\ln x \\
    \Entropy{X} &= \chainedE{\ic{\prob{x}}}{\prob{x}} \\
    \Entropy{X, Y} &= \chainedE{\ic{\prob{x, y}}}{\prob{x, y}} \\
    \Hc{X}{Y} &= \Entropy{X, Y} - \Entropy{Y} \\
    &= \chainedE{\Hc{X}{y}}{\prob{y}} = \chainedE{\ic{\probc{x}{y}}}{\prob{x, y}} \\
    \MI{X}{Y} &= \Entropy{X} + \Entropy{Y} - \Entropy{X, Y} \\
    &= \chainedE{\ic{\tfrac{\prob{x}\prob{y}}{\prob{x,y}}}}{\prob{x,y}} \\
    \MIc{X}{Y}{Z} &= \Hc{X}{Z} + \Hc{Y}{Z} - \Hc{X, Y}{Z},
\end{align*}
where $X, Y, Z$ are random variables and $x, y, z$ are outcomes these random variables can take. 

We use differential entropies interchangeably with entropies. We can do so because equalities between them hold as can be verified by symbolic expansions. For example,
\begin{align*}
    & \Entropy{X, Y} = \Hc{X}{Y} + \Entropy{Y} \\
    \Leftrightarrow & \chainedE{\ic{\prob{x, y}}}{\prob{x, y}} = 
    \E{\ic{\probc{x}{y}} + \ic{\prob{y}}}{\prob{x, y}} =
    \E{\ic{\probc{x}{y}}}{\prob{x, y}} + \chainedE{\ic{\prob{y}}}{\prob{y}},
\end{align*}
which is valid in both the discrete and continuous case (if the integrals all exist).
The question of how to transfer inequalities in the discrete case to the continuous case is dealt with in \secref{sec:continuous entropies}.

We will further require the Kullback-Leibler divergence $\kale{\cdot}{\cdot}$ and cross-entropy $\CrossEntropy{\cdot}{\cdot}$:
\begin{align*}
    \CrossEntropy{\prob{x}}{\qprob{x}} &= \chainedE{\ic{\qprob{x}}}{\prob{x}} \\
    \kale{\prob{x}}{\qprob{x}} &= \chainedE{\ic{\tfrac{\qprob{x}}{\prob{x}}}}{\prob{x}} \\
    \CrossEntropy{\probc{y}{x}}{\qprobc{y}{x}} &= \chainedE{\chainedE{\ic{\qprobc{y}{x}}}{\probc{y}{x}}}{\prob{x}} \\
    &= \chainedE{\ic{\qprobc{y}{x}}}{\prob{x, y}} \\
    \kale{\probc{y}{x}}{\qprobc{y}{x}} &=
    \chainedE{\ic{\tfrac{\qprobc{y}{x}}{\probc{y}{x}}}}{\prob{x, y}}  
\end{align*}

\subsection{Information diagrams}
\label{app:information diagrams}

Information diagrams (I-diagrams), like the one depicted in \figref{fig:mickeymouse} (or \figref{appfig:big mickey mouse} for a bigger version), visualize the relationship between information quantities: \citet{Yeung1991} shows that we can define a signed measure $\opmus$ such that these well-known quantities map to abstract sets and are consistent with set operations.
\begin{align*}
    \Entropy{A} &= \MuS{A} \\
    \Entropy{A_1, \ldots, A_n} &= \MuS{\cup_i A_i} \\
    \Hc{A_1, \ldots, A_n}{B_1, \ldots, B_n} 
    &= \MuS{\cup_i A_i - \cup_i B_i} \\
    \MIt{A_1}{\ldots}{A_n} &= \MuS{\cap_i A_i} \\
    \MIct{A_1}{\ldots}{A_n}{B_1,\dots,B_n} &= \MuS{\cap_i A_i - \cup_i B_i}
\end{align*}
Note that interaction information \citep{McGill1954} follows as canonical generalization of the mutual information to multiple variables from that work, whereas total correlation does not.

In other words, equalities can be read off directly from I-diagrams: an information quantity is the sum of its parts in the corresponding I-diagram. This is similar to Venn diagrams. The sets used in I-diagrams are just abstract symbolic objects, however.

An important distinction between I-diagrams and Venn diagrams is that while we can always read off inequalities in Venn diagrams, this is not true for I-diagrams in general because mutual information terms in more than two variables can be negative. In Venn diagrams, a set is always larger or equal any subset.

However, if we show that all information quantities are non-negative, we can read off inequalities again. 
We do this for \figref{fig:mickeymouse} at the end of \secref{sec:probabilistic model} for categorical $\Z$ and expand this to continuous $\Z$ in \secref{sec:continuous entropies}.
Thus, we can treat the Mickey Mouse I-diagram like a Venn diagram to read off equalities and inequalities. 

Nevertheless, caution is warranted sometimes.
As the signed measure can be negative, $\MuS{X \cap Y} = 0$ does \emph{not} imply $X \cap Y = \emptyset$: deducing that a mutual information term is $0$ does not imply that one can simply remove the corresponding area in the I-diagram. There could be $Z$ with $\MuS{(X \cap Y) \cap Z} < 0$, such that $\MuS{X \cap Y} = \MuS{X \cap Y \cap Z} + \MuS{X \cap Y - Z} = 0$ but $X \cap Y \neq \emptyset$.
This also means that we cannot drop the term from expressions when performing symbolic manipulations.
This is of particular importance because a mutual information of zero means two random variables are independent, which might invite one drawing them as disjoint areas.

The only time where one can safely remove an area from the diagram is for \emph{atomic} quantities, which are quantities which reference all the available random variables \citep{Yeung1991}. For example, when we only have three variables $X, Y, Z$, $\MIt{X}{Y}{Z}$ and $\MIc{X}{Y}{Z}$ are atomic quantities. We can safely remove atomic quantities from I-diagrams when they are $0$ as there are no random variables left to apply that could lead to the problem explored above.

Continuing the example, $0 = \MIt{X}{Y}{Z} = \MuS{X \cap Y \cap Z}$ would imply $X \cap Y \cap Z = \emptyset$, and we could remove it from the diagram without loss of generality. Moreover, atomic $\MIc{X}{Y}{Z} = \MuS{X \cap Y - Z} = 0$ then and could be removed from the diagram as well. 

We only use I-diagrams for the three variable case, but they supply us with tools to easily come up with equalities and inequalities for information quantities. In the general case with multiple variables, they can be difficult to draw, but for Markov chains they can be of great use.

\section{Mickey Mouse I-diagram}
\subsection{Intuition for the Mickey Mouse information quantities}
\label{app:mickey mouse intuitions}
We base the names of information quantities on existing conventions and come up with sensible extensions. For example, the name \CapturedInfo for $\capturedinfo$ was introduced by \citet{Tishby2015}. It can be seen as the intersection of $\preservedinfo$ and $\relevantinfo$ in the I-diagram, and hence we denote $\preservedinfo$ \PreservedInfo and $\relevantinfo$ \RelevantInfo, which are sensible names as we detail below.

We identify the following six atomic quantities:
\begin{description}
    \item[\LabelUncertainty{} $\labeluncertainty$] quantifies the uncertainty in our labels. If we have multiple labels for the same data sample, it will be $> 0$. It is $0$ otherwise.
    \item[\EncodingUncertainty{} $\encodinguncertainty$] quantifies the uncertainty in our latent encoding given a sample. When using a Bayesian model with random variable $\omega$ for the weights, one can further split this term into $\encodinguncertainty = \MIc{\Z}{\omega}{\X} + \Hc{\Z}{\X, \omega}$, so uncertainty stemming from weight uncertainty and independent noise \citep{Houlsby2011,Kirsch2019}.
    \item[\CapturedInfo{} $\capturedinfo$] quantifies information in the latent that is relevant for our task of predicting the labels \citep{Tishby2015}. Intuitively, we want to maximize it for good predictive performance.
    \item[\ResidualInfo{} $\residualinfo$] quantifies information for the labels that is not captured by the latent \citep{Tishby2015} but would be useful to be captured.
    \item[\RedundantInfo{} $\redundantinfo$] quantifies information in the latent that is not needed for predicting the labels\footnote{\citet{Fisher2019} uses the term ``Residual Information'' for this, which conflicts with \citet{Tishby2015}.}.
\end{description}
We also identify the following composite information quantities:
\begin{description}
    \item[\RelevantInfo{}] $\relevantinfo = \residualinfo + \capturedinfo$ quantifies the information in the data that is relevant for the labels and which our model needs to capture to be able to predict the labels.
    \item[\PreservedInfo{}] $\preservedinfo = \redundantinfo + \capturedinfo$ quantifies information from the data that is preserved in the latent.
    \item[\DecoderUncertainty{}] $\decoderuncertainty = \residualinfo + \labeluncertainty$ quantifies the uncertainty about the labels after learning about the latent $\Z$. If $\decoderuncertainty$ reaches $0$, it means that no additional information is needed to infer the correct label $\Y$ from the latent $\Z$: the optimal decoder can be a deterministic mapping. Intuitively, we want to minimize this quantity for good predictive performance.
    \item[\RDecoderUncertainty{}] $\rdecoderuncertainty = \redundantinfo + \encodinguncertainty$ quantifies the uncertainty about the latent $\Z$ given the label $\Y$. We can imagine training a new model to predict $\Z$ given $\Y$ and minimizing $\rdecoderuncertainty$ to 0 would allow for a deterministic decoder from the latent to given the label. 
    \item[\textcolor{datanoisecolor}{Nuisance}\footnotemark]\footnotetext{Not depicted in \figref{fig:mickeymouse}.} $\textcolor{datanoisecolor}{\Hc{\X}{\Y}} = \datanoise + \preservedinfo$ quantifies the information in the data that is not relevant for the task \citep{Achille2017}.
\end{description}

\subsection{Definitions \& equivalences}
\label{app:all definitions}

The following equalities can be read off from \figref{fig:mickeymouse}. For completeness and to provide a handy reference, we list them explicitly here. They can also be verified using symbolic manipulations and the properties of information quantities.

Equalities for composite quantities:
\begin{align}
\relevantinfo &= \residualinfo + \capturedinfo \\
\preservedinfo &= \redundantinfo + \capturedinfo \\
\decoderuncertainty &= \residualinfo + \labeluncertainty \\
\rdecoderuncertainty &= \redundantinfo +\encodinguncertainty \\
\textcolor{datanoisecolor}{\Hc{\X}{\Y}} &= \datanoise + \preservedinfo
\end{align}
We can combine the atomic quantities into the overall \LabelEntropy{} and \EncodingEntropy{}:
\begin{align}
    \labelentropy &= \labeluncertainty + \capturedinfo + \residualinfo \\
    \encodingentropy &= \encodinguncertainty + \capturedinfo + \redundantinfo.
\end{align}
We can express the \RelevantInfo{} $\relevantinfo$, \ResidualInfo{} $\residualinfo$, \RedundantInfo{} $\redundantinfo$ and \PreservedInfo{} $\preservedinfo$ without $\X$ on the left-hand side:
\begin{align}
    \relevantinfo &= \labelentropy - \labeluncertainty, \label{eq:equivalences} \\
    \preservedinfo &= \encodingentropy - \encodinguncertainty, \label{eq:usedinfo} \\
    \residualinfo &= \decoderuncertainty - \labeluncertainty, \label{eq:residualinfo} \\
    \redundantinfo &= \rdecoderuncertainty - \encodinguncertainty. \label{eq:redundantinfo}
\end{align}
This simplifies estimating these expressions as $\X$ is usually much higher-dimensional and irregular than the labels or latent encodings.
We also can rewrite the \CapturedInfo{} $\capturedinfo$ as:
\begin{align}
    \capturedinfo &= \labelentropy - \decoderuncertainty \\
    \capturedinfo &= \encodingentropy - \rdecoderuncertainty
\end{align}

\section{Information bottleneck \& related works}
\label{app:information bottlenecks}

\subsection{Goals \& motivation}
\label{app:optimization goals}

The \IB principle from \citet{Tishby2000} can be recast as a generalization of finding minimal sufficient statistics for the labels given the data \citep{Shamir2010,Tishby2015,Fisher2019}:
it strives for minimality and sufficiency of the latent $\Z$. Minimality is about minimizing amount of information necessary of $\X$ for the task, so minimizing the \PreservedInfo $\preservedinfo$; while sufficiency is about preserving the information to solve the task, so maximizing the \CapturedInfo $\capturedinfo$.

From \figref{fig:mickeymouse}, we can read off the definitions of \RelevantInfo{} and \PreservedInfo{}:
\begin{align}
    \relevantinfo  &= \capturedinfo + \residualinfo \label{eq:residual_vs_captured} \\
    \preservedinfo &= \capturedinfo + \redundantinfo \label{eq:redundant_vs_captured},
\end{align}
and see that maximizing the \CapturedInfo $\capturedinfo$ is equivalent to minimizing the \ResidualInfo $\residualinfo$, while minimizing the \PreservedInfo $\preservedinfo$ at the same time means minimizing the \RedundantInfo $\redundantinfo$, too, as $\relevantinfo$ is constant for the given dataset\footnote{That is, it does not depend on $\theta$.}.
Moreover, we also see that the \CapturedInfo{} $\capturedinfo$ is upper-bounded by \RelevantInfo{} $\relevantinfo$, so to capture all relevant information in our latent, we want $\relevantinfo = \capturedinfo$. 

Using the diagram, we can also see that minimizing the \ResidualInfo{} is the same as minimizing the \DecoderUncertainty{} $\decoderuncertainty$:
\begin{align*}
    \residualinfo &= \decoderuncertainty - \labeluncertainty.
\end{align*}
Ideally, we also want to minimize the \EncodingUncertainty $\encodinguncertainty$ to find the most deterministic latent encoding $\Z$. Minimizing the \EncodingUncertainty and the \RedundantInfo $\redundantinfo$ together is the same as minimizing the \RDecoderUncertainty $\rdecoderuncertainty$.

All in all, we want to minimize both the \DecoderUncertainty $\decoderuncertainty$ and the \RDecoderUncertainty $\rdecoderuncertainty$.

\subsection{\IB objectives}
\label{app:ib objectives}

\subsubsection*{``The Information Bottleneck Method'' (\IB)}

\citet{Tishby2000} introduce $MI(X; \hat{X}) - \beta MI(\hat{X};Y)$ as optimization objective for the Information Bottleneck. We can relate this to our notation by renaming $\hat{X} = \Z$, such that the objective becomes ``%
\begin{math}
    \min \preservedinfo - \beta \capturedinfo
\end{math}''.
The \IB objective minimizes the \PreservedInfo $\preservedinfo$ and trades it off with maximizing the \CapturedInfo $\capturedinfo$. \citet{Tishby2015} mention that the
\IB objective is equivalent to minimizing $\preservedinfo + \beta \residualinfo$, see our discussion above. 
\citet{Tishby2000} provide an optimal algorithm for the tabular case, when $\X$, $\Y$ and $\Z$ are all categorical. This has spawned additional research to optimize the objective for other cases and specifically for DNNs.

\subsubsection*{``Deterministic Information Bottleneck'' (\DIB{})}

\citet{Strouse2016} introduce as objective ``%
\begin{math}
    \min \encodingentropy - \beta \capturedinfo
\end{math}''.
Compared to the \IB objective, this also minimizes $\encodinguncertainty$ and encourages determinism. Vice-versa, for deterministic encoders, $\encodinguncertainty = 0$, and their objective matches the \IB objective. 
Like \citet{Tishby2000}, they provide an algorithm for the tabular case. To do so, they examine an analytical solution for their objective as it is unbounded: $\encodinguncertainty \to -\infty$ for the optimal solution. As we discuss in \secref{sec:continuous entropies}, it does not easily translate to a continuous latent representation.

\subsubsection*{``Deep Variational Information Bottleneck''}

\citet{Alemi2019} rewrite the terms in the bottleneck as maximization problem ``$\max \capturedinfo - \beta \preservedinfo$'' and swap the $\beta$ parameter. Their $\beta$ would be $\nicefrac{1}{\beta}$ in \IB above, which emphasizes that $\capturedinfo$ is important for performance and $\preservedinfo$ acts as regularizer. 

The paper derives the following variational approximation to the \IB objective, where 
$\z = f_\theta(\x, \epsilon)$ denotes a stochastic latent embedding with distribution $\pmencoder$,
$\pmdecoder$ denotes the decoder,
and $\rprob{\z}$ is some fixed prior distribution on the latent embedding:
\begin{align}
    \min \chainedE{\E{- \log \probmc{\Yhat=y}{z = f_\theta(x_n, \epsilon)} + \gamma \, \KL(p(z|x_n)||r(z))}{\epsilon \sim p(\epsilon)}}{\pxy}.
\end{align}
In principle, the distributions $\pmdecoder$ and $\pmencoder$ could be given by arbitrary parameterizations and function approximators. In practice, the implementation of DVIB presented by \citet{Alemi2019} constructs $\pmencoder$ as a multivariate Gaussian with parameterized mean and parameterized diagonal covariance using a neural network, and then uses a simple logistic regression to obtain $\pmdecoder$, while arbitrarily setting $\rprob{\z}$ to be a unit Gaussian around the origin. The requirement for $\pmencoder$ to have a closed-form Kullback-Leibler divergence limits the applicability of the \DVIB objective. 

The \DVIB objective can be written more concisely as
\begin{align*}
    \min \decoderXE + \gamma \, \kale{\pzcx}{\rprob{z}}
\end{align*}
in the notation introduced in \secref{sec:different perspective on IB and DIB}. We discuss the regularizer in more detail in \secref{app:alemia and msa}.

\subsubsection*{``Conditional Entropy Bottleneck''}

In a preprint, \citet{Fisher2019} introduce their Conditional Entropy Bottleneck as ``$\min \redundantinfo - \capturedinfo$''. We can rewrite the objective as $\redundantinfo + \residualinfo - \relevantinfo$, using equations \eqref{eq:residual_vs_captured} and \eqref{eq:redundant_vs_captured}. The last term is constant for the dataset and can thus be dropped. Likewise, the \IB objective can be rewritten as minimizing $\redundantinfo + (\beta - 1) \residualinfo$. The two match for $\beta = 2$. \citet{Fisher2019} provides experimental results that favorably compare to \citet{Alemi2019}, possibly due to additional flexibility as \citet{Fisher2019} do not constrain $\pz$ to be a unit Gaussian
and employ variational approximations for all terms. We relate \CEB to \EDM in \secref{app:EDM}.

\subsubsection*{``Conditional Entropy Bottleneck'' (2020)}

In a substantial revision of the preprint, \citet{fischer2020conditional} change their Conditional Entropy Bottleneck to include a Lagrange multiplier: ``$\min \redundantinfo - \gamma \capturedinfo$''.
Their \VCEB objective can be written more concisely as
\begin{align*}
    \min \decoderXE + \gamma (\textcolor{rdecoderuncertaintycolor}{\opH_\theta [ \Z \mathbin{\vert} \Y ]} - \textcolor{encodinguncertaintycolor}{\opH_\theta [ \Z \mathbin{\vert} \X ]}),
\end{align*}
where, without writing down the probabilistic model, we introduce variational approximations for the \RDecoderUncertainty and the \EncodingUncertainty.

They are the first to report results on CIFAR-10. It is not clear how they parameterize the model they use for CIFAR-10. They use one Gaussian per class to model $\textcolor{rdecoderuncertaintycolor}{\opH_\theta [ \Z \mathbin{\vert} \Y ]}$.

\subsubsection*{``CEB Improves Model Robustness''} 
\citet{fischer2020modelrobustness} take \CEB and switch to a deterministic model which they turn it into a stochastic encoder by adding unit Gaussian noise. They use Gaussians of fixed variance to variationally approximate $\qprobc{y}{z}$: for each class, $\qprobc{y}{z}$ is modelled as a separate Gaussian. 

They are the first to report results on ImageNet and report good rebustness against adversarial attacks without adversarial training.

\subsection{Canonical \IB{} \& \DIB{} objectives}
\label{app:different perspective on IB and DIB}
We expand the \IB and \DIB objectives into ``disjoint'' terms and drop constant ones to find a more canonical form. This leads us to focus on the optimization of the \DecoderUncertainty $\decoderuncertainty$ along with additional regularization terms. In \secref{sec:decoder uncertainty}, we discuss the properties of $\decoderuncertainty$, and in \secref{sec:regularizer estimators} we examine the regularization terms. 

\begin{exobservation}
    For \IB, we obtain 
    \begin{align}
        \argmin \preservedinfo - \beta \capturedinfo &= \argmin \decoderuncertainty + \beta' \underbrace{\redundantinfo}_{=\rdecoderuncertainty - \encodinguncertainty}, \\
    \intertext{
    and, for \DIB,
    }
        \argmin \encodingentropy - \beta \capturedinfo &= \argmin \decoderuncertainty + \beta' \rdecoderuncertainty
        = \argmin \decoderuncertainty + \beta'' \encodingentropy
    \end{align}
    with $\beta' := \frac{1}{\beta - 1} \in [0,\infty)$ and $\beta'' := \frac{1}{\beta}\in [0,1)$.
\end{exobservation}

\begin{proof}
For the steps marked with $\text{*}$, we make use of $\beta > 1$.
For \IB, we obtain 
\begin{align}
    \argmin \preservedinfo - \beta \capturedinfo
    &= \argmin \redundantinfo + (\beta - 1) \decoderuncertainty \notag
    \\
    &\overset{\text{(*)}}{=} \argmin \decoderuncertainty + \beta' \; \redundantinfo \notag
    \\
    &= \argmin \decoderuncertainty + \beta' (\rdecoderuncertainty - \encodinguncertainty), \label{eq:IBobjective} \tag{\IB}
\end{align}
and, for \DIB,
\begin{align}
    \argmin \encodingentropy - \beta \capturedinfo \notag
    & = \argmin \rdecoderuncertainty + (\beta - 1) \decoderuncertainty \notag 
    \\
    &\overset{\text{(*)}}{=} \argmin \decoderuncertainty + \beta' \rdecoderuncertainty, \label{eq:DIBobjective} \tag{\DIB}
\end{align}
with $\beta' := \frac{1}{\beta - 1} \in [0,\infty)$.
Similarly, we show for \DIB
\begin{align*}
    \argmin \encodingentropy - \beta \capturedinfo \notag
    &= \argmin \encodingentropy + \beta \decoderuncertainty \notag \\
    &\overset{\text{(*)}}{=} \argmin \decoderuncertainty + \beta'' \encodingentropy, %
\end{align*}
with $\beta'' := \frac{1}{\beta}\in [0,1)$, which is relevant in \secref{sec:regularizer estimators}.

We limit ourselves to $\beta > 1$, because, 
for $\beta < 1$, we would be maximizing the \DecoderUncertainty{}, which does not make sense: the obvious solution to this is one where $\Z$ contains no information on $\Y$, that is $\pdecoder$ is uniform.
In the case of \DIB{}, it is to map every input deterministically to a single latent;
whereas for \IB{}, we only minimize the \RedundantInfo{}, and the solution is free to contain noise.
For $\beta = 1$, we would not care about \DecoderUncertainty and only minimize \RedundantInfo and \RDecoderUncertainty, respectively, which allows for arbitrarily bad predictions.
\end{proof}

We note that we have $\beta' = \frac{\beta''}{1-\beta''}$ using the relations above.

\subsection{\IB objectives and the Entropy Distance Metric}
\label{app:EDM}

Another perspective on the \IB objectives is by expressing them using the Entropy Distance Metric. \citet[p. 140]{Mackay} introduces the entropy distance
\begin{equation}
    \edmyz = \decoderuncertainty + \rdecoderuncertainty. \label{eq:edm}
\end{equation}
as a metric when we identify random variables up to permutations of the labels for categorical variables: if the entropy distance is $0$, $\Y$ and $\Z$ are the same distribution up to a consistent permutation of the labels (independent of $\X$). If the entropy distance becomes 0, both $\decoderuncertainty = 0 = \rdecoderuncertainty$, and we can find a bijective map from $\Z$ to $\Y$.\footnote{The argument for continuous variables is the same. We need to identify distributions up to ``isentropic'' bijections.}

We can express the \RDecoderUncertainty{} $\rdecoderuncertainty$ using the \DecoderUncertainty{} $\decoderuncertainty$ and the entropies:
\begin{align*}
    \rdecoderuncertainty + \labelentropy = \decoderuncertainty + \encodingentropy, 
\end{align*}
and rewrite equation \eqref{eq:edm} as
\begin{equation*}
    \edmyz = 2 \decoderuncertainty + \encodingentropy - \labelentropy.
\end{equation*}
For optimization purposes, we can drop constant terms and rearrange:
\begin{equation*}
    \argmin \edmyz = \argmin \decoderuncertainty + \tfrac{1}{2} \encodingentropy.
\end{equation*}

\subsubsection{Rewriting \IB{} and \DIB{} using the \EDM{}}

For $\beta \ge 1$, we can rewrite equations \eqref{eq:IBobjective} and \eqref{eq:DIBobjective} as:
\begin{align}
    \argmin \edmyz + \gamma \left (\decoderuncertainty - \rdecoderuncertainty\right ) + \left ( \gamma - 1 \right ) \encodinguncertainty
\end{align}
for \IB{}, and
\begin{align}
    \argmin \edmyz + \gamma \left (\decoderuncertainty - \rdecoderuncertainty \right )
\end{align}
for \DIB{} and replace $\beta$ with $\gamma = 1 - \frac{2}{\beta} \in [-1,1]$ which allows for a linear mix between $\decoderuncertainty$ and $\rdecoderuncertainty$.

\DIB{} will encourage the model to match both distributions for $\gamma = 0$ ($\beta=2$), as we obtain a term that matches the \EDM{} from \secref{app:EDM}, and otherwise trades off \DecoderUncertainty{} and \RDecoderUncertainty{}.
\IB{} behaves similarly but tends to maximize \EncodingUncertainty{} as $\gamma - 1 \in [-2,0]$. 
\citet{Fisher2019} argues for picking this configuration similar to the arguments in \secref{app:optimization goals}.
\DIB{} will force both distributions to become exactly the same, which would turn the decoder into a permutation matrix for categorical variables.

\section{\DecoderUncertainty \texorpdfstring{$\decoderuncertainty$}{}}
\subsection{Cross-entropy loss}
\label{sec:xel}

The cross-entropy loss features prominently in \secref{sec:decoder uncertainty}.
We can derive the usual cross-entropy loss for our model by minimizing the Kullback-Leibler divergence between the empirical sample distribution $\pxy$ and the parameterized distribution $\pmgenerative$. For discriminative models, we are only interested in $\pmpredict$, and can simply set $\probm{x}=\px$:
\begin{align*}
    & \argmin_{\theta} \kale{\pxy}{\pmgenerativece}
    \\ & \quad = \argmin_{\theta} {\kale{\pycx}{\pmpredictce}}
        + \underbrace{\kale{\px}{\probm{x}}}_{=0} 
    \\ & \quad
    = \argmin_{\theta} {\CrossEntropy{\pycx}{\pmpredictce}}
    - \underbrace{\Hc{\Y}{\X} }_{\text{const.}}
    \\ & \quad
    = \argmin_{\theta} \CEpmpredict.
\end{align*}
In \secref{sec:decoder uncertainty}, we introduce the shorthand $\predictXE$ for $\CEpmpredict$ and refer to it as \PredictionCE.

\subsection{Upper bounds \& training error minimization}
\label{app:minimizing the training error}

To motivate that $\decoderuncertainty$ (or $\decoderXE$) can be used as main loss term, we show that it can bound the (training) error probability since \textit{accuracy} is often the true objective when machine learning models are deployed on real-world problems\footnote{As we only take into account the empirical distribution $\pxy$ available for training, the following derivation refers only to the empirical risk, and not to the expected risk of the estimator $\Yhat$.}.
\begin{exobservation}
    The \DecoderCE provides an upper bound on the \DecoderUncertainty:
    \begin{align*}
        \decoderuncertainty \le \decoderuncertainty + \kale{\pdecoder}{\pmdecoder} = \decoderXE,  
    \end{align*}
    and further bounds the training error:
    \begin{align*}
        \prob{\text{``$\Yhat$ is wrong''}} & \le 1 - e^{-\decoderXE}
        = 1 - e^{-\left (\decoderuncertainty + \kale{\pdecoder}{\pmdecoder} \right )}.
    \end{align*}
    Likewise, for the \PredictionCE $\predictXE$ and the \LabelUncertainty $\labeluncertainty$.
\end{exobservation}
\begin{proof}
    The upper bounds for \DecoderUncertainty $\decoderuncertainty$ and \LabelUncertainty $\labeluncertainty$ follow from the non-negativity of the Kullback-Leibler divergence, for example:
    \begin{align*}
        0 \le \kale{\pdecoder}{\pmdecoder} 
        &= \decoderXE - \decoderuncertainty, \\
        0 \le \kale{\pycx}{\pmpredict}
        &= \predictXE - \labeluncertainty.
    \end{align*}
    The derivation for the training error probability is as follows:
    \begin{align*}
        \prob{\text{``$\Yhat$ is correct''}} 
        &= \chainedE{\probc{\text{``$\Yhat$ is correct''}}{\x,\y}}{\pxy}
        = \chainedE{\chainedE{\pmdecoderce}{\pmencoder}}{\pxy} 
        \\ & = \chainedE{\pmdecoderce}{\pyz}.
    \end{align*}
    We can then apply Jensen's inequality using convex $\ic{x} = -\ln x$:
    \begin{align*}
        & \ic{\chainedE{\pmdecoderce}{\pyz}} \le \chainedE{\ic{\pmdecoderce}}{\pyz} \\
        & \Leftrightarrow \prob{\text{``$\Yhat$ is correct''}} \ge
        e^{-\CEpmdecoder} \\
        & \Leftrightarrow \prob{\text{``$\Yhat$ is wrong''}} \le 1 - e^{-\decoderXE}.
    \end{align*}
    For small $\decoderXE$, we note that one can use the approximation $e^x \approx 1 +x$ to obtain:
    \begin{align}
        \prob{\text{``$\Yhat$ is wrong''}} \lessapprox \decoderXE.
    \end{align}
    Finally, we split the \DecoderCE into the \DecoderUncertainty and a Kullback-Leibler divergence:
    \begin{equation*}
        \decoderXE = \decoderuncertainty + \kale{\pdecoder}{\pmdecoderce}.
    \end{equation*}
    If we upper-bound $\kale{\pdecoder}{\pmdecoderce}$, minimizing the \DecoderUncertainty $\decoderuncertainty$ becomes a sensible minimization objective as it reduces the probability of misclassification.

    We can similarly show that the training error is bounded by the \PredictionCE $\predictXE$.
\end{proof}
In the next section, we examine categorical $\Z$ for which optimal decoders can be constructed and $\kale{\pdecoder}{\pmdecoderce}$ becomes zero. %

\section{Categorical \texorpdfstring{$\Z$}{Z}}
\label{app:optimal decoders}

For categorical $\Z$, $\pdecoder$ can be computed exactly for a given encoder $\pmencoder$ by using the empirical data distribution, which, in turn, allows us to compute $\decoderuncertainty$\footnote{$\pdecoder$ depends on $\theta$ through $\pmencoder$: $\pdecoder = \frac{\sum_x \pxy \pmencoder}{\sum_x \px \pmencoder}$.}. This is similar to computing a confusion matrix between $\Y$ and $\Z$ but using information content instead of probabilities.

Moreover, if we set $\pmdecoder := \pdecoderhat$ to have an optimal decoder, we obtain equality in equation \eqref{eq:decoder bound with KALE}, and obtain
\begin{math}
    \predictXE \le \decoderXE = \decoderuncertainty.
\end{math}
If the encoder were also deterministic, we would obtain
\begin{math}
    \predictXE = \decoderXE = \decoderuncertainty.
\end{math}
We can minimize $\decoderuncertainty$ directly using gradient descent. $\ddtheta \decoderuncertainty$ only depends on $\pdecoder$ and $\ddtheta \pmencoder$:
\begin{equation*}
    \ddtheta \decoderuncertainty =
    \E{\ddtheta \left [ \ln \pmencoder \right ] \chainedE{\ic{\pdecoder}}{\probhc{\y}{\x}} }{\prob{\x,\z}}.
\end{equation*}

\begin{proof}
    \begin{align*}
        \ddtheta \decoderuncertainty &= \ddtheta \chainedE{\ic{\pdecoder}}{\pyz}
        = \ddtheta \chainedE{\ic{\pdecoder}}{\pxyz}
        = \chainedE{\ddtheta \chainedE{\ic{\pdecoder}}{\pmencoder}}{\pxy} \\
        & = \chainedE{\chainedE{\ddtheta \left [ \ic{\pdecoder} \right ]}{\pxy}}{\pmencoder} +
        \ic{\pdecoder} \ddtheta \left [ \ln \pmencoder \right ]
        \\
        & =
        \chainedE{\ddtheta \left [ \ic{\pdecoder} \right ]}{\pxyz} + 
        \ic{\pdecoder} \ddtheta \left [ \ln \pmencoder \right ].
    \end{align*}
    And now we show that $\chainedE{\ddtheta \left [ \ic{\pdecoder} \right ]}{\pxyz} = 0$:
    \begin{align*}
        \chainedE{\ddtheta \left [ \ic{\pdecoder} \right ]}{\pxyz} 
        &= \chainedE{\ddtheta \left [ \ic{\pdecoder} \right ]}{\pyz} = \chainedE{\frac{-1}{\pdecoder} \ddtheta \pdecoder}{\pyz}
        \\
        &= - \int \frac{\pyz}{\pdecoder} \ddtheta \pdecoder \, dy \, dz 
        = - \int \prob{\z} \int \ddtheta \pdecoder \, dy \, dz
        \\
        &= - \int \prob{\z} \ddtheta \Big[ \underbrace{\int \pdecoder \, dy}_{=1} \Big] \, dz = 0.
    \end{align*}
    Splitting the expectation and reordering of $\chainedE{\ic{\pdecoder}  \ddtheta \left [ \ln \pmencoder \right ]}{\pxyz}$, we obtain the result.
\end{proof}

The same holds for \RDecoderUncertainty{} $\rdecoderuncertainty$ and for the other quantities as can be verified easily.

If we minimize $\decoderuncertainty$ directly, we can compute $\pdecoder$ after every training epoch and fix $\pmdecoder := \pdecoderhat$ to create the discriminative model $\pmpredict$. This is a different perspective on the self-consistent equations from \citet{Tishby2000,Gondek}.

\subsection{Empirical evaluation of \texorpdfstring{$\kale{\pdecoder}{\pmdecoderce}$}{the difference Decoder Cross-Entropy and Uncertainty} during training}
\label{sec:exp for decoderuncertainty}

\begin{figure*}[t]
    \centering
    \includegraphics[width=\linewidth, clip, trim=0 0 0 0]{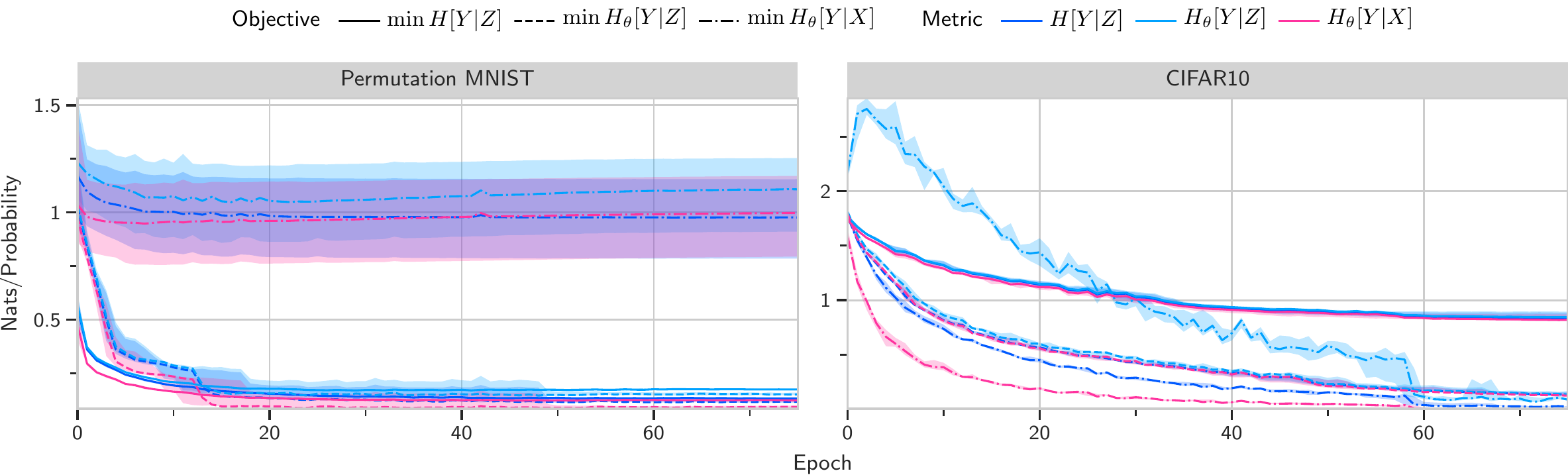}
    \caption{
        \emph{\DecoderUncertainty, \DecoderCE and \PredictionCE for Permutation-MNIST and CIFAR-10 with a categorical $\Z$.}  $C=100$ categories are used for $\Z$. We optimize with different minimization objectives in turn and plot the metrics. $\kale{\pdecoder}{\pmdecoderce}$ is small when training with $\decoderXE$ or $\decoderuncertainty$. When training with $\predictXE$ on CIFAR-10, $\kale{\pdecoder}{\pmdecoderce}$ remains quite large. We run 8 trials each and plot the median with confidence bounds (25\% and  75\% quartiles). See \secref{sec:exp for decoderuncertainty} for more details.
    }
    \label{fig:categorical_bounds}
\end{figure*}
We examine the size of the gap between \DecoderUncertainty and \DecoderCE and the training behavior of the two cross-entropies with \emph{categorical} latent $\Z$ on Permutation MNIST and CIFAR-10.
For Permutation MNIST \citep{Goodfellow}, we use the common fully-connected ReLU $784-1024-1024-C$ encoder architecture, with $C=100$ categories for $\Z$. For CIFAR-10 \citep{Krizhevsky2009}, we use a standard ResNet18 model with $C$ many output classes as encoder \citep{He2016}. See \secref{app:experiment_details} for more details about the hyperparameters.
Even though a $C \times 10$ matrix and a SoftMax would suffice to describe the decoder matrix $\pmdecoder$\footnote{For categorical $\Z$, $\pmdecoder$ is a stochastic matrix which sums to 1 along the $\Yhat$ dimension.}, we have found that over-parameterization using a separate DNN benefits optimization a lot. Thus, to parameterize the decoder matrix, we use fully-connected ReLUs $C-1024-1024-10$ with a final SoftMax layer.
We compute it once per batch during training and back-propagate into it.

\Figref{fig:categorical_bounds} shows the three metrics as we train with each of them in turn. Our results do not achieve SOTA accuracy on the test set---we impose a harder optimization problem as $\Z$ is categorical, and we are essentially solving a hard-clustering problem first and then map these clusters to $\Yhat$. Results are provided for the training set in order to compare with the optimal decoder.

As predicted, the \DecoderCE upper-bounds both the \DecoderUncertainty $\decoderuncertainty$ and the \PredictionCE in all cases. Likewise, the gap between $\decoderXE$ and $\decoderuncertainty$ is tiny when we minimize $\decoderXE$. On the other hand, minimizing \PredictionCE can lead to large gaps between $\decoderXE$ and $\decoderuncertainty$, as can be seen for CIFAR-10.

Very interestingly, on MNIST \DecoderCE provides a better training objective whereas on CIFAR-10 \PredictionCE trains lower. \DecoderUncertainty does not train very well on CIFAR-10, and \PredictionCE does not train well on Permutation MNIST at all. We suspect DNN architectures in the literature have evolved to train well with cross-entropies, but we are surprised by the heterogeneity of the results for the two datasets and models.

\section{Surrogates for regularization terms}
\subsection{Differential entropies}
\label{sec:math for differential entropies}

\begin{exobservation}
    After adding \ZeroEntropyNoise, the inequality $\redundantinfo \le \rdecoderuncertainty \le \encodingentropy$ also holds in the continuous case, and we can minimize $\redundantinfo$ in the \IB objective by minimizing $\rdecoderuncertainty$ or $\encodingentropy$, similarly to the \DIB objective. 
    We present a formal proof in \secref{sec:math for differential entropies}.
\end{exobservation}

\begin{theorem}
\label{thm:noise helps}
    For random variables $A$, $B$, we have
    \begin{equation*}
        \Entropy{A+B} \ge \Entropy{B}.
    \end{equation*}
\end{theorem}
\begin{proof}
    See \citet[section 2.2]{Bercher}.
\end{proof}

\begin{proposition}
    Let $Y$, $Z$ and $X$ be random variables satisfying the independence property $Z\perp Y | X$, and $F$ a possibly stochastic function such that $Z = F(X) + \eps$, with independent noise $\eps$ satisfying $\eps \perp F(X), \eps \perp Y$ and $H(\eps) = 0$. Then the following holds whenever $\capturedinfo$ is well-defined.
    \begin{equation*}
        \redundantinfo \le \rdecoderuncertainty \le \encodingentropy.
    \end{equation*}
\end{proposition}
\begin{proof}
    First, we note that $\encodinguncertainty = \Hc{F(X) + \eps}{X} \geq \Hc{\eps}{X} = \Entropy{\eps}$ with \cref{thm:noise helps}, as $\eps$ is independent of $X$, and thus $\encodinguncertainty \ge 0$. We have $\encodinguncertainty = \Hc{Z}{X, Y}$ by the conditional independence assumption, and by the non-negativity of mutual information, $\capturedinfo \ge 0$.
    Then:
    \begin{gather*}
        \redundantinfo + \underbrace{\encodinguncertainty}_{\ge 0} = \rdecoderuncertainty \\
        \rdecoderuncertainty + \underbrace{\capturedinfo}_{\ge 0} = \encodingentropy
    \end{gather*}
\end{proof}

The probabilistic model from \secref{sec:probabilistic model} fulfills the conditions exactly, and the two statements motivate our proposition. 

It is important to note that while \ZeroEntropyNoise is necessary for preserving inequalities like $\redundantinfo \le \rdecoderuncertainty \le \encodingentropy$ in the continuous case, any Gaussian noise will suffice for optimization purposes:
we optimize via pushing down an upper bound, and constant offsets will not affect this.

Thus, if we had $\Entropy{\eps} \ne 0$, even though $\redundantinfo + \encodinguncertainty \not\le \rdecoderuncertainty$, we could instead use
\begin{align*}
    \redundantinfo + \encodinguncertainty - \Entropy{\eps}\le \rdecoderuncertainty - \Entropy{\eps}
\end{align*}
as upper bound to minimize.
The gradients remain the same.

This also points to the nature of differential entropies as lacking a proper point of origin by themselves. We choose one by fixing $\Entropy{\eps}$. Just like other literature usually only considers mutual information as meaningful, we consider $\encodinguncertainty - \Entropy{\eps}$ as more meaningful than $\encodinguncertainty$. However, we can side-step this discussion conveniently by picking a canonical noise as point of origin in the form of \ZeroEntropyNoise $\Entropy{\eps}=0$.

\subsection{Upper bounds}
\label{sec:math for upperbound}
We derive this result as follows:
\begin{align*}
    \rdecoderuncertainty & = \chainedE{{\Hc{\Z}{\y}}}{\py} \\
    & \le \chainedE{ \tfrac{1}{2} \ln \det (2 \pi e \; {\cCov{\Z}{\y})} }{\py} \label{eq:var_ub}\\
    & \le \chainedE{ \sum_i \tfrac{1}{2} \ln (2 \pi e \; {\cVar{\Z_i}{\y})}}{\py}\\
    & \approx \chainedE{ \sum_i \tfrac{1}{2} \ln (2 \pi e \; {\hcVar{\Z_i}{\y})}}{\py},
\end{align*}

\begin{theorem}
    Given a $k$-dimensional random variable $X = (X_i)_{i=1}^k$ with $\Var{X_i} > 0$ for all $i$, 
    \begin{align*}
        \Entropy{X} & \le \tfrac{1}{2} \ln \det (2 \pi e \; {\Cov{X}} ) \\
        & \le \sum_i \tfrac{1}{2} \ln (2 \pi e \; {\Var{X_i}}).
    \end{align*}
\end{theorem}

\begin{proof}
First, the multivariate normal distribution with same covariance is the maximum entropy distribution for that covariance, and thus $\Entropy{X} \le \ln \det (2 \pi e \, \Cov{X} )$, when we substitute the differential entropy for a multivariate normal distribution with covariance $\Cov{X}$.
Let $\Sigma_0 := \Cov{X}$ be the covariance matrix and $\Sigma_1 := \opDiag(\Var{X_i})_i$ the matrix that only contains the diagonal.
Because we add independent noise, $\Var{X_i} > 0$ and thus $\Sigma_1^{-1}$ exists. It is clear that $\tr (\Sigma_1^{-1} \Sigma_0) = k$.
Then, we can use the KL-Divergence between two multivariate normal distributions $\N_0$, $\N_1$ with same mean $0$ and covariances $\Sigma_0$ and $\Sigma_1$ to show that $\ln \det \Sigma_0 \le \ln \det \Sigma_1$:
\begin{align*}
    & \quad 0 \le \kale{\N_0}{\N_1} = \tfrac{1}{2} \left( \tr (\Sigma_1^{-1} \Sigma_0) - k + \ln \left ( \frac{\det \Sigma_1}{\det \Sigma_0} \right ) \right) \\
    & \Leftrightarrow 0 \le \tfrac{1}{2} \ln \left ( \frac{\det \Sigma_1}{\det \Sigma_0} \right ) 
    \Leftrightarrow \tfrac{1}{2} \ln \det \Sigma_0 \le \tfrac{1}{2} \ln \det \Sigma_1. 
\end{align*}
We substitute the definitions of $\Sigma_0$ and $\Sigma_1$, and obtain the second inequality after adding $k \ln (2 \pi e)$ on both sides.
\end{proof}

\begin{theorem}
    Given a $k$-dimensional real-valued random variable $X = (X_i)_{i=1}^k \in \reals^k$, we can bound the entropy by the mean squared norm of the latent:
    \begin{align}
        \chainedE{\left\Vert X \right\Vert^2}{} \le C' \Rightarrow \Entropy{X} \le C,
    \end{align}
    with $C':=\frac{k e^{2C/k}}{2 \pi e}$.
\end{theorem}
\begin{proof}
    We begin with the previous bound:
    \begin{align*}
        \Entropy{X}
        & \le \sum_i \tfrac{1}{2} \ln (2 \pi e \; {\Var{X_i}}) 
        = \tfrac{k}{2} \ln 2 \pi e + \tfrac{1}{2} \ln \prod_i {\Var{X_i}} 
        \\
        & \le \tfrac{k}{2} \ln 2 \pi e + \tfrac{1}{2} \ln \left ( \tfrac{1}{k} \sum_i \Var{X_i} \right )^k
        = \tfrac{k}{2} \ln \tfrac{2 \pi e}{k} \sum_i \Var{X_i} 
        \\
        & \le \tfrac{k}{2} \ln \tfrac{2 \pi e}{k} \chainedE{\left\Vert X \right\Vert^2}{},
    \end{align*}
    where we use the AM-GM inequality:
    \begin{align*}
        \left ( \prod_i \Var{X_i} \right )^{\tfrac{1}{k}} \le \tfrac{1}{k} \sum_i \Var{X_i}
    \end{align*}
    and the monotony of the logarithm with:
    \begin{align*}
        \sum_i \Var{X_i} = \sum_i \E{X_i^2}{} - \E{X_i}{}^2 \le \sum_i \E{X_i^2}{} = \chainedE{\left\Vert X \right\Vert^2}{}
    \end{align*}
    Bounding using $\chainedE{\left\Vert X \right\Vert^2}{} \le C'$, we obtain
    \begin{align*}
        \Entropy{X} \le \tfrac{k}{2} \ln \tfrac{2 \pi e}{k} C' = C,
    \end{align*}
    and solving for $C'$ yields the statement. 
\end{proof}

This theorem provides justification for the use of $\ln \MSA$ as a regularizer, but does not justify the use of $\MSA$ directly. Here, we give two motivations.
We first observe that $\ln x \le x - 1$ due to $\ln$'s strict convexity and $\ln 1 = 0$, and thus:
\begin{align*}
    \Entropy{X} \le \tfrac{k}{2} \ln \tfrac{2 \pi e}{k} \chainedE{\left\Vert X \right\Vert^2}{} = \tfrac{k}{2}\left( \ln \tfrac{2 \pi}{k} \chainedE{\left\Vert X \right\Vert^2}{} - 1 \right ) \le \pi \chainedE{\left\Vert X \right\Vert^2}{}.
\end{align*}

We can also take a step back and remind ourselves that \IB objectives are actually Lagrangians, and $\beta$ in $\min \preservedinfo - \beta \capturedinfo$ is introduced as Lagrangian multiplier for the constrained objective:
\begin{align*}
    \min \preservedinfo \text{ s.t. } \capturedinfo \ge C.
\end{align*}
We can similarly write our canonical \DIB objective $\decoderuncertainty + \beta'' \encodingentropy$ as constrained objective
\begin{align*}
    \min \decoderuncertainty \text{ s.t. } \encodingentropy \le C,
\end{align*}
and use above statement to find the approximate form
\begin{align*}
    \min \decoderuncertainty \text{ s.t. } \MSA \le C'.
\end{align*}
Reintroducing a Lagrangian multiplier recovers our reguralized $\MSA$ objective:
\begin{align*}
    \min \decoderuncertainty + \gamma \MSA.
\end{align*}

\subsection{\texorpdfstring{``Deep Variational Information Bottleneck'' and $\MSA$}{``Deep Variational Information Bottleneck'' and Mean-Squared Activations}}
\label{app:alemia and msa}
\citet{Alemi2019} model $\pmencoder$ explicitly as multivariate Gaussian with parameterized mean and parameterized diagonal covariance in their encoder and regularize it to become close to $\normaldist{0}{I_k}$ by minimizing the Kullback-Leibler divergence $\kale{\pmencoder}{\normaldist{0}{I_k}}$ alongside the cross-entropy:
\begin{align*}
    \min \decoderXE + \gamma \, \kale{\pzcx}{\rprob{z}},
\end{align*}
as detailed in \secref{app:ib objectives}.

We can expand the regularization term to
\begin{align*}
    & \kale{\pzcx}{\normaldist{0}{I_k}} \\
    & \qquad = \chainedE{\chainedE{\ic{(2\pi)^{-\frac{k}{2}} \, e^{-\frac{1}{2}\|\Z\|^2|}}}{\pzcx}}{\px} - \encodinguncertainty \\
    & \qquad = \E{\frac{k}{2}\ln(2\pi) + \frac{1}{2}\|\Z\|^2}{\pz} - \encodinguncertainty.
\intertext{
After dropping constant terms (as they don't matter for optimization purposes), we obtain
}
    & \qquad = \frac{1}{2} \MSA - \encodinguncertainty.
\end{align*}
When we inject \ZeroEntropyNoise into the latent $\Z$, we have $\encodinguncertainty \ge 0$ and thus $\MSA - \encodinguncertainty \le \MSA$.
Thus, the $\MSA$ regularizer also upper-bounds \DVIB's regularizer in this case.

In particular, we have equality when we use a deterministic encoder. 
When we inject \ZeroEntropyNoise and use a deterministic encoder, we are optimizing the \DVIB objective function when we use the $\MSA$ regularizer. In other words, in this particular case, we could reinterpret ``$\min \decoderXE + \gamma \, \MSA$'' as optimizing the \DVIB objective from \citet{Alemi2019} if they were using a constant covariance instead of parameterizing it in their encoder. This does not hold for stochastic encoders.

We empirically compare \DVIB and the surrogate objectives from \secref{sec:summary surrogate objectives} in \secref{app:dvib_vs_uib}. In the corresponding plot in \figref{appfig:dvib_vs_uib}, we can indeed note that $\MSA$ and \DVIB are separated by a factor of 2 in the Lagrange multiplier.

\subsection{Detailed Comparison to \CEB, \VCEB \& \DVIB}
\label{app:detailed_prior_art_comparison}

In \citet{Fisher2019}, the introduced \CEB objective "$\min \redundantinfo - \gamma \capturedinfo$" is rewritten to "$\min \gamma \decoderuncertainty + \rdecoderuncertainty - \encodinguncertainty$" similar to the \IB objective in \cref{obs:rewritten_objectives} in \secref{sec:ib background}. However, these atomic quantities are not separately examined in detail.

Both \citet{Alemi2019} and \citet{fischer2020conditional} focus on the application of variational approximations to these quantities. Using a slight abuse of notation to denote all variational approximations, we can write the \VCEB objective\footnote{We will not examine the original objective without Lagrange multipliers from \citet{Fisher2019} here.} \citep{fischer2020conditional} and the \DVIB objective \citep{Alemi2019} more concisely as
\begin{align*}
    \text{\VCEB} & \equiv \min_\theta \decoderXE + \beta' (\textcolor{rdecoderuncertaintycolor}{\opH_\theta [ \Z \mathbin{\vert} \Y ]} - \textcolor{encodinguncertaintycolor}{\opH_\theta [ \Z \mathbin{\vert} \X ]}), \\
    \text{\DVIB} & \equiv \min_\theta \decoderXE + \beta'' (\textcolor{encodinguncertaintycolor}{\opH_\theta [ \Z ]} - \textcolor{encodinguncertaintycolor}{\opH_\theta [ \Z \mathbin{\vert} \X ]}).
\end{align*}
\DVIB does not specify how to choose stochastic encoders and picks the variational marginal $\qprob{z}$ to be a unit Gaussian. We relate how this choice of marginal relates to the $\MSA$ surrogate objective in \secref{app:alemia and msa}. 
\citet{Alemi2019} use VAE-like encoders that output mean and standard deviation for latents that are then sampled from a multivariate Gaussian distribution with diagonal covariance in their experiments. They run experiments on MNIST and on features extracted from the penultimate layer of pretrained models on ImageNet.

While \VCEB as introduced in \citet{Fisher2019} is agnostic to the choice of stochastic encoder, \citet{fischer2020conditional} mention that stochastic encoders can be similar to encoders and decoders in VAEs \citep{kingma2013auto} or like in \DVIB mentioned above. Both VAEs and \DVIB explicitly parameterize the distribution of the latent to sample from it before passing samples to the decoder. 

\citet{fischer2020modelrobustness} use an existing classifier architecture to output means for a Gaussian distribution with unit diagonal covariance. They further parameterize the variational approximation for the \RDecoderUncertainty $\qprobc{y}{z}$ with one Gaussian of fixed variance per class and learn this reverse decoder during training as well.
\citet{fischer2020modelrobustness} report results on CIFAR-10 and ImageNet that show good robustness against adversarial attacks without adversarial training, similar to the results in this paper.

This specific (and not motivated) instantiation of the \VCEB objective in \citet{fischer2020modelrobustness} is similar to the $\logVarZY$ surrogate objective introduced in \secref{sec:summary surrogate objectives} with a deterministic encoder and \ZeroEntropyNoise injection. However, the latter uses minibatch statistics instead of learning a reverse decoder, trading variational tightness for ease of computation and optimization. 

Compared to this prior literature, this paper examines the usage of implicit stochastic encoders (for example when using dropout) and presents three different simple surrogate objectives together with a principled motivation for \ZeroEntropyNoise injection, which has a dual use in enforcing meaningful compression and in simplifying the estimation of information quantities. Moreover, multi-sample approaches are examined to differentiate between \DecoderCE and \PredictionCE. %
In particular, implicit stochastic encoders together with \ZeroEntropyNoise and simple surrogates make it easier to use \IB objectives in practice compared to using explicitly parameterized stochastic encoders and variational approaches.

\subsection{An information-theoretic approach to VAEs}

While \citet{Alemi2019} draw a general connection to $\beta$-VAEs \citep{higgins2016beta}, we can use the insights from this paper to derive a simple VAE objective. Taking the view that VAEs learn latent representations that compress input samples, we can approach them as entropy estimators. Using $\Entropy{\X} + \Hc{\Z}{\X} = \Hc{\X}{\Z} + \Entropy{\Z}$, we obtain the ELBO
\begin{align}
    \Entropy{\X} = \Hc{\X}{\Z} + \Entropy{\Z} - \Hc{\Z}{\X} \overset{\text{(1)}}{\le} \thetaEntropy{\X}{\Z} + \Entropy{\Z} - \Hc{\Z}{\X} \overset{\text{(2)}}{\le} \thetaEntropy{\X}{\Z} + \Entropy{\Z}. \label{eq:VAE}
\end{align}

We can also put \cref{eq:VAE} into words: we want to find latent representations such that the reconstruction cross-entropy $\Hc{\X}{\Z}$ and the latent entropy $\Entropy{\Z}$, which tell us about the length of encoding an input sample, become minimal and approach the true entropy as average optimal encoding length of the dataset distribution.

The first inequality $\text{(1)}$ stems from introducing a cross-entropy approximation $\thetaEntropy{\X}{\Z}$ for the conditional entropy $\Hc{\X}{\Z}$. The second inequality (2) stems from the injection of zero-entropy noise with a stochastic encoder. For a deterministic encoder, we would have equality. We also note that $\text{(1)}$ is the \DVIB objective for a VAE with $\beta = 1$, and $\text{(2)}$ is the DIB objective for a VAE.

Finally, we can use one of the surrogates introduced in \secref{sec:summary surrogate objectives} to upper bound $\Entropy{\Z}$. For optimization purposes, we can substitute the $L_2$ activation regularizer $\MSA$ from \cref{obs:simple_approximators} and obtain as objective
\begin{align*}
    \min_\theta \thetaEntropy{\X}{\Z} + \MSA.
\end{align*}

It turns out that this objective is examined amongst others in the recently published \citet{ghosh2019variational} as a \emph{CV-VAE}, which uses a deterministic encoder and noise injection with constant variance. The paper derives this objective by noticing that the explicit parameterizations that are commonly used for VAEs are cumbersome, and the actual latent distribution does often not necessarily match the induced distribution (commonly a unit Gaussian) which causes sampling to generate out-of-distribution data. It fits a separate density estimator on $\prob{\z}$ after training for sampling. The paper goes on to then examine other methods of regularization, but also provides experimental results on CV-VAE, which are in line with VAEs and WAEs. The derivation and motivation in the paper are different and make no use of information-theoretic principles. Our short derivation above shows the power of using the insights from \cref{sec:decoder uncertainty} and \ref{sec:summary surrogate objectives} for applications outside of supervised learning. 

\subsection{Soft clustering by entropy Minimization with Gaussian noise}
\label{sec:entropy minimization with gaussian noise}
Consider the problem of minimizing $\rdecoderuncertainty$ and $\decoderuncertainty$, in the setting where $Z = f_\theta(X) + \epsilon \sim \mathcal{N}(0, \sigma^2)$---i.e. the embedding $Z$ is obtained by adding Gaussian noise to a deterministic function of the input.  Let the training set be enumerated $x_1, \dots, x_n$, with $\mu_i = f_\theta(x_i)$. Then the distribution of $Z$ is given by a mixture of Gaussians with the following density, where $d(x, \mu_i) := \|x - \mu_i\|/\sigma^2$.
\begin{equation*}
   \pz \propto \frac{1}{n} \sum_{i=1}^n \exp(-d(z, \mu_i))
\end{equation*}
Assuming that each $x_i$ has a deterministic label $y_i$, we then find that the conditional distributions $\pdecoder$ and $\probc{\z}{\y}$ are given as follows:
\begin{align*}
    \probc{\z}{\y} &\propto \frac{1}{n_y} \sum_{i: y_i=y} \exp (-d(z, \mu_{i})) \\
    \probc{y}{z} & = \sum_{i:y_i=y} \probc{\mu_i}{z} = \sum_{i:y_i=y} \frac{\probc{z}{\mu_i} \prob{\mu_i}}{\pz}\\
    &= \frac{\sum_{i: y_i=y} \probc{\z}{\mu_i}}{\sum_{k=1}^n \probc{z}{\mu_k}} = \frac{\sum_{i: y_i=y} \exp (-d(z, \mu_i))}{\sum_{k=1}^n \exp(-d(z, \mu_k) )},
\end{align*}
where $n_y$ is the number of $x_i$ with class $y_i=y$.
Thus, the conditional $Z|Y$ can be interpreted as a mixture of Gaussians and $Y|Z$ as a Softmax marginal with respect to the distances between $Z$ and the mean embeddings. We observe that $\rdecoderuncertainty$ is lower-bounded by the entropy of the random noise added to the embeddings:
\begin{equation*}
    \rdecoderuncertainty \geq \Hc{f_\theta(X) + \epsilon}{\Y} \geq \Entropy{\epsilon}
\end{equation*}
with equality when the distribution of $f_\theta(X)|Y$ is deterministic -- that is $f_\theta$ is constant for each equivalence class. 

Further, the entropy $\decoderuncertainty$ is minimized when $\encodingentropy$ is large compared to $\rdecoderuncertainty$ as we have the decomposition
\begin{equation*}
    \decoderuncertainty = \rdecoderuncertainty - \encodingentropy + \labelentropy.
\end{equation*}
In particular, when $f_\theta$ is constant over equivalence classes of the input, then $\decoderuncertainty$ is minimized when the entropy $H[f_\theta(X) + \epsilon]$ is large -- i.e. the values of $f_\theta(x_i)$ for each equivalence class are distant from each other and there is minimal overlap between the clusters. Therefore, the optima of the information bottleneck objective under Gaussian noise share similar properties to the optima of geometric clustering of the inputs according to their output class.

{
To gain a better understanding of local optimization behavior, we decompose the objective terms as follows:
\newcommand{\ddmuk}{\frac{d}{d\mu_k}}
\begin{align*}
    \rdecoderuncertainty &= \chainedE{\CrossEntropy{\probc{z}{y}}{\probc{z}{y}}}{\py} \\
    & = \chainedE{\CrossEntropy{\probc{z}{x}}{\probc{z}{y}}}{\pxy} \\
    & = \chainedE{\kale{\probc{z}{x}}{\probc{z}{y}} + \Hc{Z}{x}}{\pxy} \\
    & = \chainedE{\kale{\probc{z}{x}}{\probc{z}{y}}}{\pxy} \\
    & \qquad + \underbrace{\encodinguncertainty}_{=const}. \\
\intertext{
    To examine how the mean embedding $\mu_k$ of a single datapoint $x_k$ affects this entropy term, we look at the derivative of this expression with respect to $\mu_k = f_\theta(x_k)$. We obtain: 
}
    \ddmuk \rdecoderuncertainty &= \ddmuk \Hc{Z}{y_k} \\ 
    &= \ddmuk \chainedE{\kale{\probc{z}{x}}{\probc{z}{y}}}{\probc{x}{y_k}} \\
    &= \sum_{i \neq i: y_i=y_k} \frac{1}{n_{y_k}} \ddmuk {\kale{\probc{z}{x_i}}{\probc{z}{y_k}}} \\
    & \qquad + \frac{1}{n_{y_k}} \ddmuk \kale{\probc{z}{x_k}}{\probc{z}{y_k}}.
\end{align*}
While these derivatives do not have a simple analytic form, we can use known properties of the KL divergence to develop an intuition on how the gradient will behave.
We observe that in the left-hand sum $\mu_k$ only affects the distribution of $Z|Y$ (that is we are differentiating a sum of terms that look like a reverse KL), whereas it has greater influence on $\probc{z}{x_k}$ in the right-hand term, and so its gradient will more closely resemble that of the forward KL. The left-hand-side term will therefore push $\mu_k$ towards the centroid of the means of inputs mapping to $y$, whereas the right-hand side term is mode-seeking. 
}

\subsection{A note on differential and discrete entropies}
\label{sec:pathologies}

The mutual information between two random variables can be defined in terms of the KL divergence between the product of their marginals and their joint distribution. However, the KL divergence is only well-defined when the Radon-Nikodym derivative of the density of the joint with respect to the product exists. Mixing continuous and discrete distributions---and thus differential and continuous entropies---can violate this requirement, and so lead to negative values of the ``mutual information''. This is particularly worrying in the setting of training stochastic neural networks, as we often assume that an stochastic embedding is generated as a deterministic transformation of an input from a finite dataset to which a continuous perturbation is added. We provide an examples where naive computation without ensuring that the product and joint distributions of the two random variables have a well-defined Radon-Nikodym derivative yields negative mutual information.

Let $X \sim U([0,0.1])$, $Z = X + R$ with $R \sim U(\{0,1\})$. Then
\begin{align*}
    \MI{X}{Z} = \Entropy{X} = \log \tfrac{1}{10} \le 0.
\end{align*}
Generally, given $X$ as above and an invertible function $f$ such that $Z = f(X)$, $\MI{X}{Z}=\Entropy{X}$ and can thus be negative. In a way, these cases can be reduced to (degenerate) expressions of the form $\MI{X}{X} = \Entropy{X}$.

We can avoid these cases by adding independent continuous noise.

These examples show that not adding noise can lead to unexpected results. While they still yield finite quantities that bear a relation to the entropies of the random variables, they violate some of the core assumptions we have such that mutual information is always positive.

\section{Experiment details}
\label{app:experiment_details}

\subsection{DNN architectures and hyperparameters}
\label{app:dnn_archs_and_setup}
For our experiments, we use PyTorch \citep{Paszke2019} and the Adam optimizer \citep{Kingma2015}. In general, we use an initial learning rate of $0.5 \times 10^{-3}$ and multiply the learning rate by $\sqrt{0.1}$ whenever the loss plateaus for more than 10 epochs for CIFAR-10. For MNIST and Permutation MNIST, we use an initial learning rate of $10^{-4}$ and multiply the learning rate by 0.8 whenever the loss plateaus for more than 3 epochs.

Sadly, we deviate from this in the following experiments: when optimizing the decoder uncertainty for \emph{categorical Z} for CIFAR-10, we used 5 epochs patience for the decoder uncertainty objective and a initial learning rate of $10^{-4}$. We do not expect this difference to affect the qualitative results mentioned in \secref{app:optimal decoders} when comparing to other objectives. We also only used 5 epochs patience when comparing the two cross-entropies on CIFAR-10 in \secref{sec:decoder uncertainty}. As this was used for both sets of experiments, it does not matter.

We train the experiments for creating the information plane plots for 150 epochs. The toy experiment (\figref{fig:entropy_minimization_and_noise}) is trained for 20 epochs. All other experiments train for 100 epochs.

We use a batchsize of 128 for most experiments. We use a batchsize of 32 for comparing the cross-entropies for CIFAR-10 (where we take 8 dropout samples each), and a batchsize of 16 for MNIST (where we take 64 dropout samples each).

For MNIST, we use a standard dropout CNN, following \url{https://github.com/pytorch/examples/blob/master/mnist/main.py}. For Permutation MNIST, we use a fully-connected model (for experiments with categorical $\Z$ in \secref{app:optimal decoders}): $784 \times 1024 \times 1024\times C$.
For CIFAR-10, we use a regular deterministic ResNet18 model \citep{He2016} for the experiments in \secref{app:optimal decoders}. (As the model outputs a categorical distribution it becomes stochastic through that and we don't need stochasticity in the weights.) For the other experiments as well as the Imagenette experiments, we use a ResNet18v2 \citep{HeV22016}. When we need a stochastic model for CIFAR-10 (for \emph{continuous} $\Z$), we add DropConnect \citep{wan2013regularization} with rate 0.1 to all but the first convolutional layers and dropout with rate 0.1 before the final fully-connected layer. Because of memory issues, we reuse the dropout masks within one batch. The model trains to 94\% accuracy on CIFAR-10.

For CIFAR-10, we always remove the maximum pooling layer and change the first convolutional layer to have kernel size 3 with stride 1 and padding 1. We also use dataset augmentation during training, but not during evaluation on the training set and test set for purposes of computing metrics. We crop randomly after the padding the training images by 4 pixels in every direction and randomly flip images horizontally.

We generally sample 30 values of $\gamma$ for the information plane plots from the specified ranges, using a log scale. For the ablation studies mentioned below, we sample  10 values of $\gamma$ each. We always sample $\gamma = 0$ separately and run a trial with it.

Baselines were tuned by hand (without regularization) using grad-student descent and small grid searches.

\subsection{Cluster setup \& used resources}

We make use of a local SLURM cluster \citep{slurm}. We run our experiments on GPUs (Geforce RTX 2080 Ti). We estimate reproducing all results would take 94 GPU days.

\subsection{Comparison of the surrogate objectives}

\begin{figure*}[h]
    \centering %
	\includegraphics[width=0.9\linewidth]{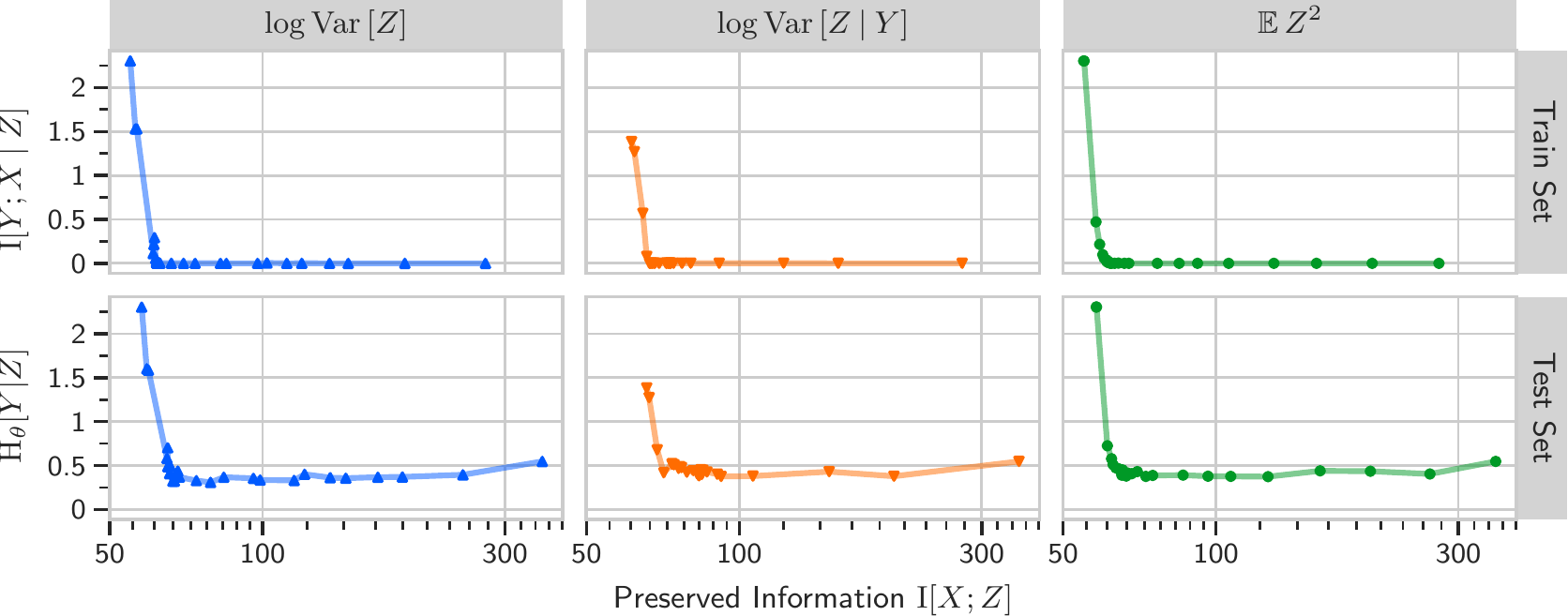} %
    \caption{\emph{Information Plane Plot of the latent $\Z$ similar to \citet{Tishby2015} but using a \emph{ResNet18} model on \emph{CIFAR-10} using the different regularizes from \secref{sec:summary surrogate objectives} (\emph{without dropout, but with \ZeroEntropyNoise}).} The dots are colored by $\gamma$.
    See \secref{sec:exp info plane plots} for more details.}
    \label{fig:cifar10_kraskov_IBP_big}
\end{figure*}

\begin{figure*}[h]
    \includegraphics[width=\linewidth, clip, trim=0 0 0 7]{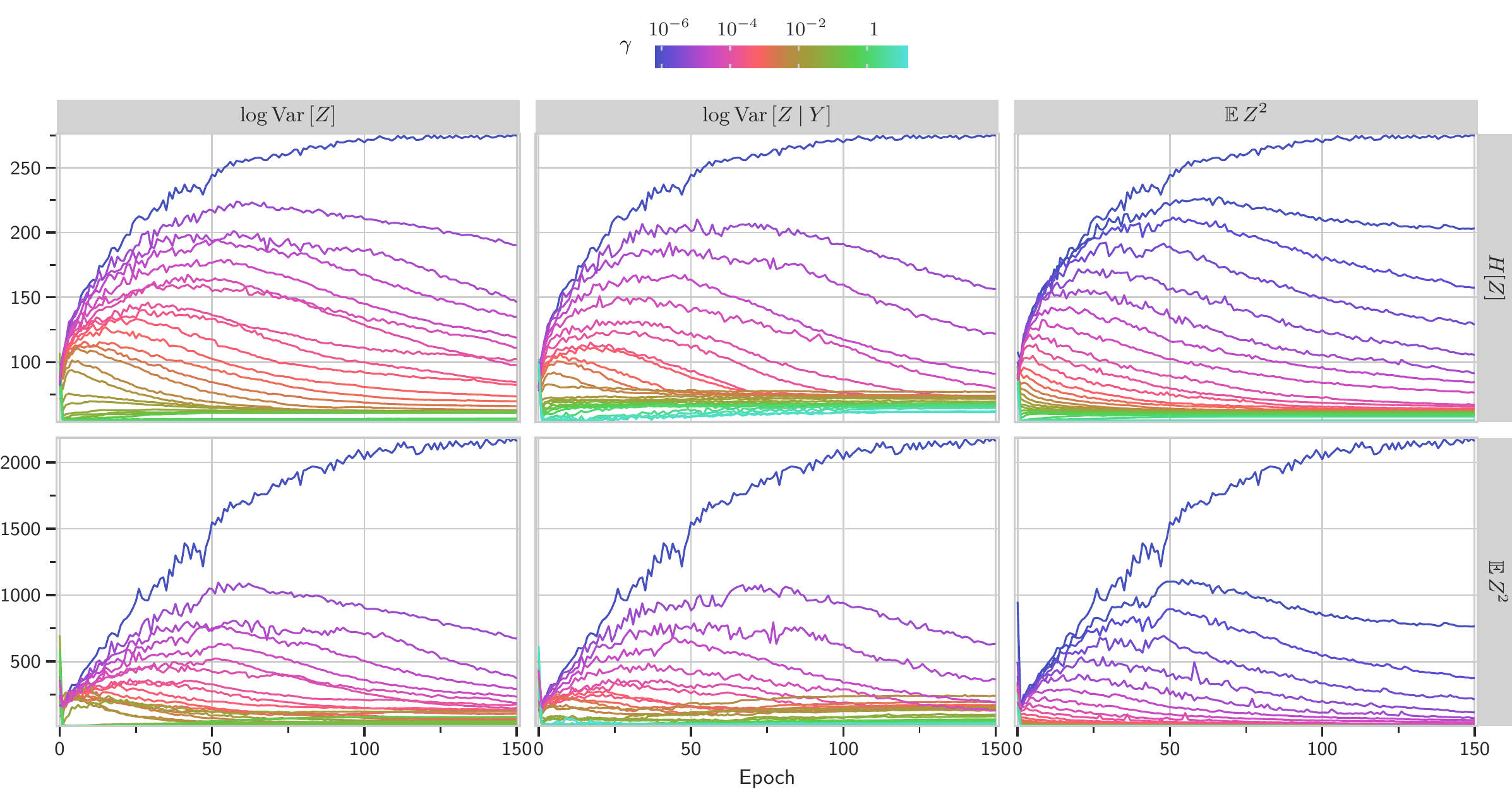}
    \caption{\emph{Entropy estimates while training with different $\gamma$ and with different surrogate regularizers on \emph{CIFAR-10} with a \emph{ResNet18} model.} Entropies are estimated on training data based on \citet{Kraskov2003}. Qualitatively all three regularizers push $\encodingentropy$ and $\rdecoderuncertainty$ down. $\rdecoderuncertainty$ is not shown here because it always stays very close to $\encodingentropy$. $\MSA$ tends to regularize entropies more strongly for small $\gamma$.
    See \secref{sec:exp surrogate regularizers} for more details.
    }
    \label{appfig:cifar10_kraskov_regularizers}
    \vspace{-1.5em}
\end{figure*}

\begin{sidewaysfigure*}[h!]
    \centering
    \includegraphics[width=\linewidth, clip, trim=0 0 0 5]{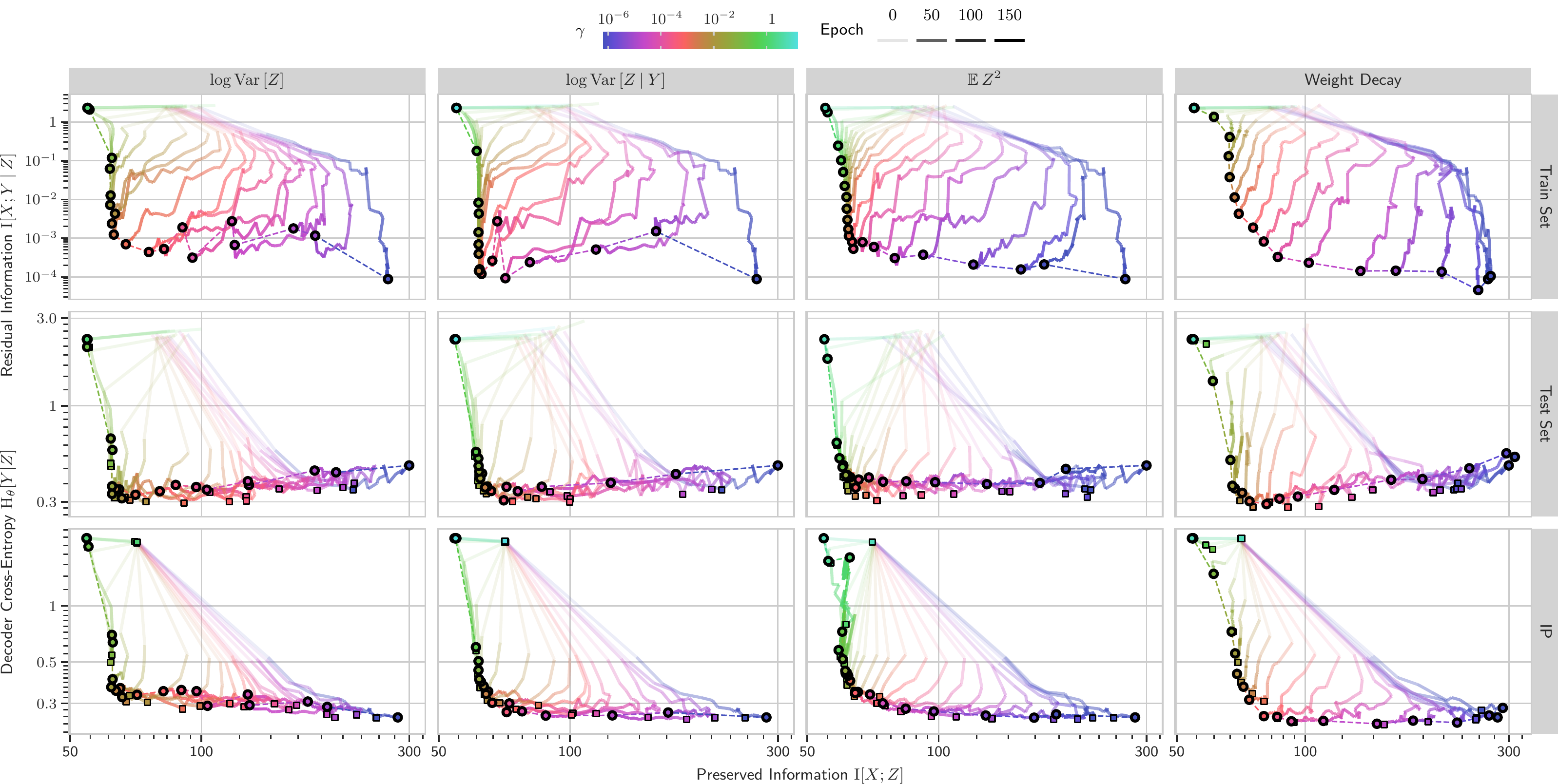}%
    \caption{\emph{\emph{Without dropout but with \ZeroEntropyNoise:} Information Plane Plot of training trajectories for ResNet18 models on CIFAR-10 and different regularizers.}
    The trajectories are colored by their respective $\gamma$; their transparency changes by epoch. %
    Compression (\PreservedInfo $\downarrow)$ trades-off with performance (\ResidualInfo $\downarrow$). See \secref{sec:exp surrogate regularizers}.
    The circle marks the final epoch of a trajectory. The square marks the best epoch (\ResidualInfo $\downdownarrows$).
    }%
    \label{appfig:cifar10_beta_trajectories_no_dropout}%
\end{sidewaysfigure*}

\begin{sidewaysfigure*}[h!]
    \centering
    \includegraphics[width=\linewidth, clip, trim=0 0 0 5]{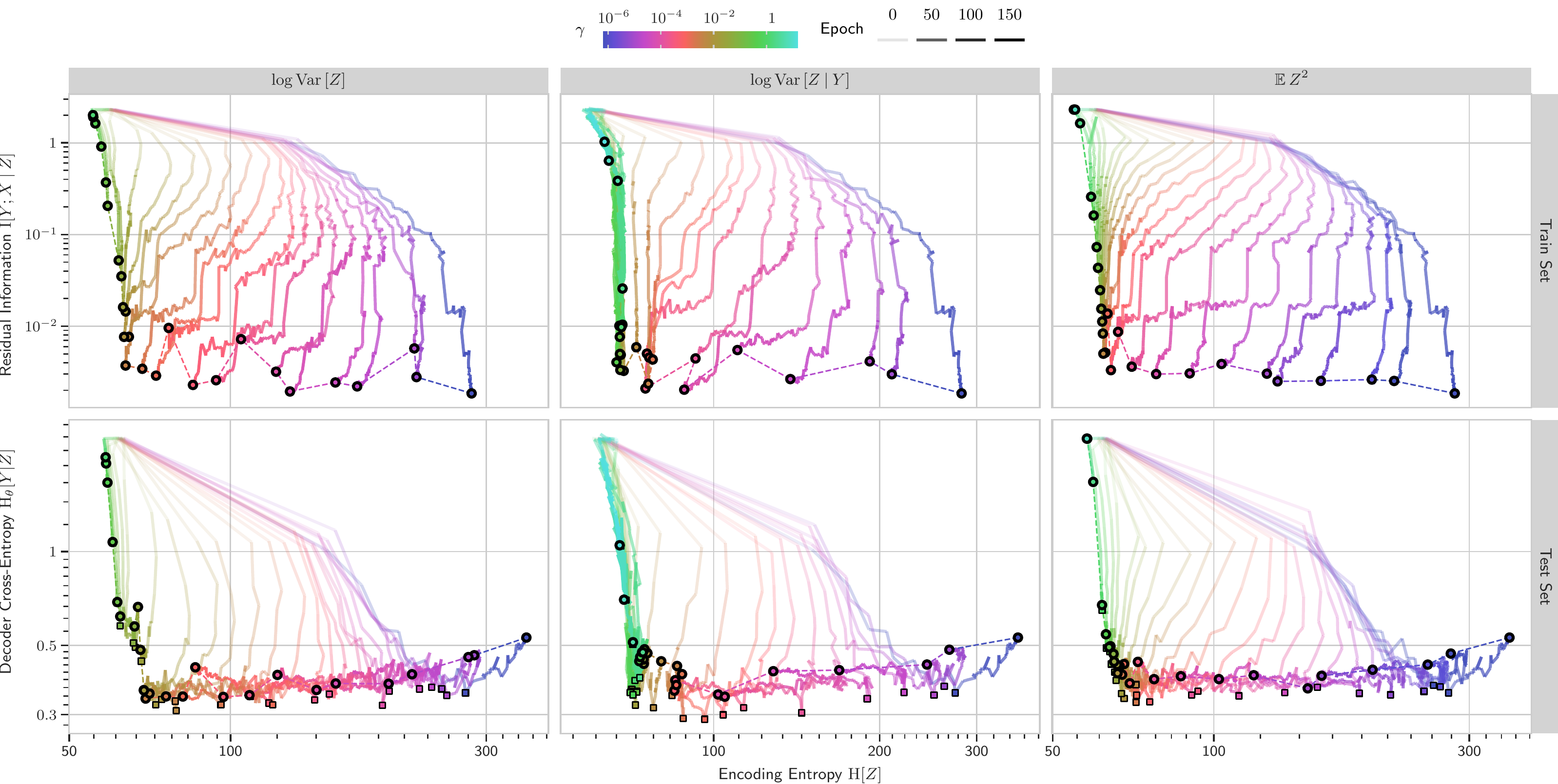}%
    \caption{\emph{\emph{With dropout and with \ZeroEntropyNoise:} Information Plane Plot of training trajectories for ResNet18 models on CIFAR-10 and different regularizers.}
    The trajectories are colored by their respective $\gamma$; their transparency changes by epoch. %
    Compression (\PreservedInfo $\downarrow)$ trades-off with performance (\ResidualInfo $\downarrow$). See \secref{sec:exp surrogate regularizers}.
    The circle marks the final epoch of a trajectory. The square marks the best epoch (\ResidualInfo $\downdownarrows$).
    }%
    \label{appfig:cifar10_beta_trajectories_dropout}%
\end{sidewaysfigure*}

\begin{sidewaysfigure*}[h!]
    \centering
    \includegraphics[width=\linewidth, clip, trim=0 0 0 5]{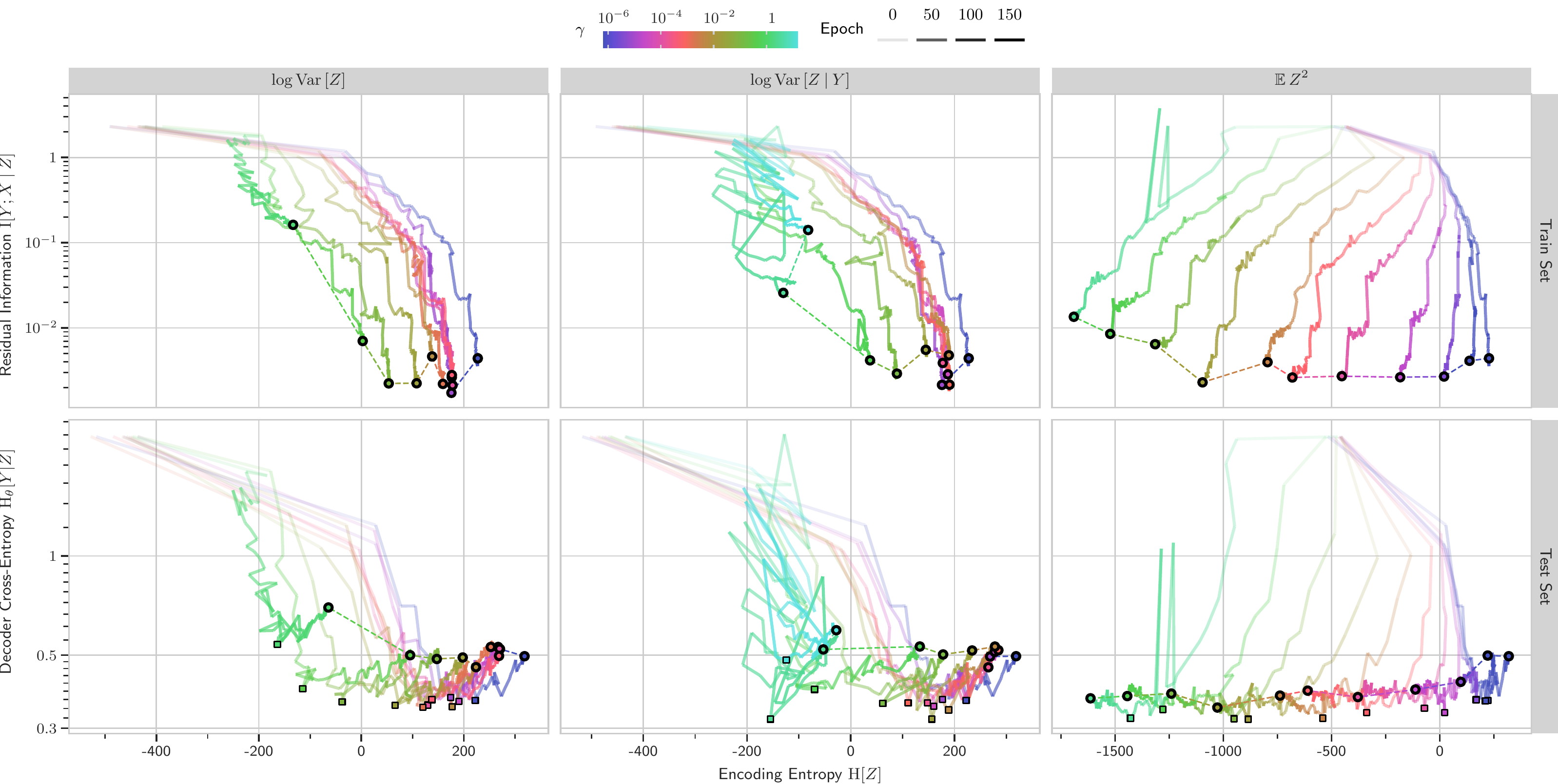}%
    \caption{\emph{\emph{With dropout but without \ZeroEntropyNoise:} Information Plane Plot of training trajectories for ResNet18 models on CIFAR-10 and different regularizers.}
    The trajectories are colored by their respective $\gamma$; their transparency changes by epoch. %
    Compression (\PreservedInfo $\downarrow)$ trades-off with performance (\ResidualInfo $\downarrow$). See \secref{sec:exp surrogate regularizers}.
    The circle marks the final epoch of a trajectory. The square marks the best epoch (\ResidualInfo $\downdownarrows$).
    }%
    \label{appfig:cifar10_beta_trajectories_no_noise}%
\end{sidewaysfigure*}

\begin{sidewaysfigure*}[h!]
    \centering
    \includegraphics[width=\linewidth, clip, trim=0 0 0 5]{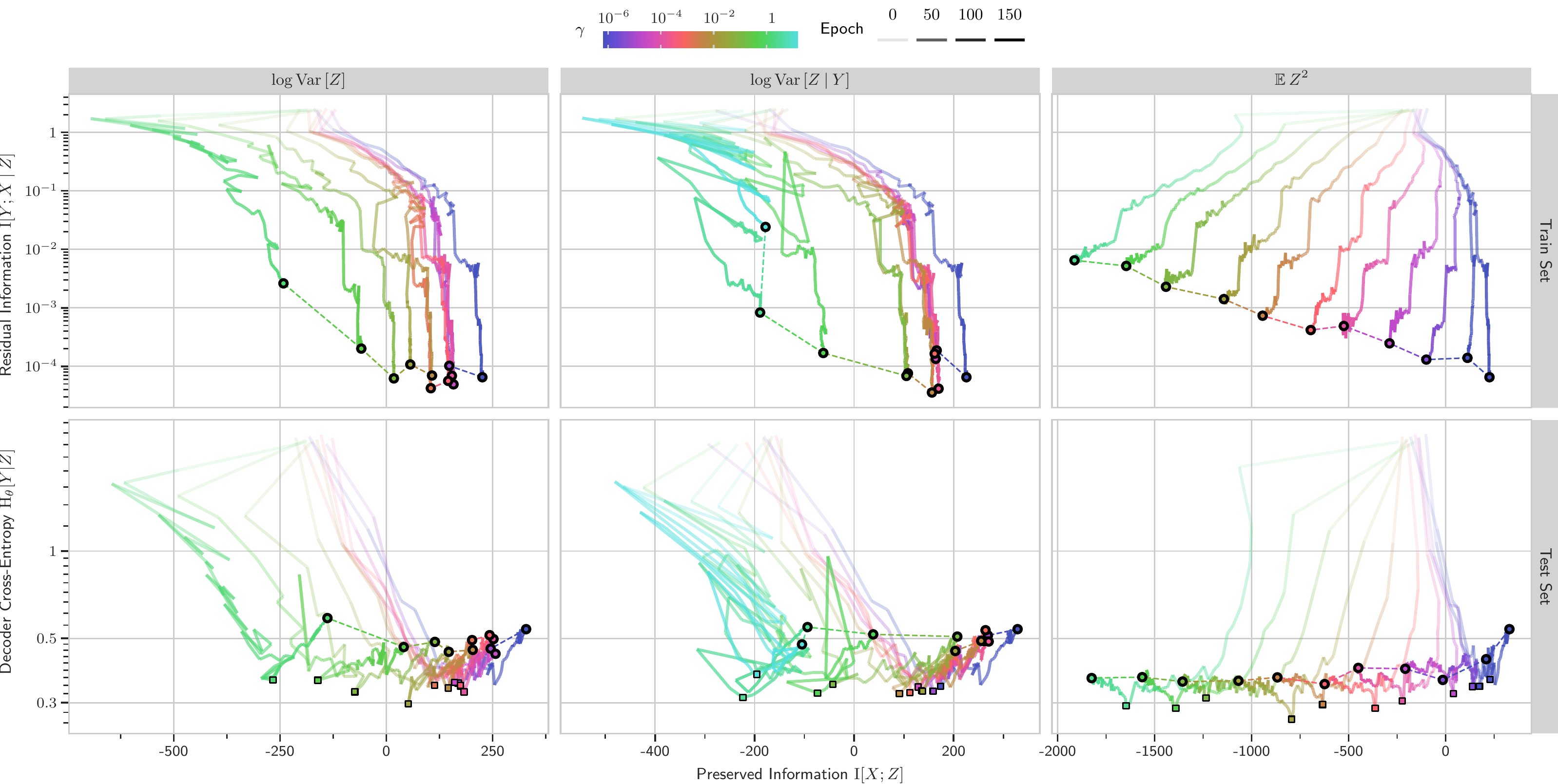}%
    \caption{\emph{\emph{Without dropout and without \ZeroEntropyNoise:} Information Plane Plot of training trajectories for ResNet18 models on CIFAR-10 and different regularizers.}
    The trajectories are colored by their respective $\gamma$; their transparency changes by epoch. %
    Compression (\PreservedInfo $\downarrow)$ trades-off with performance (\ResidualInfo $\downarrow$). See \secref{sec:exp surrogate regularizers}.
    The circle marks the final epoch of a trajectory. The square marks the best epoch (\ResidualInfo $\downdownarrows$).
    }%
    \label{appfig:cifar10_beta_trajectories_no_dropout_no_noise}%
\end{sidewaysfigure*}

\begin{sidewaysfigure*}[h!]
    \centering
    \includegraphics[width=\linewidth, clip, trim=0 0 0 5]{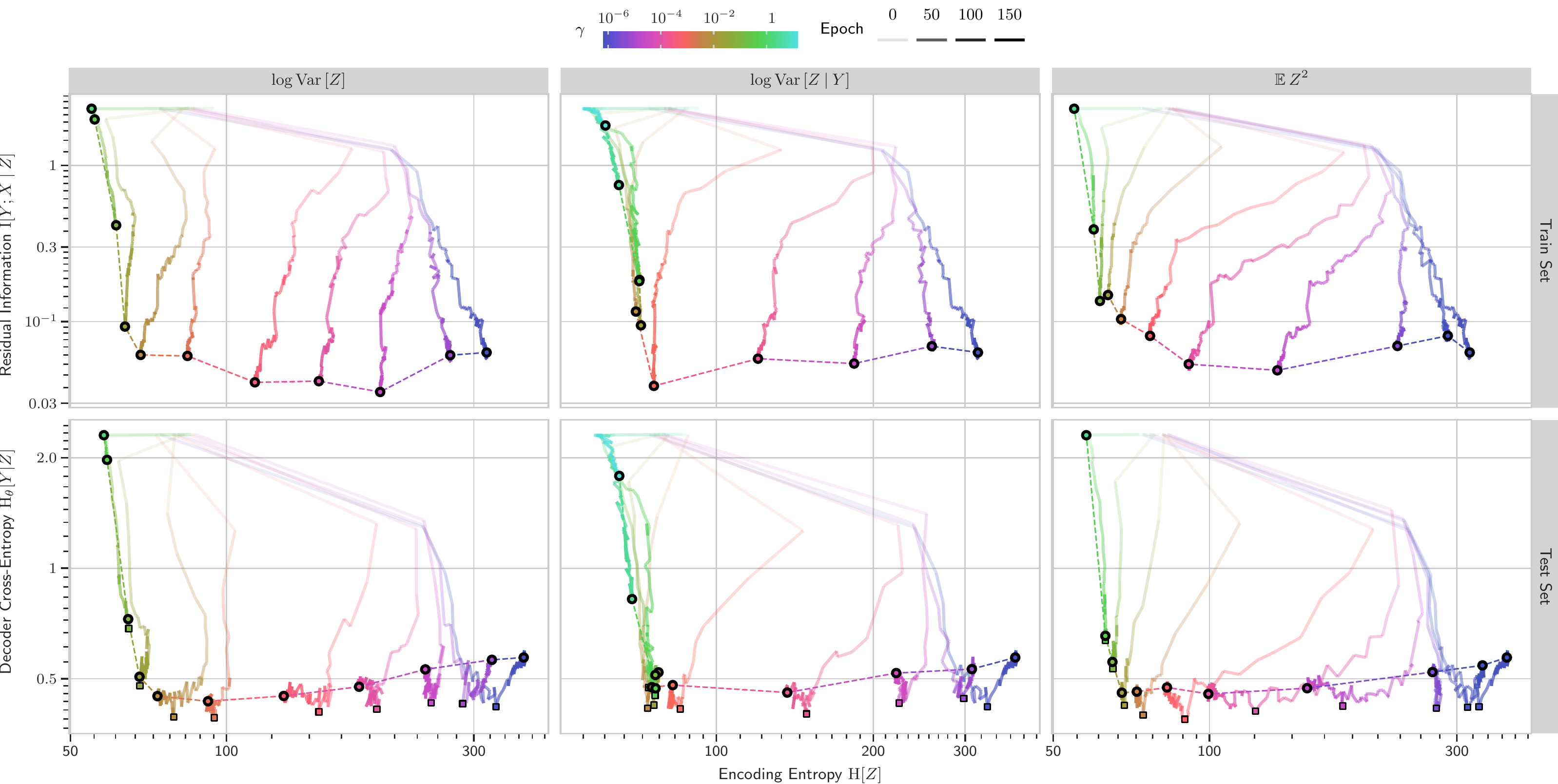}%
    \caption{\emph{\emph{With more dropout and \ZeroEntropyNoise:} Information Plane Plot of training trajectories for ResNet18 models on CIFAR-10 and $\logVarZY$ regularizer with batchsizes 128 and 256.}
    The trajectories are colored by their respective $\gamma$; their transparency changes by epoch. %
    Compression (\PreservedInfo $\downarrow)$ trades-off with performance (\ResidualInfo $\downarrow$). See \secref{sec:exp surrogate regularizers}.
    The circle marks the final epoch of a trajectory. The square marks the best epoch (\ResidualInfo $\downdownarrows$). A DropConnect rate of 0.3 and dropout rate of 0.4 were used instead of 0.1 for each.
    }%
    \label{appfig:cifar10_beta_trajectories_more_dropout}%
\end{sidewaysfigure*}

\begin{figure*}[h!]
    \centering
    \includegraphics[width=\linewidth, clip, trim=0 0 0 5]{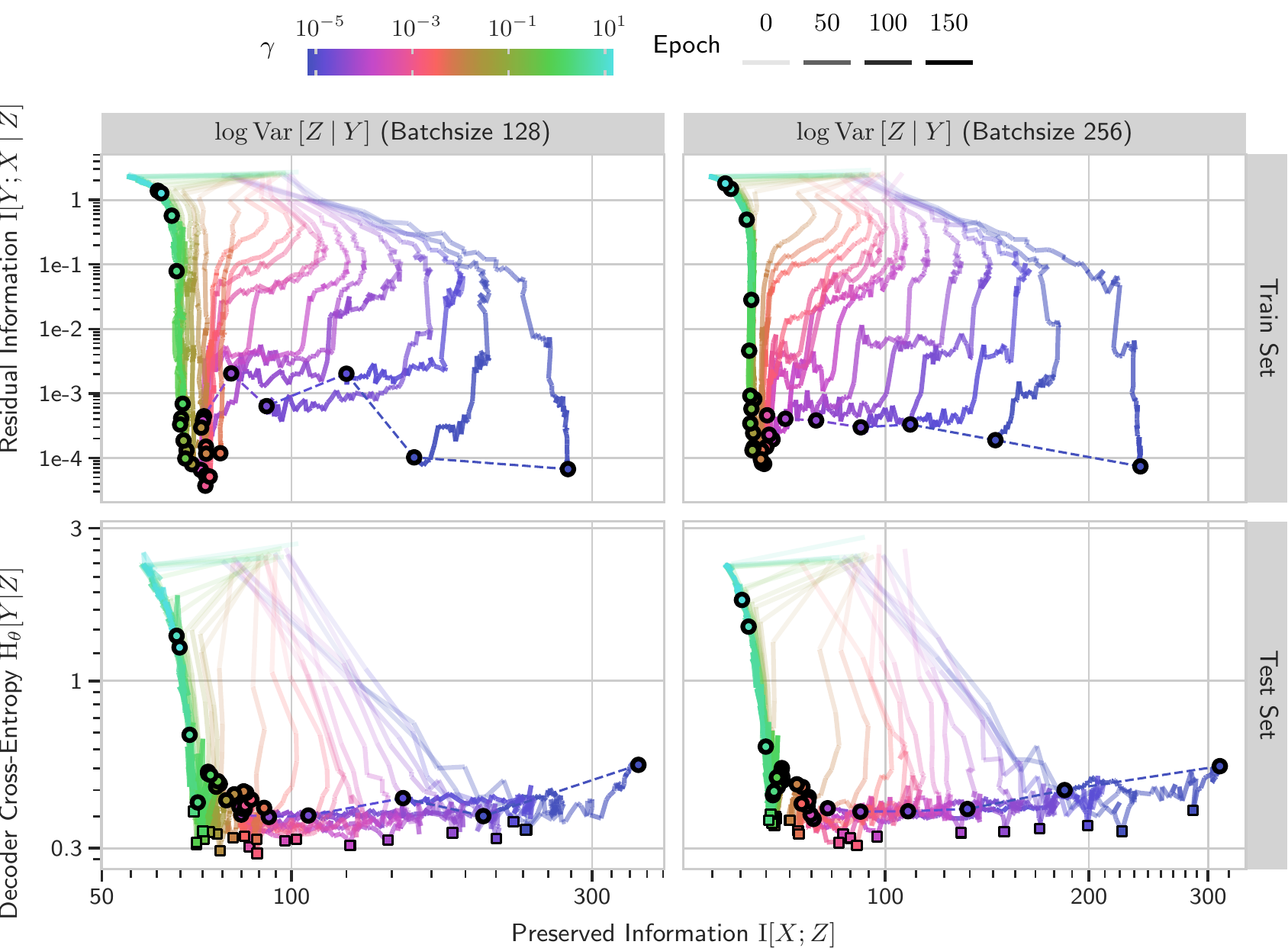}%
    \caption{\emph{\emph{Without dropout but with \ZeroEntropyNoise:} Information Plane Plot of training trajectories for ResNet18 models on CIFAR-10 and $\logVarZY$ regularizer with batchsizes 128 and 256.}
    The trajectories are colored by their respective $\gamma$; their transparency changes by epoch. %
    Compression (\PreservedInfo $\downarrow)$ trades-off with performance (\ResidualInfo $\downarrow$). See \secref{sec:exp surrogate regularizers}.
    The circle marks the final epoch of a trajectory. The square marks the best epoch (\ResidualInfo $\downdownarrows$).
    }%
    \label{appfig:cifar10_beta_trajectories_no_dropout_batchsize}%
\end{figure*}

\begin{figure*}[h]
    \centering %
	\includegraphics[width=\linewidth]{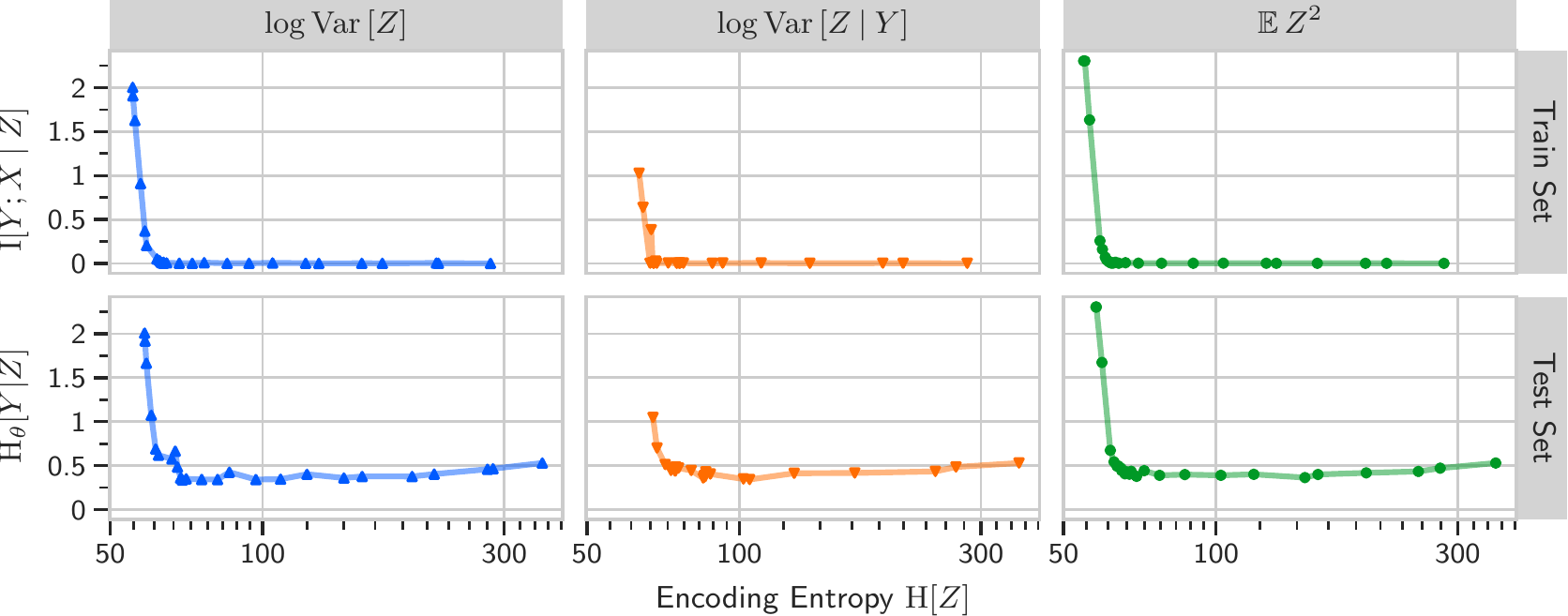} %
    \caption{\emph{Information Plane Plot of the latent $\Z$ similar to \citet{Tishby2015} but using a \emph{ResNet18} model on \emph{CIFAR-10} using the different regularizes from \secref{sec:summary surrogate objectives} (\emph{with dropout and \ZeroEntropyNoise}).} The dots are colored by $\gamma$.
    See \secref{sec:exp info plane plots} for more details.}
    \label{fig:cifar10_kraskov_IBP_big_dropout}
\end{figure*}

\begin{figure*}[h!]
    \begin{subfigure}[h]{\textwidth}
        \centering
        \includegraphics[width=0.9\linewidth, clip, trim=0 0 0 0]{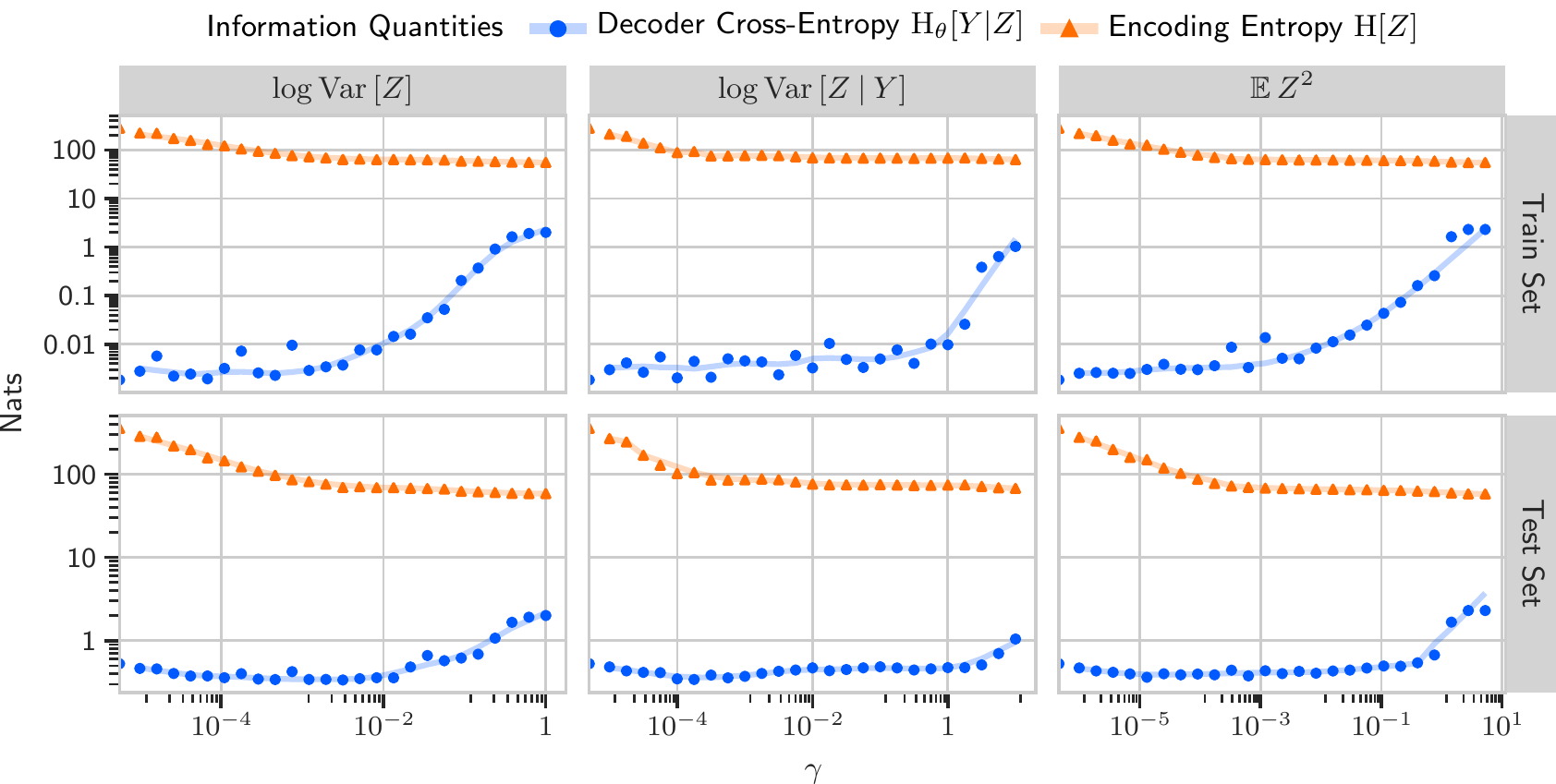}%
        \caption{With dropout and \ZeroEntropyNoise.}%
        \label{appfig:cifar10_kraskov_gamma_appendix_dropout}%
    \end{subfigure}
    \begin{subfigure}[h]{\textwidth}
        \centering
        \includegraphics[width=0.9\linewidth, clip, trim=0 0 0 0]{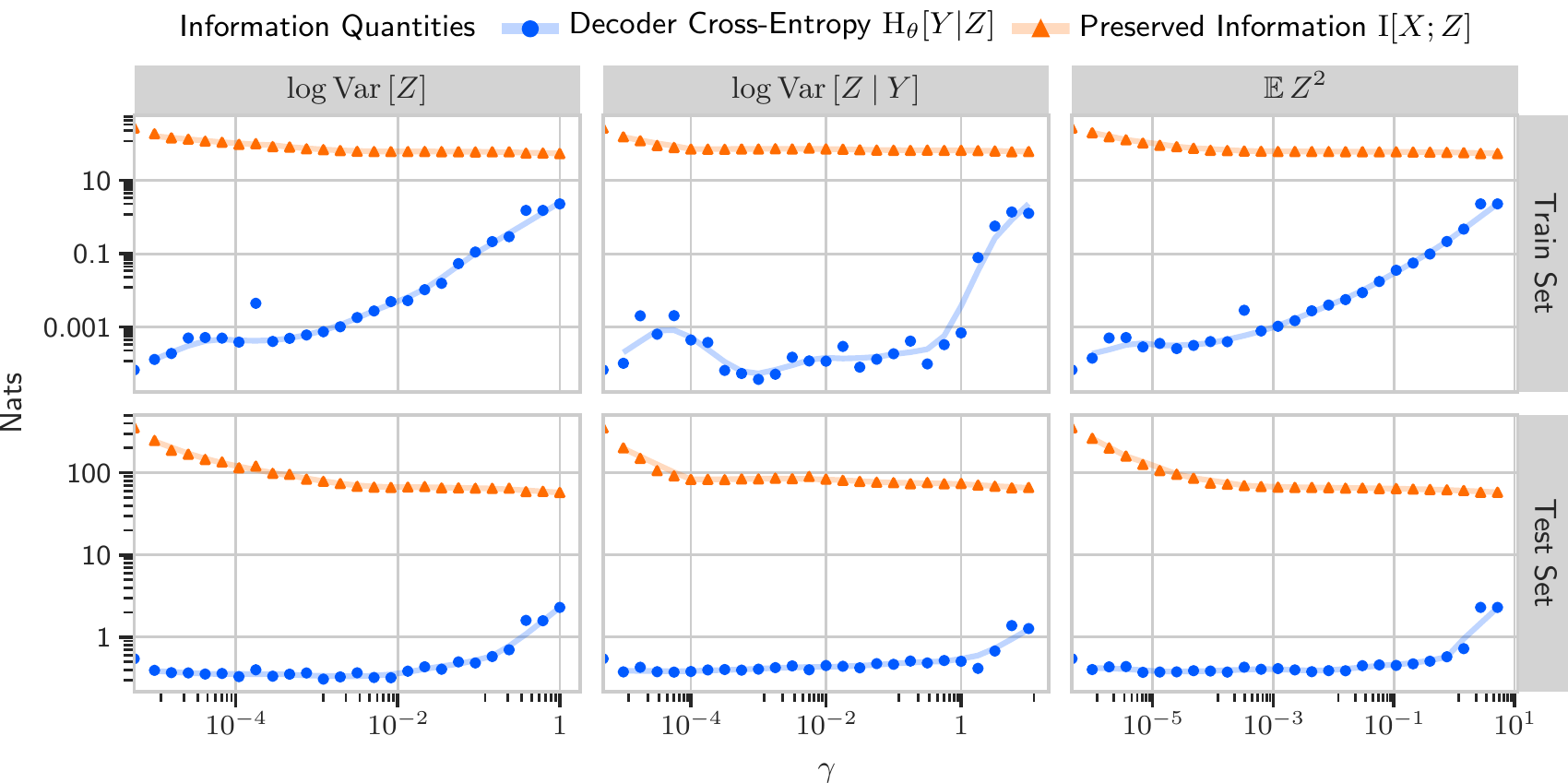}%
        \caption{Without dropout but with \ZeroEntropyNoise.}%
        \label{appfig:cifar10_kraskov_gamma_appendix_no_dropout}%
    \end{subfigure}
    \begin{subfigure}[h]{\textwidth}
        \centering
        \includegraphics[width=0.9\linewidth, clip, trim=0 0 0 0]{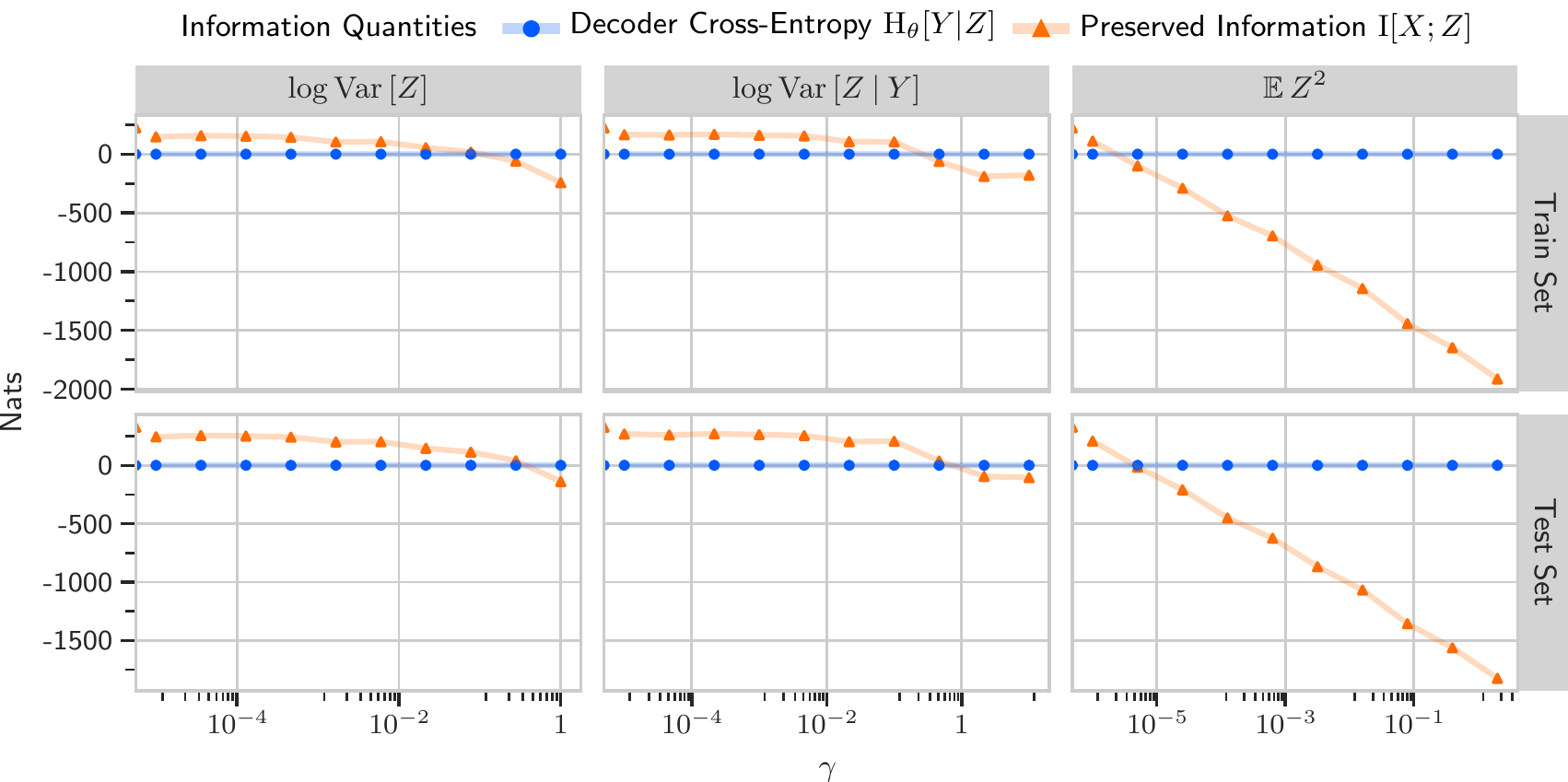}%
        \caption{Without dropout and without \ZeroEntropyNoise.}%
        \label{appfig:cifar10_kraskov_gamma_appendix_no_dropout_no_noise}%
    \end{subfigure}
    \caption{\emph{Information quantites for different $\gamma$ at the end of training for ResNet18 models on CIFAR-10 and $\logVarZY$ regularizer with batchsizes 128 and 256.}
    Compression (\PreservedInfo $\downarrow)$ trades-off with performance (\ResidualInfo $\downarrow$). See \secref{sec:exp surrogate regularizers}.
    }%
    \label{appfig:cifar10_kraskov_gamma_appendix}%
\end{figure*}

\begin{figure*}[t]
    \centering
    \includegraphics[width=0.9\linewidth, clip, trim=0 0 0 0]{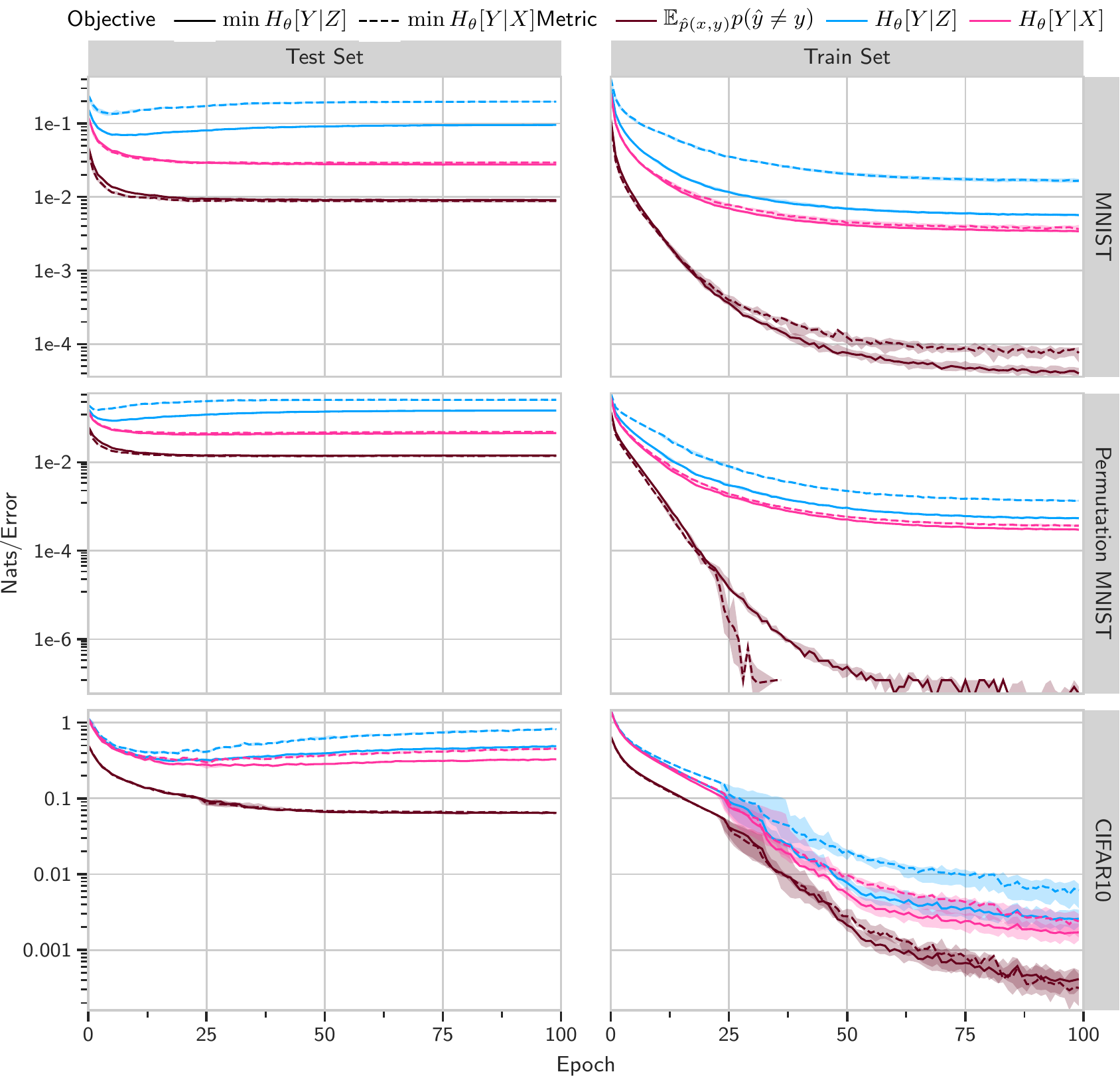}
    \caption{\emph{Training error probability, \DecoderCE $\decoderXE$ and \PredictionCE $\predictXE$ with continuous $\Z$.}  $K=100$ dimensions are used for $\Z$, and we use dropout to obtain stochastic models. 
    Minimizing $\decoderXE$ (solid) leads to smaller cross-entropies and lower training error probability than minimizing $\predictXE$ (dashed). This suggests a better data fit, which is what we desire for a loss term. 
    We run 8 trials each and plot the median with confidence bounds (25\% and  75\% quartiles).
    See \secref{sec:decoder uncertainty} and \ref{sec:exp for cross-entropies} for more details.}
    \label{fig:continuous_bounds_all}
    \vspace{-1.5em}
\end{figure*}

\begin{figure*}[t]
    \centering
    \includegraphics[width=0.9\linewidth, clip, trim=0 0 0 0]{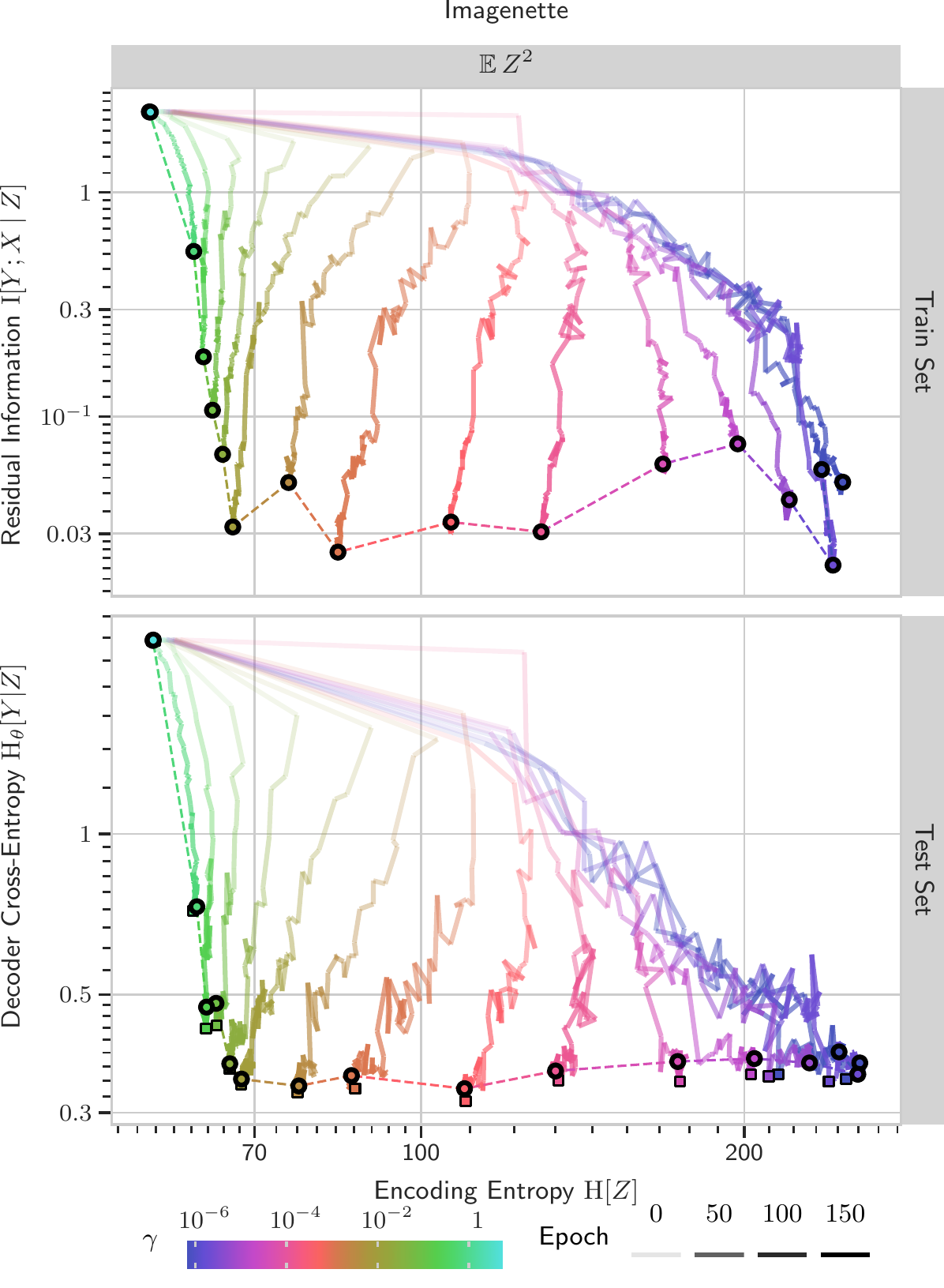}
    \caption{\emph{Information Plane Plot of the latent $\Z$ similar to \citet{Tishby2015} but using a \emph{ResNet18v2} model on \emph{Imagenette} using the $\MSA$ surrogate obejctive from \secref{sec:summary surrogate objectives} (\emph{with dropout and \ZeroEntropyNoise}).} The dots are colored by $\gamma$.} %
    \label{fig:imagenette_train_and_test}
    \vspace{-1.5em}
\end{figure*}

\begin{figure*}[t]
    \centering
    \includegraphics[width=0.9\linewidth, clip, trim=0 0 0 0]{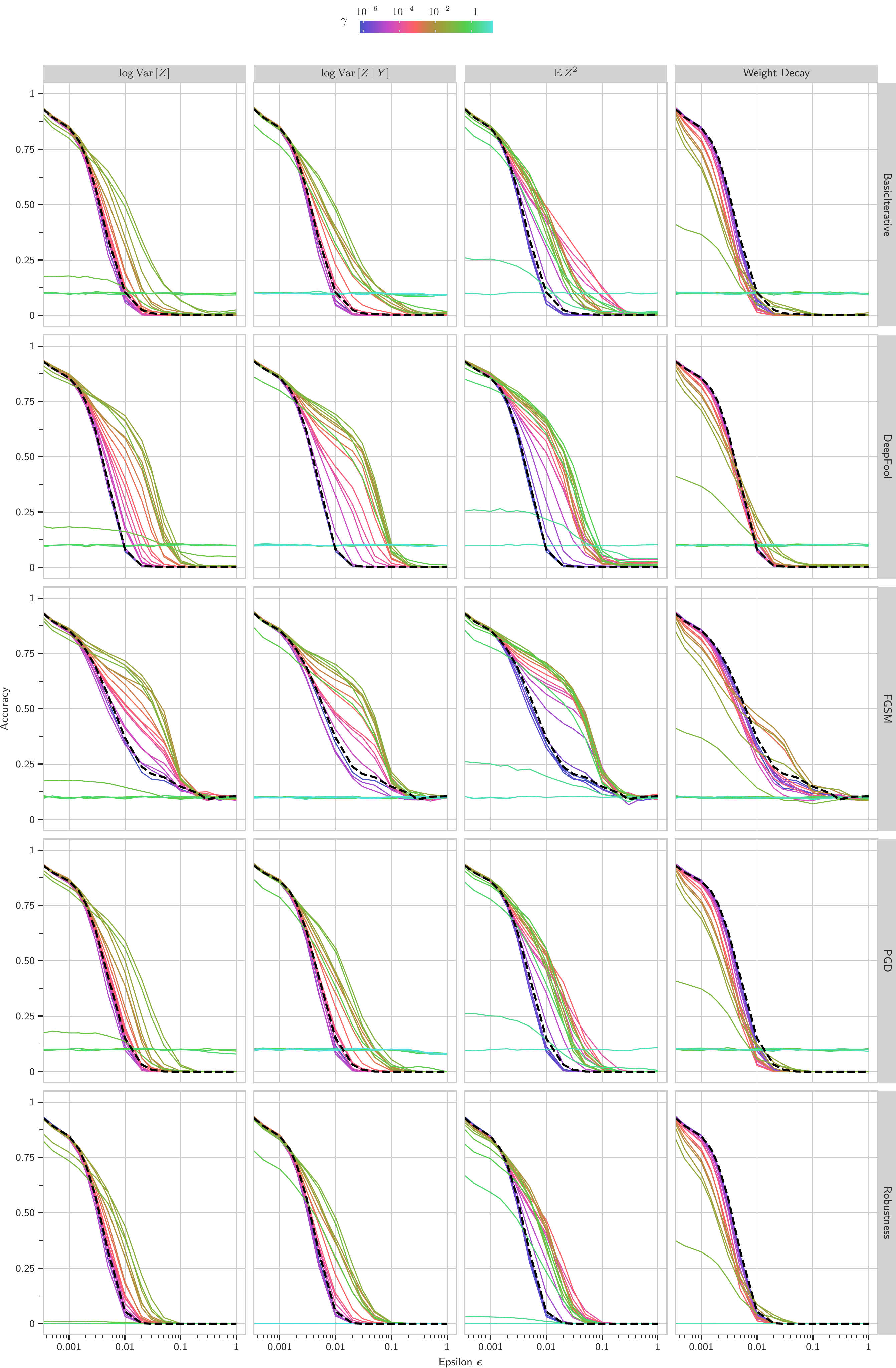}
    \caption{\emph{Adversarial robustness of ResNet18 models trained on CIFAR-10 with surrogate objectives in comparison to regularization with L2 weight-decay as non-IB method for different attack strengths $\eps$.} The robustness is evaluated using FGSM, PGD, DeepFool and BasicIterative attacks of varying $\epsilon$ values. The dashed black line represents a model trained only with cross-entropy and no noise injection. We see that models trained with the surrogate IB objective (colored by $\gamma$) see improved robustness over a model trained only to minimize the cross-entropy training objective (shown in black) while the models regularized with weight-decay actually perform worse.} %
    \label{fig:appendix_robustness}
    \vspace{-1.5em}
\end{figure*}

\begin{sidewaysfigure*}[t]
    \centering
    \includegraphics[width=0.9\linewidth, clip, trim=0 0 0 0]{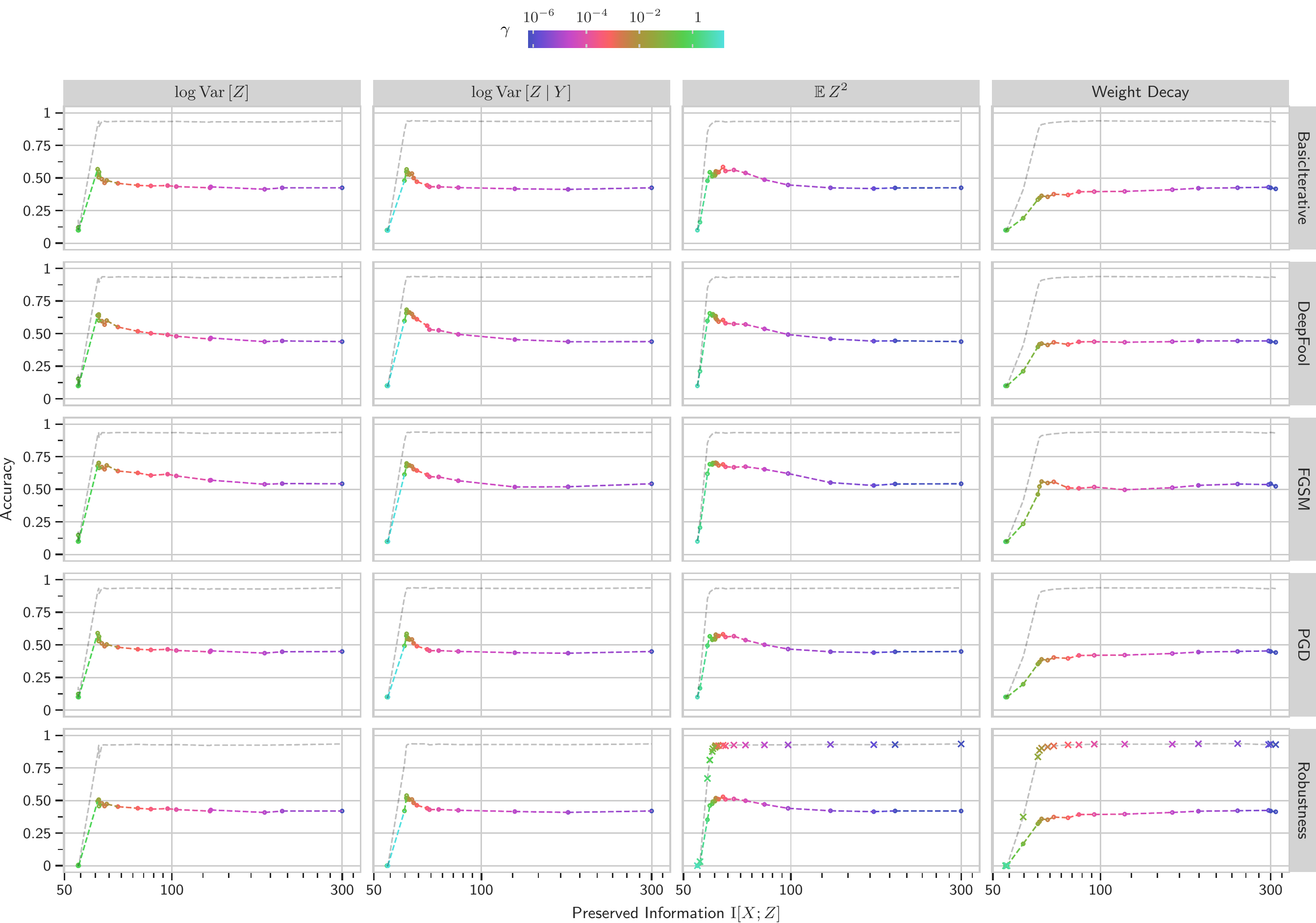}
    \caption{\emph{Average adversarial robustness  over $\eps \in [0, 0.1]$ of ResNet18 models trained on CIFAR-10 with surrogate objectives in or L2 weight-decay (as non-IB method) compared to normal accuracy for different amounts of \PreservedInfo.} $\circ$ markers show robustness. $\times$ markers show the normal accuracy. We see that robustness depends on the \PreservedInfo. If the latent is compressesed too much, robustness (and accuracy) are low. If the latent is not compressed enough, robustness and thus generalization suffer.}%
    \label{fig:appendix_integrated_robustness}
    \vspace{-1.5em}
\end{sidewaysfigure*}

As can be seen in \figref{appfig:cifar10_kraskov_regularizers}, the different surrogate regularizers have very similar effects on $\encodingentropy$ and $\rdecoderuncertainty$.
Regularizing with $\MSA$ shows a stronger initial regularization effect, but is difficult to compare quantitatively as its hyperparameter does not map to an equivalent $\beta$, unlike regularizing using entropy estimates.
Overall, we find $\logVarZ$ to provide stable training trajectories (and expected visualizations) while also having a more meaningful hyperparameter than $\MSA$, though $\MSA$ is trivial to implement and communicate\footnote{Which is the reason why we showcase it in \figref{fig:imagenette_beta_trajectories_test} and in equation \eqref{eq:explicit_loss}.}.
$\logVarZY$ performs worse which we hypothesize is due to the increased variance (given equal batch sizes) from conditioning on  $\Y$.
It further does not minimize the \PreservedInfo $\preservedinfo$ as strongly as the other regularizers.

\subsubsection{Measurement of information quantities}
\label{app:measuring information quantities}

Measuring information quantities can be challenging. As mentionend in the introduction, there are many complex ways of measuring entropies and mutual information terms. We can side-step the issue by making use of the bounds we have established and the zero-entropy noise we are injecting, and design experiments around that.

First, to estimate the \PreservedInfo $\preservedinfo$, we note that when we use a deterministic model as encoder and only inject \ZeroEntropyNoise, we have $\encodinguncertainty = 0$ and $\preservedinfo = \preservedinfo + \encodinguncertainty = \encodingentropy$.
We use the entropy estimator from \citet[equation $(20)$]{Kraskov2003} to estimate the \EncodingEntropy $\encodingentropy$ and thus $\preservedinfo$.

To estimate the \ResidualInfo $\residualinfo$, we similarly note that $\residualinfo = \residualinfo + \labeluncertainty = \decoderuncertainty$. Instead of estimating the entropy using \citet{Kraskov2003}, we can use the \DecoderCE $\decoderXE$ which provides a tighter bound as long as we also minimize $\decoderXE$ as part of the training objective.

When we use stochastic models as encoder, we cannot easily compute $\preservedinfo$ anymore. In the ablation study in the next section, we thus change the X axis accordingly.

Similarly, when we look at the trajectories on the test set instead of the training set, for example in \figref{appfig:cifar10_beta_trajectories_dropout}, we change the Y axis to signify the \DecoderUncertainty $\decoderXE$. It is still an upper-bound, but we do not minimize it directly anymore.

For the plots in \figref{fig:cifar10_beta_trajectories}, we retrained the decoder on the test set to obtain a tighter bound on $\decoderuncertainty$ (while keeping the encoder fixed). We then sampled the latent using the test set to estimate the trajectories. We only did this for the CIFAR-10 model without dropout. For our ablations, we did not retrain the decoder and thus only present plots on the test and training set, respectively.

At this point, it is important to recall that the \DecoderUncertainty is also the negative log-likelihood (when training with a single dropout sample), which provides a different perspective on the plots. It makes it clear that we can see how much a model overfits by comparing the best and final epochs of a trajectory in the plot (marked by a circle and a square, respectively).

\subsubsection{Ablation study}
\label{app:ablation study}

We perform an ablation study to determine whether injecting noise is necessary. Furthermore, we investigate the more interesting case of using a stochastic model as encoder, and if we can use a stochastic model without injecting \ZeroEntropyNoise.

We also investigate whether $\logVarZY$ performs better when we increase batchsize as we hypothesized that a batchsize of 128 does not suffice as it leaves only $\approx 13$ samples per class to approximate $\rdecoderuncertainty$).

\Figref{appfig:cifar10_beta_trajectories_no_dropout} shows a larger version of \figref{fig:cifar10_beta_trajectories} for all three regularizers and also training trajectories on the test set. As described in the previous section, this allows us to validate that the regularizers prevent overfitting on the training set: with increasing $\gamma$, the model overfits less.

\Figref{appfig:cifar10_beta_trajectories_no_dropout_no_noise} and \figref{appfig:cifar10_beta_trajectories_no_noise} shows that injecting noise is necessary independently of whether we use dropout or not. Regularizing with $\MSA$ still has a very weak effect. We hypothesize that, similar to the toy experiment depicted in \figref{fig:entropy_minimization_and_noise}, floating-point precision issues might provide a natural noise source eventually. This would change the effectiveness of $\gamma$ and might require much higher values to observe similar regularization effects as when we do inject \ZeroEntropyNoise.

\Figref{appfig:cifar10_beta_trajectories_dropout} shows trajectories for a stochastic encoder (as described above with DropConnect/dropout rate 0.1). It overfits less than a deterministic one.

\Figref{appfig:cifar10_beta_trajectories_more_dropout} shows the effects of using higher dropout rates (using DropConnect/dropout rates of 0.3/0.5). It overfits less than model with DropConnect/dropout rates of 0.1/0.1.

The plots in \figref{appfig:cifar10_kraskov_gamma_appendix} show the effects of different $\gamma$ with different regularizers more clearly. On both training and test set, one can clearly see the effects of regularization.

Overall, $\logVarZY$ performs worse as a regularizer. In \figref{appfig:cifar10_beta_trajectories_no_dropout_batchsize}, we compare the effect of doubling batchsize. Indeed, $\logVarZY$ performs better with higher batchsize and looks closer to $\logVarZ$.

\subsubsection{Comparison between \DecoderCE and \PredictionCE}
\label{sec:exp for cross-entropies}
When training deterministic models or dropout models with a single sample (as one usually does), the estimators for both the \DecoderCE $\decoderXE$ and the \PredictionCE $\predictXE$ coincide. In \secref{sec:decoder uncertainty}, we discuss the differences from a theoretical perspective. Here, we empirically evaluate the difference between optimizing the estimators for each of the two cross-entropy losses, for which we will draw multiple dropout samples during training and inference.

We examine models with \emph{continuous} $\Z$ on MNIST and CIFAR-10 \citep{LeCun1998, Krizhevsky2009}. Specifically, we use a standard dropout CNN as an encoder for MNIST, and a modified ResNet18 to which we add DropConnect in each layer for CIFAR-10. We use $K=100$ dimensions for the continuous latent $\Z$ in the last fully-connected layer, and use a linear decoder to obtain the final $10$-dimensional output of class logits. For MNIST, we compute the cross-entropies using 64 dropout samples; for CIFAR-10, we use 8. For the purpose of this examination of training behavior, it is not necessary to achieve SOTA accuracy:
our models obtain 99.2\% accuracy on MNIST and 93.6\% on CIFAR-10.

\Figref{fig:continuous_bounds_all} shows the training error probability as well as the value of each cross-entropy loss for models trained either with the \DecoderCE or the \PredictionCE. The \DecoderCE $\decoderXE$ outperforms \PredictionCE $\predictXE$ as a training objective: the training error probability and both cross-entropies are lower when minimizing $\decoderXE$ compared to minimizing $\predictXE$. We compare only the training, rather than the test, losses of the models to isolate the effect of each loss term on training performance; we leave the prevention of overfitting to the regularization terms considered later.
Recently, \citet{Dusenberry2020} also observed empirically that the \DecoderCE $\decoderXE$ as an objective is both easier to optimize and provides better generalization performance.

\subsection{Differential entropies and noise}
\label{sec:exp minimizing entropy}
We demonstrate the importance of adding noise to continuous latents by constructing a pathological sequence of parameters which attain monotonically improving and unbounded regularized objective values ($\encodingentropy$) while all computing \textit{the same function}. We use MNIST with a standard dropout CNN as encoder, with $K=128$ continuous dimensions in $\Z$, and a $K \times 10$ linear layer as decoder. 
After every training epoch, we decrease the entropy of the latent by normalizing and then scaling the latent to bound the entropy. We multiply the weights of the decoder to not change the overall function. As can be seen in \figref{fig:entropy_minimization_and_noise}, without noise, entropy can decrease freely during training without change in error rate until it is affected by floating-point issues;
while when adding \ZeroEntropyNoise, the error rate starts increasing gradually and meaningfully as the entropy starts to approach zero.
We conclude that entropy regularization is meaningful only when noise is added to the latent.

\subsection{Comparison between DVIB and surrogate objectives on Permutation-MNIST}
\label{app:dvib_vs_uib}

Comparing \DVIB and our surrogate objectives is not straightforward because \DVIB uses a VAE-like model that explicitly parameterize mean and standard deviation of the latent whereas the stochastic models we focus on in \secref{sec:decoder uncertainty} and beyond are implicit by using dropout. 

For this comparison, we use the same architecture and optimization strategy for \DVIB as described in \citet{Alemi2019}: the encoder is a ReLU-MLP of the form $796-1024-1024-2K$ with K=256 latent dimensions that outputs mean and standard deviation explicitly and separately. For the standard deviation, we use a softplus transform with a bias of $-5$. We use Polyak averaging with a decay constant of $0.999$ \citep{polyak1992acceleration}. We train the model for 200 epochs with Adam with learning rate $10^{-4}$, $\beta_1=0.5, \beta_2=-.999$ \citep{Kingma2015} and decay the learning rate by 0.97 every 2 epochs. The marginal is fixed to a unit Gaussian around the origin. We use a softmax layer as decoder. We use 12 latent samples during training and test time.

For our surrogate objectives, we use a similar ReLU-MLP of the form $796-1024-1024-K$ with K=256 latent dimensions and dropout layers of rate 0.3 after the first and second layer. We use also 12 dropout samples during training and test time. We train for 75 epochs with Adam and learning rate $0.5 \times 10^{-4}$. We half the learning rate every time the loss does not decrease for 13 epochs.

\begin{figure}
    \centering
    \includegraphics[clip, trim=0 0 0 5]{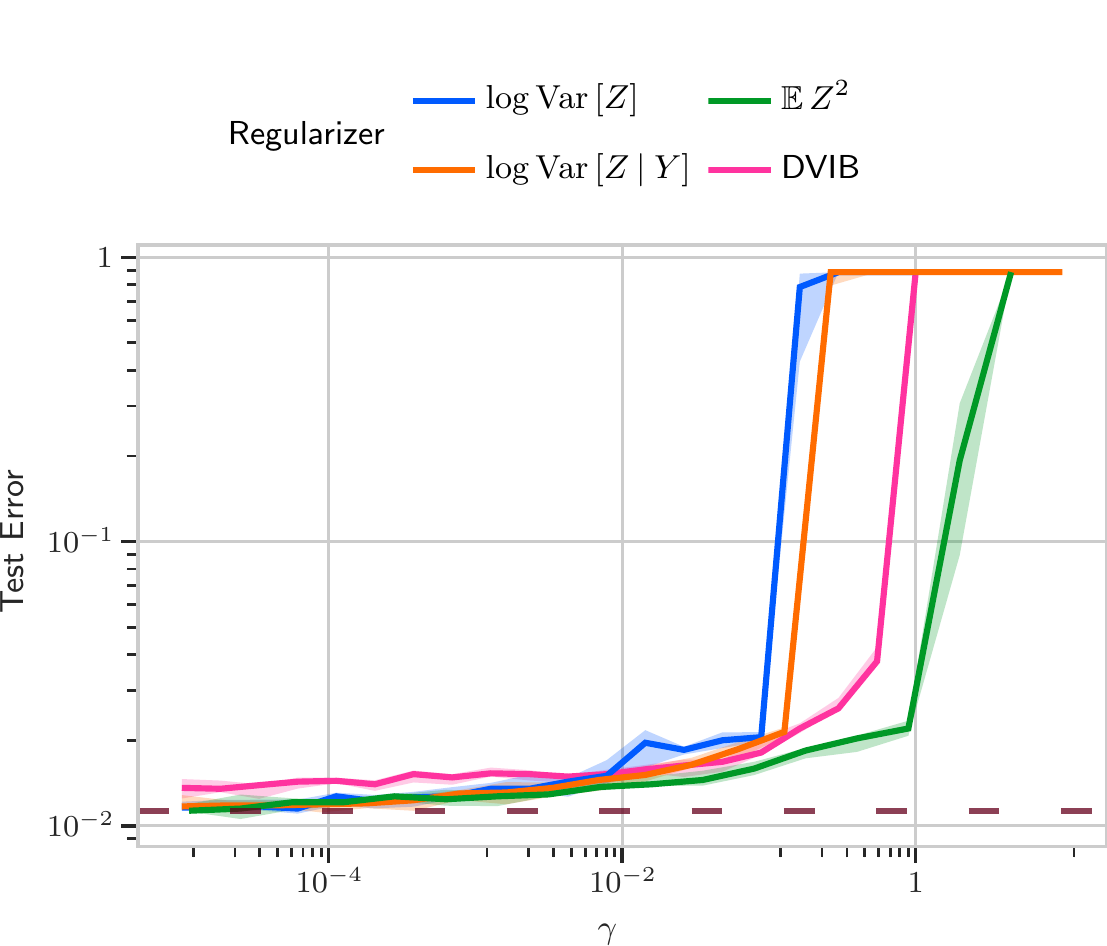}%
    \caption{\emph{Comparison of test error for different Lagrange multiplier for \DVIB and surrogate objectives from \secref{sec:summary surrogate objectives} on Permutation-MNIST.} The purple strongly dashed line shows the test error reported for \DVIB in \citet{Alemi2019}. 5 trials with 95\% confidence interval shown. Even though we could not reproduce the baseline reported in that paper, the simpler surrogate objective reach at least a similar test error as reported there. We also see that \DVIB behaves similar to $\MSA$, but shifted by a factor $2$ in $\gamma$, as predicted by \secref{app:alemia and msa}.%
    }%
    \label{appfig:dvib_vs_uib}
    \vspace{-1.5em}
\end{figure}

We run 5 trials for each experiments. We were not able to reproduce the baseline of an error of $1.13\%$ for $\beta=10^{-3}$ from \citet{Alemi2019}. We show a comparison in \figref{appfig:dvib_vs_uib}. Our methods do reach an error of $1.13\%$ overall though, so the simpler surrogate objectives perform as well good or better than \DVIB.

From \secref{app:alemia and msa}, we know that \DVIB's $\beta$ would have to be twice the $\gamma$ frm our \secref{sec:summary surrogate objectives}. We can see this correspondence in the plot. This also implies that \DVIB's $\beta$ is not related to the IB objective's $\beta$ from \secref{sec:information bottlenecks}. This makes sense as \DVIB arbitrarily fixes the marginal to be a unit Gaussian.

\begin{sidewaysfigure*}[h]
    \section{Large Version of the Mickey Mouse I-Diagram}
    \centering %
	\includegraphics[width=0.8\textwidth]{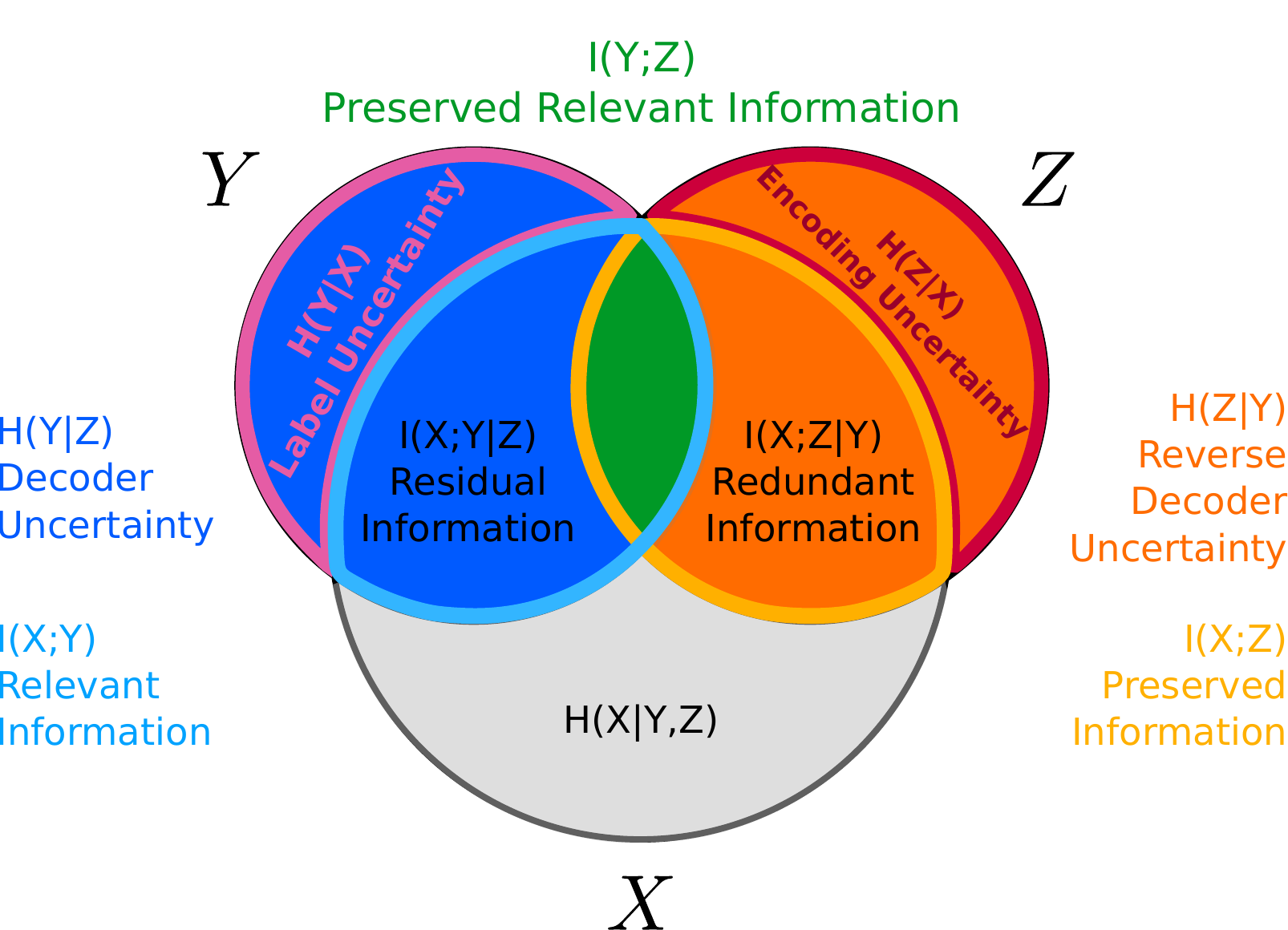} %
    \caption{\emph{Mickey Mouse I-diagram.} See \figref{fig:mickeymouse} for details.}%
    \label{appfig:big mickey mouse}
\end{sidewaysfigure*}%

\end{document}

%% file: math_commands.tex
\usepackage{amsmath,amsfonts,bm}

\def\figref#1{figure~\ref{#1}}
\def\Figref#1{Figure~\ref{#1}}

\def\secref#1{section~\ref{#1}}
\def\Secref#1{Section~\ref{#1}}

\def\1{\bm{1}}

\def\eps{{\epsilon}}

\DeclareMathAlphabet{\mathsfit}{\encodingdefault}{\sfdefault}{m}{sl}
\SetMathAlphabet{\mathsfit}{bold}{\encodingdefault}{\sfdefault}{bx}{n}

\newcommand{\KL}{D_{\mathrm{KL}}}
\newcommand{\Var}{\mathrm{Var}}

\newcommand{\Cov}{\mathrm{Cov}}

\DeclareMathOperator*{\argmin}{arg\,min}

%% file: arxiv_submission.bbl
\begin{thebibliography}{61}
\providecommand{\natexlab}[1]{#1}
\providecommand{\url}[1]{\texttt{#1}}
\expandafter\ifx\csname urlstyle\endcsname\relax
  \providecommand{\doi}[1]{doi: #1}\else
  \providecommand{\doi}{doi: \begingroup \urlstyle{rm}\Url}\fi

\bibitem[Achille and Soatto(2018{\natexlab{a}})]{Achille2017}
Alessandro Achille and Stefano Soatto.
\newblock Emergence of invariance and disentanglement in deep representations.
\newblock \emph{The Journal of Machine Learning Research}, 19\penalty0
  (1):\penalty0 1947--1980, 2018{\natexlab{a}}.

\bibitem[Achille and Soatto(2018{\natexlab{b}})]{Achille2018}
Alessandro Achille and Stefano Soatto.
\newblock Information dropout: Learning optimal representations through noisy
  computation.
\newblock \emph{IEEE transactions on pattern analysis and machine
  intelligence}, 40\penalty0 (12):\penalty0 2897--2905, 2018{\natexlab{b}}.

\bibitem[Alemi et~al.(2016)Alemi, Fischer, Dillon, and Murphy]{Alemi2019}
Alexander~A Alemi, Ian Fischer, Joshua~V Dillon, and Kevin Murphy.
\newblock Deep variational information bottleneck.
\newblock \emph{arXiv preprint arXiv:1612.00410}, 2016.

\bibitem[Amjad and Geiger(2019)]{AliAmjad2018}
Rana~Ali Amjad and Bernhard~Claus Geiger.
\newblock Learning representations for neural network-based classification
  using the information bottleneck principle.
\newblock \emph{IEEE Transactions on Pattern Analysis and Machine
  Intelligence}, 2019.

\bibitem[Belghazi et~al.(2018)Belghazi, Baratin, Rajeshwar, Ozair, Bengio,
  Courville, and Hjelm]{Belghazi2018}
Mohamed~Ishmael Belghazi, Aristide Baratin, Sai Rajeshwar, Sherjil Ozair,
  Yoshua Bengio, Aaron Courville, and Devon Hjelm.
\newblock Mutual information neural estimation.
\newblock In \emph{International Conference on Machine Learning}, pages
  531--540, 2018.

\bibitem[Bengio et~al.(2009)]{Bengio2009}
Yoshua Bengio et~al.
\newblock Learning deep architectures for ai.
\newblock \emph{Foundations and trends{\textregistered} in Machine Learning},
  2\penalty0 (1):\penalty0 1--127, 2009.

\bibitem[Bercher and Vignat(2002)]{Bercher}
J-F Bercher and Christophe Vignat.
\newblock A renyi entropy convolution inequality with application.
\newblock In \emph{2002 11th European Signal Processing Conference}, pages
  1--4. IEEE, 2002.

\bibitem[Burda et~al.(2015)Burda, Grosse, and Salakhutdinov]{Burda}
Yuri Burda, Roger Grosse, and Ruslan Salakhutdinov.
\newblock Importance weighted autoencoders.
\newblock \emph{arXiv preprint arXiv:1509.00519}, 2015.

\bibitem[Burgess et~al.(2018)Burgess, Higgins, Pal, Matthey, Watters,
  Desjardins, and Lerchner]{burgess2018understanding}
Christopher~P. Burgess, Irina Higgins, Arka Pal, Loic Matthey, Nick Watters,
  Guillaume Desjardins, and Alexander Lerchner.
\newblock Understanding disentangling in $\beta$-vae.
\newblock \emph{arXiv preprint arXiv:1804.03599}, 2018.

\bibitem[Cover and Thomas(2012)]{Cover1991}
Thomas~M Cover and Joy~A Thomas.
\newblock \emph{Elements of information theory}.
\newblock John Wiley \& Sons, 2012.

\bibitem[Dusenberry et~al.(2020)Dusenberry, Jerfel, Wen, Ma, Snoek, Heller,
  Lakshminarayanan, and Tran]{Dusenberry2020}
Michael~W Dusenberry, Ghassen Jerfel, Yeming Wen, Yi-an Ma, Jasper Snoek,
  Katherine Heller, Balaji Lakshminarayanan, and Dustin Tran.
\newblock Efficient and scalable bayesian neural nets with rank-1 factors.
\newblock \emph{arXiv preprint arXiv:2005.07186}, 2020.

\bibitem[Fischer(2019)]{Fisher2019}
Ian Fischer.
\newblock {The Conditional Entropy Bottleneck}.
\newblock \emph{Submission to ICLR 2019, International Conference on Learning
  Representations}, 2019.

\bibitem[Fischer(2020)]{fischer2020conditional}
Ian Fischer.
\newblock The conditional entropy bottleneck.
\newblock \emph{arXiv preprint arXiv:2002.05379}, 2020.

\bibitem[Fischer and Alemi(2020)]{fischer2020modelrobustness}
Ian Fischer and Alexander~A. Alemi.
\newblock Ceb improves model robustness.
\newblock \emph{Entropy}, 22\penalty0 (10):\penalty0 1081, Sep 2020.
\newblock ISSN 1099-4300.
\newblock \doi{10.3390/e22101081}.
\newblock URL \url{http://dx.doi.org/10.3390/e22101081}.

\bibitem[Gal and Ghahramani(2016)]{Gal2015}
Yarin Gal and Zoubin Ghahramani.
\newblock Dropout as a bayesian approximation: Representing model uncertainty
  in deep learning.
\newblock In \emph{international conference on machine learning}, pages
  1050--1059, 2016.

\bibitem[Ghosh et~al.(2019)Ghosh, Sajjadi, Vergari, Black, and
  Sch{\"o}lkopf]{ghosh2019variational}
Partha Ghosh, Mehdi~SM Sajjadi, Antonio Vergari, Michael Black, and Bernhard
  Sch{\"o}lkopf.
\newblock From variational to deterministic autoencoders.
\newblock \emph{arXiv preprint arXiv:1903.12436}, 2019.

\bibitem[Gondek and Hofmann(2003)]{Gondek}
David Gondek and Thomas Hofmann.
\newblock Conditional information bottleneck clustering.
\newblock In \emph{3rd ieee international conference on data mining, workshop
  on clustering large data sets}, pages 36--42. Citeseer, 2003.

\bibitem[Goodfellow et~al.(2013)Goodfellow, Mirza, Xiao, Courville, and
  Bengio]{Goodfellow}
Ian~J Goodfellow, Mehdi Mirza, Da~Xiao, Aaron Courville, and Yoshua Bengio.
\newblock An empirical investigation of catastrophic forgetting in
  gradient-based neural networks.
\newblock \emph{arXiv preprint arXiv:1312.6211}, 2013.

\bibitem[He et~al.(2016{\natexlab{a}})He, Zhang, Ren, and Sun]{He2016}
Kaiming He, Xiangyu Zhang, Shaoqing Ren, and Jian Sun.
\newblock Deep residual learning for image recognition.
\newblock In \emph{Proceedings of the IEEE conference on computer vision and
  pattern recognition}, pages 770--778, 2016{\natexlab{a}}.

\bibitem[He et~al.(2016{\natexlab{b}})He, Zhang, Ren, and Sun]{HeV22016}
Kaiming He, Xiangyu Zhang, Shaoqing Ren, and Jian Sun.
\newblock Identity mappings in deep residual networks.
\newblock \emph{Lecture Notes in Computer Science}, page 630–645,
  2016{\natexlab{b}}.

\bibitem[Higgins et~al.(2016)Higgins, Matthey, Pal, Burgess, Glorot, Botvinick,
  Mohamed, and Lerchner]{higgins2016beta}
Irina Higgins, Loic Matthey, Arka Pal, Christopher Burgess, Xavier Glorot,
  Matthew Botvinick, Shakir Mohamed, and Alexander Lerchner.
\newblock beta-vae: Learning basic visual concepts with a constrained
  variational framework.
\newblock 2016.

\bibitem[Hinton(1990)]{Hinton1990}
Geoffrey~E Hinton.
\newblock Connectionist learning procedures.
\newblock In \emph{Machine learning}, pages 555--610. Elsevier, 1990.

\bibitem[Houlsby et~al.(2011)Houlsby, Husz{\'a}r, Ghahramani, and
  Lengyel]{Houlsby2011}
Neil Houlsby, Ferenc Husz{\'a}r, Zoubin Ghahramani, and M{\'a}t{\'e} Lengyel.
\newblock Bayesian active learning for classification and preference learning.
\newblock \emph{arXiv preprint arXiv:1112.5745}, 2011.

\bibitem[Howard(2019)]{imagewang}
Jeremy Howard.
\newblock Imagewang.
\newblock 2019.
\newblock URL \url{https://github.com/fastai/imagenette/}.

\bibitem[Inoue(2019)]{Inoue2019}
Hiroshi Inoue.
\newblock Multi-sample dropout for accelerated training and better
  generalization.
\newblock \emph{arXiv preprint arXiv:1905.09788}, 2019.

\bibitem[Jette et~al.(2002)Jette, Yoo, and Grondona]{slurm}
Morris~A. Jette, Andy~B. Yoo, and Mark Grondona.
\newblock Slurm: Simple linux utility for resource management.
\newblock In \emph{In Lecture Notes in Computer Science: Proceedings of Job
  Scheduling Strategies for Parallel Processing (JSSPP) 2003}, pages 44--60.
  Springer-Verlag, 2002.

\bibitem[Kingma and Ba(2014)]{Kingma2015}
Diederik~P Kingma and Jimmy Ba.
\newblock Adam: A method for stochastic optimization.
\newblock \emph{arXiv preprint arXiv:1412.6980}, 2014.

\bibitem[Kingma and Welling(2013)]{kingma2013auto}
Diederik~P Kingma and Max Welling.
\newblock Auto-encoding variational bayes.
\newblock \emph{arXiv preprint arXiv:1312.6114}, 2013.

\bibitem[Kirsch et~al.(2019)Kirsch, van Amersfoort, and Gal]{Kirsch2019}
Andreas Kirsch, Joost van Amersfoort, and Yarin Gal.
\newblock Batchbald: Efficient and diverse batch acquisition for deep bayesian
  active learning.
\newblock In \emph{Advances in Neural Information Processing Systems}, pages
  7024--7035, 2019.

\bibitem[Kraskov et~al.(2004)Kraskov, St{\"o}gbauer, and
  Grassberger]{Kraskov2003}
Alexander Kraskov, Harald St{\"o}gbauer, and Peter Grassberger.
\newblock Estimating mutual information.
\newblock \emph{Physical review E}, 69\penalty0 (6):\penalty0 066138, 2004.

\bibitem[Krizhevsky et~al.(2009)Krizhevsky, Hinton, et~al.]{Krizhevsky2009}
Alex Krizhevsky, Geoffrey Hinton, et~al.
\newblock Learning multiple layers of features from tiny images.
\newblock 2009.

\bibitem[Kurakin et~al.(2017)Kurakin, Goodfellow, and
  Bengio]{kurakin2016adversarial}
Alexey Kurakin, Ian~J Goodfellow, and Samy Bengio.
\newblock Adversarial machine learning at scale.
\newblock 2017.

\bibitem[{Lecun} et~al.(1998){Lecun}, {Bottou}, {Bengio}, and
  {Haffner}]{LeCun1998}
Y.~{Lecun}, L.~{Bottou}, Y.~{Bengio}, and P.~{Haffner}.
\newblock Gradient-based learning applied to document recognition.
\newblock \emph{Proceedings of the IEEE}, 86\penalty0 (11):\penalty0
  2278--2324, 1998.

\bibitem[MacKay(2003)]{Mackay}
David J.~C. MacKay.
\newblock \emph{Information Theory, Inference, and Learning Algorithms}.
\newblock Cambridge University Press, 2003.

\bibitem[Madry et~al.(2018)Madry, Makelov, Schmidt, Tsipras, and
  Vladu]{madry2018towards}
Aleksander Madry, Aleksandar Makelov, Ludwig Schmidt, Dimitris Tsipras, and
  Adrian Vladu.
\newblock Towards deep learning models resistant to adversarial attacks.
\newblock In \emph{International Conference on Learning Representations}, 2018.

\bibitem[McAllester and Stratos(2018)]{McAllester2018}
David McAllester and Karl Stratos.
\newblock Formal limitations on the measurement of mutual information.
\newblock \emph{arXiv preprint arXiv:1811.04251}, 2018.

\bibitem[McGill(1954)]{McGill1954}
William McGill.
\newblock Multivariate information transmission.
\newblock \emph{Transactions of the IRE Professional Group on Information
  Theory}, 4\penalty0 (4):\penalty0 93--111, 1954.

\bibitem[Moosavi-Dezfooli et~al.(2016)Moosavi-Dezfooli, Fawzi, and
  Frossard]{moosavi2016deepfool}
Seyed-Mohsen Moosavi-Dezfooli, Alhussein Fawzi, and Pascal Frossard.
\newblock Deepfool: a simple and accurate method to fool deep neural networks.
\newblock In \emph{Proceedings of the IEEE conference on computer vision and
  pattern recognition}, pages 2574--2582, 2016.

\bibitem[Noh et~al.(2017)Noh, You, Mun, and Han]{Noh2017}
Hyeonwoo Noh, Tackgeun You, Jonghwan Mun, and Bohyung Han.
\newblock Regularizing deep neural networks by noise: Its interpretation and
  optimization.
\newblock In \emph{Advances in Neural Information Processing Systems}, pages
  5109--5118, 2017.

\bibitem[Noshad et~al.(2019)Noshad, Zeng, and Hero]{Noshad2019}
Morteza Noshad, Yu~Zeng, and Alfred~O Hero.
\newblock Scalable mutual information estimation using dependence graphs.
\newblock In \emph{ICASSP 2019-2019 IEEE International Conference on Acoustics,
  Speech and Signal Processing (ICASSP)}, pages 2962--2966. IEEE, 2019.

\bibitem[Oord et~al.(2018)Oord, Li, and Vinyals]{Oord2018}
Aaron van~den Oord, Yazhe Li, and Oriol Vinyals.
\newblock Representation learning with contrastive predictive coding.
\newblock \emph{arXiv preprint arXiv:1807.03748}, 2018.

\bibitem[Paszke et~al.(2019)Paszke, Gross, Massa, Lerer, Bradbury, Chanan,
  Killeen, Lin, Gimelshein, Antiga, et~al.]{Paszke2019}
Adam Paszke, Sam Gross, Francisco Massa, Adam Lerer, James Bradbury, Gregory
  Chanan, Trevor Killeen, Zeming Lin, Natalia Gimelshein, Luca Antiga, et~al.
\newblock Pytorch: An imperative style, high-performance deep learning library.
\newblock In \emph{Advances in Neural Information Processing Systems}, pages
  8024--8035, 2019.

\bibitem[Polyak and Juditsky(1992)]{polyak1992acceleration}
Boris~T Polyak and Anatoli~B Juditsky.
\newblock Acceleration of stochastic approximation by averaging.
\newblock \emph{SIAM journal on control and optimization}, 30\penalty0
  (4):\penalty0 838--855, 1992.

\bibitem[Poole et~al.(2019)Poole, Ozair, Oord, Alemi, and Tucker]{Poole2019}
Ben Poole, Sherjil Ozair, Aaron van~den Oord, Alexander~A Alemi, and George
  Tucker.
\newblock On variational bounds of mutual information.
\newblock \emph{arXiv preprint arXiv:1905.06922}, 2019.

\bibitem[Saxe et~al.(2019)Saxe, Bansal, Dapello, Advani, Kolchinsky, Tracey,
  and Cox]{Saxe2019}
Andrew~M Saxe, Yamini Bansal, Joel Dapello, Madhu Advani, Artemy Kolchinsky,
  Brendan~D Tracey, and David~D Cox.
\newblock On the information bottleneck theory of deep learning.
\newblock \emph{Journal of Statistical Mechanics: Theory and Experiment},
  2019\penalty0 (12):\penalty0 124020, 2019.

\bibitem[Shamir et~al.(2010)Shamir, Sabato, and Tishby]{Shamir2010}
Ohad Shamir, Sivan Sabato, and Naftali Tishby.
\newblock Learning and generalization with the information bottleneck.
\newblock \emph{Theoretical Computer Science}, 411\penalty0 (29-30):\penalty0
  2696--2711, 2010.

\bibitem[Shannon(1948)]{Shannon}
Claude~E Shannon.
\newblock A mathematical theory of communication.
\newblock \emph{Bell system technical journal}, 27\penalty0 (3):\penalty0
  379--423, 1948.

\bibitem[Shwartz-Ziv and Tishby(2017)]{Shwartz-Ziv2017}
Ravid Shwartz-Ziv and Naftali Tishby.
\newblock Opening the black box of deep neural networks via information.
\newblock \emph{arXiv preprint arXiv:1703.00810}, 2017.

\bibitem[Solla et~al.(1988)Solla, Levin, and Fleisher]{Solla1988}
Sara~A. Solla, Esther Levin, and Michael Fleisher.
\newblock Accelerated learning in layered neural networks.
\newblock \emph{Complex Systems}, 2, 1988.

\bibitem[Srivastava et~al.(2014)Srivastava, Hinton, Krizhevsky, Sutskever, and
  Salakhutdinov]{Srivastava2014}
Nitish Srivastava, Geoffrey Hinton, Alex Krizhevsky, Ilya Sutskever, and Ruslan
  Salakhutdinov.
\newblock Dropout: a simple way to prevent neural networks from overfitting.
\newblock \emph{The journal of machine learning research}, 15\penalty0
  (1):\penalty0 1929--1958, 2014.

\bibitem[Strouse and Schwab(2017)]{Strouse2016}
DJ~Strouse and David~J Schwab.
\newblock The deterministic information bottleneck.
\newblock \emph{Neural computation}, 29\penalty0 (6):\penalty0 1611--1630,
  2017.

\bibitem[Szegedy et~al.(2013)Szegedy, Zaremba, Sutskever, Bruna, Erhan,
  Goodfellow, and Fergus]{szegedy2013intriguing}
Christian Szegedy, Wojciech Zaremba, Ilya Sutskever, Joan Bruna, Dumitru Erhan,
  Ian Goodfellow, and Rob Fergus.
\newblock Intriguing properties of neural networks.
\newblock \emph{arXiv preprint arXiv:1312.6199}, 2013.

\bibitem[Tishby and Zaslavsky(2015)]{Tishby2015}
Naftali Tishby and Noga Zaslavsky.
\newblock Deep learning and the information bottleneck principle.
\newblock In \emph{2015 IEEE Information Theory Workshop (ITW)}, pages 1--5.
  IEEE, 2015.

\bibitem[Tishby et~al.(2000)Tishby, Pereira, and Bialek]{Tishby2000}
Naftali Tishby, Fernando~C Pereira, and William Bialek.
\newblock The information bottleneck method.
\newblock \emph{arXiv preprint physics/0004057}, 2000.

\bibitem[Tschannen et~al.(2019)Tschannen, Djolonga, Rubenstein, Gelly, and
  Lucic]{Tschannen2019}
Michael Tschannen, Josip Djolonga, Paul~K Rubenstein, Sylvain Gelly, and Mario
  Lucic.
\newblock On mutual information maximization for representation learning.
\newblock \emph{arXiv preprint arXiv:1907.13625}, 2019.

\bibitem[Wan et~al.(2013{\natexlab{a}})Wan, Zeiler, Zhang, Le~Cun, and
  Fergus]{Wan2013}
Li~Wan, Matthew Zeiler, Sixin Zhang, Yann Le~Cun, and Rob Fergus.
\newblock Regularization of neural networks using dropconnect.
\newblock In \emph{International conference on machine learning}, pages
  1058--1066, 2013{\natexlab{a}}.

\bibitem[Wan et~al.(2013{\natexlab{b}})Wan, Zeiler, Zhang, Le~Cun, and
  Fergus]{wan2013regularization}
Li~Wan, Matthew Zeiler, Sixin Zhang, Yann Le~Cun, and Rob Fergus.
\newblock Regularization of neural networks using dropconnect.
\newblock In \emph{International conference on machine learning}, pages
  1058--1066, 2013{\natexlab{b}}.

\bibitem[Xu et~al.(2020)Xu, Zhao, Song, Stewart, and Ermon]{Xu2020}
Yilun Xu, Shengjia Zhao, Jiaming Song, Russell Stewart, and Stefano Ermon.
\newblock A theory of usable information under computational constraints.
\newblock \emph{arXiv preprint arXiv:2002.10689}, 2020.

\bibitem[Yeung(1991)]{Yeung1991}
Raymond~W Yeung.
\newblock A new outlook on shannon's information measures.
\newblock \emph{IEEE transactions on information theory}, 37\penalty0
  (3):\penalty0 466--474, 1991.

\bibitem[Zhang et~al.(2016)Zhang, Bengio, Hardt, Recht, and Vinyals]{Zhang2019}
Chiyuan Zhang, Samy Bengio, Moritz Hardt, Benjamin Recht, and Oriol Vinyals.
\newblock Understanding deep learning requires rethinking generalization.
\newblock \emph{arXiv preprint arXiv:1611.03530}, 2016.

\bibitem[Zhang et~al.(2018)Zhang, Xiang, Hospedales, and Lu]{Zhang}
Ying Zhang, Tao Xiang, Timothy~M Hospedales, and Huchuan Lu.
\newblock Deep mutual learning.
\newblock In \emph{Proceedings of the IEEE Conference on Computer Vision and
  Pattern Recognition}, pages 4320--4328, 2018.

\end{thebibliography}
